\documentclass[nohyperref, 11pt]{article}
\usepackage[a4paper,top=1.5cm,bottom=1.6cm,left=1.8cm,right=1.8cm]{geometry}
\usepackage{microtype}
\usepackage{graphicx}
\usepackage{booktabs} 

\usepackage{amssymb,amsmath,amsthm,amsfonts}

\usepackage{xspace}
\usepackage{tabularx}
\usepackage{indentfirst}
\usepackage{cancel}
\usepackage[small]{caption}
\usepackage{eucal}
\usepackage{eso-pic}
\usepackage{url}
\usepackage{apptools}
\usepackage{booktabs}
\usepackage{subcaption}
\usepackage{mathtools}
\usepackage{amsthm}
\usepackage{bbold}
\usepackage{afterpage}
\usepackage{parskip}
\usepackage{listings}
\usepackage{fancyhdr}
\usepackage{textcomp,comment}

\usepackage{multirow}
\usepackage[utf8]{inputenc}   
\usepackage[english]{babel}   
\usepackage{setspace}
\usepackage{times}
\usepackage{amsmath}
\usepackage{graphics}
\usepackage{graphicx}
\usepackage{caption}
\usepackage{hyperref}
\usepackage{fancyhdr}
\usepackage{color}
\usepackage{cancel}
\usepackage{listings}
\usepackage{xcolor}
\usepackage{booktabs}
\usepackage{subcaption}
\usepackage{amsmath}
\usepackage{amssymb}
\usepackage{bbm}
\usepackage{braket}
\usepackage{subcaption}
\usepackage{tikz}
\usepackage{csquotes}
\usetikzlibrary{shapes,arrows}
\usepackage[utf8]{inputenc} 
\usepackage[T1]{fontenc} 
\usepackage{mathpazo} 
\usepackage{authblk}
\usepackage{enumerate}
\usepackage{float}
\usepackage{stfloats}

\theoremstyle{plain}
\newtheorem{theorem}{Theorem}[section]
\newtheorem{corollary}[theorem]{Corollary}
\newtheorem{lemma}[theorem]{Lemma}
\newtheorem{proposition}[theorem]{Proposition}

\theoremstyle{remark}

\usepackage{hyperref}

\title{Failure and success of the spectral bias prediction for
Kernel Ridge Regression: the case of low-dimensional data }
\author[1]{Umberto M. Tomasini$^{1}$, Antonio Sclocchi$^{1}$ and Matthieu Wyart$^{1}$ }
\date{%
   \small 
   $^1$ Institute of Physics, EPFL, Lausanne, Switzerland\\
   \{name.surname\}@epfl.ch\\
   }
\usepackage[
backend=biber,
style=alphabetic-verb,
maxnames=10
]{biblatex}

\begin{document}

\maketitle

\begin{abstract}
Recently, several theories including the replica method made predictions for the generalization error of Kernel Ridge Regression. In some regimes, they predict that the method has a `spectral bias': decomposing the true function $f^*$  on the eigenbasis of the kernel, it fits well the coefficients associated with the O(P) largest eigenvalues, where $P$ is the size of the training set. This prediction works very well on benchmark data sets such as images, yet the assumptions these approaches make on the data are never satisfied in practice. To clarify when the spectral bias prediction holds, we first focus on a one-dimensional model where rigorous results are obtained and then use scaling arguments to generalize and test our findings in higher dimensions. Our predictions include  the classification case  $f(x)=$sign$(x_1)$ with a data distribution that vanishes at the decision boundary $p(x)\sim x_1^{\chi}$. For $\chi>0$ and a Laplace kernel, we find  that (i) there exists a cross-over ridge  $\lambda^*_{d,\chi}(P)\sim P^{-\frac{1}{d+\chi}}$ such that for $\lambda\gg \lambda^*_{d,\chi}(P)$, the replica method applies, but not for $\lambda\ll\lambda^*_{d,\chi}(P)$, (ii) in the ridge-less case,  spectral bias predicts the correct training curve exponent only in the limit $d\rightarrow\infty$. 


\end{abstract}

\section{Introduction and Motivations}

Given the task of learning an unknown function $f^*$, a widely used algorithm is Kernel Ridge Regression (KRR) \cite{Smola2002}. Given a set of $P$ training points $\{x_i,f^*(x_i)\}_{i=1,...,P}$,  KRR builds a predictor function $f_P$ that is linear in a given  kernel $K$, such that it minimizes the following training loss:
	\begin{equation}
	    \sum\limits_{i=1}^P\left|f^*(x_i)-f_P(x_i)\right|^2 +\lambda ||f_P||_K^2,
	    \label{min}
	\end{equation}
	where $\lambda$ is the ridge parameter which controls the regularisation of the kernel norm $||.||_K$ of $f_P$.
	
	Minimising \eqref{min} is a convex problem, which yields the following explicit solution: 
	\begin{equation}
		f_P(x)=\vec{k}(x)(K+\lambda \mathbb{1})^{-1}\vec{y},
		\label{pred}
	\end{equation} 
	where $k(x)_i=K(x,x_i)$, $K_{ij}=K(x_i,x_j)$ is the $P\times P$ Gram matrix, in the noiseless setting we consider $y_i=f^*(x_i)$, $\lambda$ is the ridge regularization parameter and $\mathbb{1}$ is the $P\times P$ identity matrix.

	
	The generalization properties of KRR are an active field of research. In recent years, interest in the subject has been further increased by the discovery that for certain initializations, deep-learning behaves as a kernel method used in the ridge-less case \cite{Jacot2018}. 
	The key quantity of interest is the generalization error $\varepsilon_t$, namely how much error the predictor function $f_P(x)$ does on average on the data distribution $p(x)$, with $x$ in some space $\mathcal{D}$. Using the mean square loss, $\varepsilon_t$ is given by:
	\begin{equation}
		\varepsilon_t = \int_{\mathcal{D}} p(d^d x) (f_P(x)-f^*(x))^2.
		\label{mse}
	\end{equation}
	It is crucial to characterize $\varepsilon_t$ with respect to the number $P$ of training points since it allows quantification of how many samples are needed to achieve a given test error. It is empirically observed that, asymptotically for large $P$, $\varepsilon_t$ often behaves as a power law in $P$, with a certain exponent $\beta$: $\varepsilon_t(P)\sim P^{-\beta}$ \cite{Hestness2017, Spigler2020}. The exponent $\beta$ depends on the data distribution, the task, and the choice of kernel. 
	
	Recent theoretical efforts have characterized the test error in the noiseless setting considered here. 
	In \cite{Spigler2020}, $f^*$ was assumed to be Gaussian and the training set was assumed to be on a lattice. 
    In	\cite{Bordelon2020,Canatar2021,Canatar2020,Loureiro2021,Cui2021}, the replica method \cite{mezard} was used, assuming that the predictor $f_P$ is self-averaging (i.e. concentrates) and using a Gaussian assumption: a tuple of kernel eigenvectors $(\phi_1,...,\phi_P)$, once evaluated on $P$ training points, behaves as a Gaussian vector.   Random matrix theory was used in \cite{Jacot2020} with the same Gaussian assumption (with results not guaranteed to hold in the ridge-less case), or in \cite{mei2021generalization} with a `spectral gap' assumption. None of these assumptions should hold in practical applications \footnote{The spectral gap assumption used in \cite{mei2021generalization} may hold for Gaussian data in high dimension, but breaks down for real data which are highly anisotropic, see e.g. \cite{Spigler2020}.}. It is thus important to understand the universality of these results, and when they break down.

	
	{\bf Spectral bias:} These predictions for $\varepsilon_t$ rely on the exact eigendecomposition of the kernel:
	\begin{equation}
		\int p(y) K(y,x)\phi_{\rho}(y)dy = \lambda_\rho\phi_{\rho}(x),
		\label{eigenproblem}
	\end{equation}
	with $\{\phi_{\rho}\}$ the normalised eigenvectors and $\{\lambda_{\rho}\}$ the eigenvalues in decreasing order. In particular, the true function $f^*$ can be written as:
	\begin{equation}
		f^*(x)=\sum\limits_{\rho=1}^{\infty} c_{\rho}\phi_{\rho}(x),
	\end{equation}
	The key result is that KRR learns faster the eigenmodes corresponding to the $P$ largest eigenvalues, and makes an error on the following ones. 
	Specifically, in the noiseless case with no ridge $(\lambda=0)$ and assuming $c_{\rho}^2\sim \rho^{-a}$ and $\lambda_{\rho}\sim \rho^{-b}$ with $2b>(a-1)$, the prediction of the typical test error $\varepsilon_B$ in  \cite{Spigler2020,Bordelon2020} yields:
	\begin{equation}
		\varepsilon_B\sim\sum_{\rho=P}^{\infty}c_{\rho}^2\sim P^{-a+1},
		\label{sumP}
	\end{equation}
	These predictions are validated on the binary classification (corresponding to $f^*(x)=\pm 1$) of image data sets \cite{Bordelon2020, Spigler2020,Jacot2020}.
	 Why it is so is not well understood, since real data do not follow the assumptions made, whose universality class is not characterized. To understand the limit of validity of these theories, we seek to test them in simple models. It requires diagonalizing the kernel and having full control over the test error. Unfortunately, explicit diagonalisations of kernels is difficult, except if the data distribution $P(x)$ is uniform   on the sphere  \cite{Bordelon2020} or on the torus \cite{gretton}. For non-uniform data, the only settings that the authors are aware of are (i) a Gaussian kernel with a Gaussian data distribution \cite{gretton} and (ii) the work of \cite{Basri2020} where $p(x)$ is piece-wise uniform. 
	
	\subsection{This Paper}
	We consider data $x\in \mathbb{R}^d$ where the first component $x_1$ is distributed as $p(x_1)\sim |x_{1}|^\chi$ when $x_{1}\rightarrow 0$ for $\chi\geq 0$. We use the Laplacian kernel $K(x,y)=K(|x-y|)=\exp(-||x-y||_2 /\sigma)$, where $||.||_2$ is the $L_2$ norm and $\sigma>0$ defines the width of the kernel, and consider functions $f^*(x)=f^*(x_1)$ that depend only on the first component $x_1$ and can be singular or not at $x_1=0$. We first study the one-dimensional case where we are able to rigorously prove results by eigendecomposition of the kernel. We then extend these results to generic dimension $d$ by scaling arguments which are validated by numerical simulations.
	\begin{itemize} 
		
		\item In Section \ref{1d1}, for $d=1$, we compute the scaling of the generalization error with respect to the number of training points $P$ for vanishing ridge.
		
		\item In the same section, inspired by \cite{Basri2020} we derive an exact differential equation for the eigenvectors of the kernel $K$, which holds for a general data distribution $p(x)$. This equation is related to the Schrodinger equation in quantum mechanics. We solve it using methods developed in that field, to obtain the asymptotic behavior of the kernel eigenvectors and the eigenvalues.
		
		\item In Section \ref{phaseDiagr}, in the one-dimensional case, we find  that there exists a cross-over ridge  $\lambda^*_{1,\chi}(P)\sim P^{-\frac{1}{1+\chi}}$ such that for $\lambda\gg \lambda^*_{1,\chi}(P)$ spectral bias holds, but not for $\lambda\ll\lambda^*_{1,\chi}(P)$ where the exponent of the training curve is different. 
		 We repeat the same analysis for the test error prediction provided by \cite{Jacot2020} in Appendix \ref{kareApp}, and we observe the same crossover.
		 We show that when $\lambda\ll\lambda^*_{1,\chi}(P)$, the predictor is not self-averaging: its relative variance does not vanish even for very large $P$.
		 
		 \item In Section \ref{dGen}, we generalize these results to any dimension $d$ by scaling arguments that extend the proved results in $d=1$. One finds  a cross-over ridge $\lambda^*_{d,\chi}(P)\sim P^{-\frac{1}{d+\chi}}$ for any $d$ such that the spectral bias does not hold for $\lambda\ll\lambda^*_{d,\chi}(P)$, because the predictor is not self-averaging near the decision boundary. 
		 We confirm our results numerically and show that our model captures well the performance of KRR on CIFAR-10.
		

	\end{itemize}

	\section{Our models}\label{setting}


	\textbf{One dimension.} We consider a one-dimensional class of problems, where the data $x\in \mathbb{R}$ are distributed according to the probability distribution:
	\begin{equation}
		p(x)=\frac{1}{\Gamma\left(\frac{1+\chi}{2}\right)}|x|^{\chi} e^{-x^2},
		\label{pdf}
	\end{equation}
	where $\chi\ge 0$ and $\Gamma$ is the Euler gamma function $\Gamma(t)=\int_0^{\infty}dx\, x^{t-1}e^{-x}$. Our true function $f_\xi^*(x)$ depends on a parameter $\xi$ and it is defined as:
	\begin{equation}
		f_\xi^*(x)=\text{sign}(x)|x|^{-\xi}
		\label{trueFun}
	\end{equation}
	We restrict to $\xi$ such that $\xi<\frac{\chi+1}{2}$, to have the $L_2$ norm with respect to $p(x)$ finite. Note that for $\xi=0$ the task \eqref{trueFun} boils down to a binary classification problem. 
	For $\chi=0$ the data distribution is uniform, while the case $\chi>0$ is meant to model the presence of diminished density of data between data of different labels. Such a reduction of density is apparent in low-dimensional representations of real datasets, as for the t-SNE visualization of MNIST in \cite{tSNE}.

	
	\textbf{Generic dimension.} We generalize the one-dimensional setting above to a generic dimension $d$. We consider a cylindrical embedding of the data $x = [x_1,..., x_{d+1}]$ so that the first coordinate $x_1$ is distributed according to $p(x_1)\propto x_1^{\chi}e^{-x_1^2}$, while the other coordinates $x_2,..., x_{d+1}$ are uniformly randomly distributed on the sphere $\sum_{i=2}^{d+1}x_i^2 = 1$. In this setting, we consider the true function $f^*(x) = \text{sign}(x_1)$, so that the hyper-plane $x_1=0$ corresponds to the decision boundary of a binary classification problem.
	

	\section{Test Error analysis} \label{1d1}
	We  now state our result about the generalisation error in the setting described in Section \ref{setting} for $d=1$.
	
	\begin{theorem}[Test error] \label{testErrTh}
		Consider a training set $\{x_i,f^*(x_i)\}_{i=1...P}$, where the samples $x_i$ are i.i.d. with respect to the PDF \eqref{pdf} and the true function $f^*$ is \eqref{trueFun}. In the limit of large $P$, the following asymptotic relation for the test error \eqref{mse} of KRR with Laplacian kernel with width $\sigma$ and vanishing ridge $\lambda\rightarrow0^+$ holds:
		\begin{equation}
			\begin{aligned}
				\varepsilon_t \sim P^{-1+\left(\frac{2\xi}{\chi+1}\right)}
			\end{aligned}
			\label{scalTestError}
		\end{equation}
	\end{theorem}
	
	The full proof is reported in Appendix \ref{0ridgeproofs}.	The intuition behind \eqref{scalTestError} is the following. If we call $x_{A}<0$ and $x_{B}>0$ the points of the sampled training set which are closest to $x=0$, we have that their typical value is the following:
	\begin{equation}
		\langle |x_{A}| \rangle \sim \langle |x_{B}| \rangle \sim P^{-\frac{1}{\chi+1}}.
		\label{extremalPoints}
	\end{equation}
	This is given by the fact that $\langle x_{B} \rangle$ is defined as the extremal point such that in the interval $[0,x_{B}]$ there is just one sampled point on average:
	\begin{equation}
	    \frac{1}{P}\sim \int_{0}^{\langle x_{B} \rangle}dx\,p(x),
	\end{equation}
	which yields \eqref{extremalPoints}. The same holds for $x_{A}$. We then consider the asymptotic limit of $\sigma \rightarrow \infty$, where the Laplacian kernel becomes a cone in $x$. For $\lambda\rightarrow0^+$, the predictor $f_P$ is then given by the following piece-wise linear function for $\xi=0$:
	\begin{equation}
		f_P(x)=\left\{ 
		\begin{aligned}
			&\text{sign}(x),\hspace{1.8cm} \text{ for } x\ge x_{B} \text{ or } x\le x_{A} \\
			&\frac{2 x}{x_B-x_A}  - \frac{x_A + x_B}{x_B-x_A}, \text{ for } x_A < x< x_{B}
		\end{aligned}
		\right.
	\label{piecewise}
	\end{equation}
	A representation of \eqref{piecewise} is given by the blue line in Fig. \ref{predictors}. For $\xi>0$, the predictor for $x_A<x<x_B$ will be as in Eq. \eqref{piecewise}, and it will approximate $f^*_\xi$ with a piece-wise function otherwise. The leading contribution to the test error in the asymptotic limit of large $P$ is given by the interval $[x_{A},x_{B}]$:
	\begin{equation}
	\begin{aligned}
	 		\varepsilon_t \sim \int_{x_{A}}^{x_{B}}dx\, p(x)(f_P(x)-f^*_\xi (x))^2\sim\\
	 		\sim\int_0^{x_{B}}x^{\chi-2\xi}dx \sim P^{-1+\left(\frac{2\xi}{\chi+1}\right)},   
	\end{aligned}
	\end{equation}
	in accordance to \eqref{scalTestError}. Considering a generic finite $\sigma$, \eqref{scalTestError} still holds, as we prove and numerically test in Appendix \ref{0ridgeproofs}. 
	
	

	
		\begin{figure}
		\centering
		\includegraphics[width=.7\linewidth]{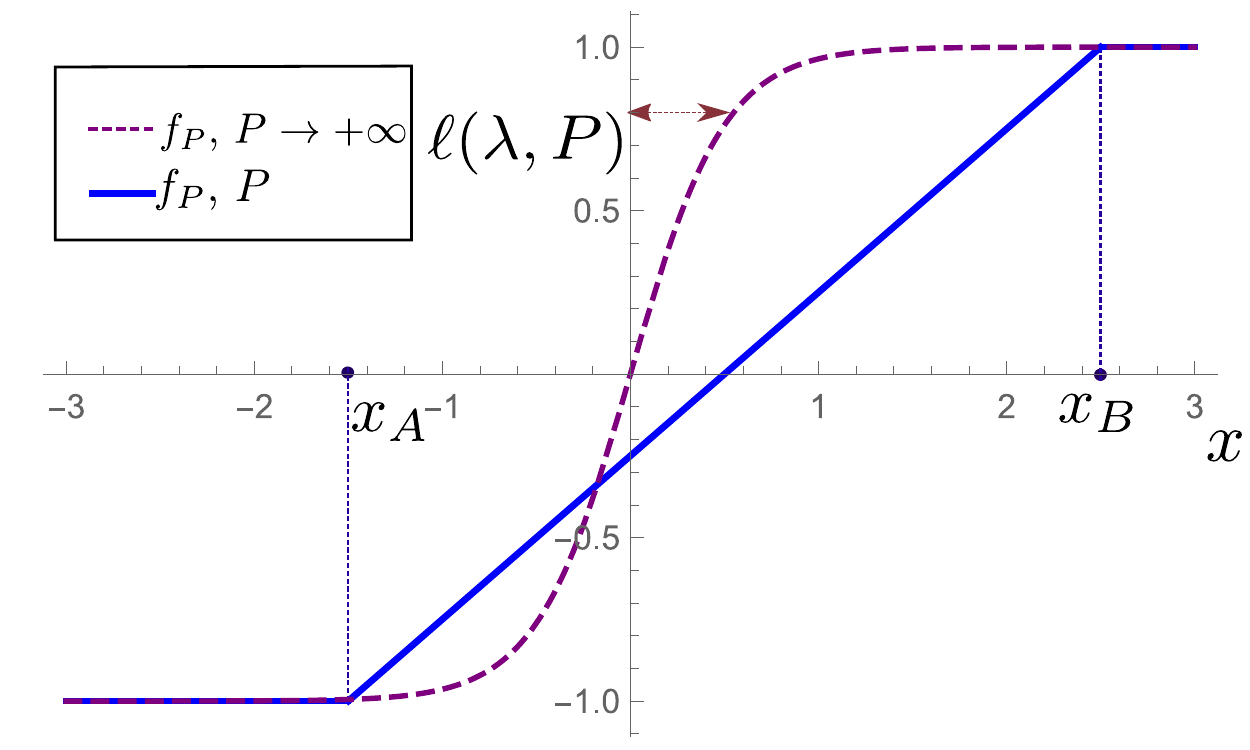}

		\caption{Representations of KRR predictors $f_P$ \eqref{pred} for $\xi=0$ and fixed  $\lambda/P$. The blue line is the predictor $f_P$ for finite $P$, in the case of the extremal point $x_{B}\sim P^{-1/(\chi+1)}$ \eqref{extremalPoints} being much larger than the characteristic scale $\ell(\lambda,P)\sim \left(\frac{\lambda\sigma}{P}\right)^{\frac{1}{(2+\chi)}}$ of the predictor \eqref{ell}. In the limit $P\rightarrow\infty$, the predictor $f_P$ is represented by the dashed purple line.}
		\label{predictors}
	\end{figure}

	\subsection{Eigendecomposition of the kernel}
	\label{sec:kernel_decomposition}
	To effectively test the spectral bias prediction for the KRR test error \eqref{sumP} in our context, we need to solve the eigenproblem \eqref{eigenproblem} for the Laplacian kernel with width $\sigma$ and the probability distribution \eqref{pdf}. All the proofs and more detailed statements of what follows are provided in Appendix \ref{Eigenproofs}, except for Thm. \ref{thmpde}.
	
	We first show a general result regarding the problem of finding the eigenvectors $\phi_{\rho}$ of the Laplacian kernel using a generic $p(x)$, which is recast in solving a differential equation. We will then use this result in the particular context of \eqref{pdf}.
	\begin{theorem}\label{thmpde}
		Let $K$ be the Laplacian kernel with width $\sigma$. Consider a one-dimensional input space $x\in \mathbb{R}$. Then the eigenvectors $\phi_{\rho}$ of the kernel, defined in \eqref{eigenproblem}, solve the following differential equation for $\lambda_{\rho}\neq 0$:
		\begin{equation}
			\partial_x^2\phi_{\rho}(x)=\left(-2\frac{p(x)}{\lambda_{\rho} \sigma} +\frac{1}{\sigma^2}\right)\phi_{\rho}(x).
			\label{pde}
		\end{equation}
	\end{theorem}
	
	\begin{proof}
		Let's rewrite the eigendecomposition relation as follows, writing explicitly the kernel $K$:
		\begin{equation}
		\begin{aligned}
			\int_{-\infty}^{x} p(y)& \phi_{\rho}(y) e^{-\frac{(x-y)}{\sigma}}dy +\\ &+\int_{x}^{\infty} p(y) \phi_{\rho}(y) e^{-\frac{(y-x)}{\sigma}}dy= \lambda_\rho\phi_{\rho}(x).
		\end{aligned}
			\label{b1}
		\end{equation}
		We derive two times the relation \eqref{b1} with respect to $x$ , following an idea similar to \cite{Basri2020}, getting:
		\begin{equation}
		\begin{aligned}
			-\frac{2}{\sigma}&p(x)\phi_{\rho}(x)+\frac{1}{\sigma^2}\left(\int_{-\infty}^{x} p(y) \phi_{\rho}(y) e^{-\frac{(x-y)}{\sigma}}dy +\right.\\ 
			&+\left.\int_{x}^{\infty} p(y) \phi_{\rho}(y) e^{-\frac{(y-x)}{\sigma}}dy\right) = \lambda_\rho \partial_x^2 \phi_{\rho}(x).
			\end{aligned}
			\label{b2}
		\end{equation}
		Substituting \eqref{b1} into \eqref{b2} and dividing by $\lambda_{\rho}\neq 0$, we get \eqref{pdeApp}.
	\end{proof}
	
	The functional operator entering \eqref{pde} is symmetric  with respect to $x$ because $p(x)=p(-x)$. Thus there is always an eigenbasis for the space of solutions for which the $\phi_{\rho}$ are  either even or odd functions in $x$. From the definition of the $\phi_\rho$ in \eqref{eigenproblem} and from \eqref{pde} in the limit of $|x|\rightarrow\infty$, we get the boundary  condition $\phi_\rho(x)\rightarrow0$ for $|x|\rightarrow\infty$.

	We asymptotically solve the equation \eqref{pde} for $\phi_\rho$, in the limit of small $\lambda_\rho$. We use the so-called Wentzel–Kramers–Brillouin (WKB) method \cite{Ghatak1991}, designed to solve the following differential equation: 
	\begin{equation}
	    \partial_x^2\psi(x)+\Gamma^2(x)\psi(x)=0.
	    \label{wkb1}
	\end{equation}
	This equation is  encountered for example in quantum mechanics: the Schroedinger equation has the same form of \eqref{wkb1}, with $\psi$ being the wave function of a particle and $\Gamma^2(x)=\frac{2m}{\hbar}[E-V(x)]$, with $m$ the mass particle, $\hbar$ the rescaled Planck's constant, $E$ the total energy and $V(x)$ the potential energy function of the system. 
	In the KRR case of \eqref{pde}, $\psi$ is the eigenvector $\phi_\rho$ and $\Gamma^2(x)$ is related to the PDF $p(x)$.
	
	The WKB solution of \eqref{wkb1} is obtained as follows. It is crucial to identify a small parameter $\lambda_0\ll1$ and a function $\tilde{\Gamma}$ finite in the limit $\lambda_0\rightarrow 0^+$ such that we can rewrite $\Gamma^2$:
	\begin{equation}
	    \Gamma^2(x)= \frac{1}{\lambda_0}\tilde{\Gamma}^2(x).
	    \label{wkb2}
	\end{equation}
	The limit of $\lambda_0\rightarrow 0^{+}$ is equivalent to consider the function $\Gamma^2(x)$ as slowly changing in $x$. The role of $\lambda_0$ is played in quantum mechanics by $\hbar$ and in the KRR setting of \eqref{pde} by the eigenvalue $\lambda_\rho$. One then seeks a solution of \eqref{wkb1} of the form:
	\begin{equation}
	    \psi(x)= e^{\frac{iS(x)}{\lambda_0}}, \quad S(x) = S_0(x)+\lambda_0 S_1(x)+\lambda_0^2 S_2(x)+...
	    \label{wkb3}
	\end{equation}
	where the function $S(x)$ is expanded in series of $\lambda_0$. If we substitute the solution \eqref{wkb3} into \eqref{wkb1}, we can get expressions for each function $S_{i}(x)$ for $i$ arbitrarily large. At the first order in $\lambda_0$, we get the following solution:
	\begin{equation}
	    \psi_1(x)= \frac{C}{\left(\tilde{\Gamma}^2(x)\right)^{1/4}}\exp\left( \frac{i}{\sqrt{\lambda_0}} \int^{x}dy\sqrt{\tilde{\Gamma}^2(y)}\right),
	    \label{wkb4}
	\end{equation}
	which is essentially an oscillatory or exponential function multiplied by an amplitude dependent on $x$. The contributes to $S(x)$ from $S_2(x)$ onwards are negligible with respect to the others provided that:
	\begin{equation}
	    \left|\frac{1}{2\Gamma}\partial_x^2\Gamma-\frac{3}{4\Gamma^2}(\partial_x\Gamma)^2\right|\ll \Gamma^2(x),
	\end{equation}
	which holds in the case of \eqref{pde} except in the proximity of the two points $x_1$ and $x_2$ where $\Gamma^2(x)=0$.
	
	In the Appendix \ref{Eigenproofs}, in the Lemmas \ref{evecChiMagg0} and \ref{evecChi0}, we derive at leading order in $\lambda_\rho$ the full form of the eigenvectors $\phi_\rho$ for all $x\in \mathbb{R}$. Close to the points $x_1$ and $x_2$, we linearize the function $\Gamma^2(x)$  to solve analytically the differential equation \eqref{wkb1} using the Airy functions \cite{Florentin1966}. Then we patch together the solution around the  points $x_1$ and $x_2$ and the WKB solution using the Modified Airy Functions (MAF) \cite{Ghatak1991}.
	
	

Once we solve the differential equation \eqref{pde} and we get the eigenvectors $\phi_\rho$ at the leading order in $\lambda_\rho$, we can compute the coefficients $c_\rho$ by projecting the true function \eqref{trueFun} on the eigenvectors. In particular, we are interested in the coefficients $c_\rho$ at the leading order in $\lambda_\rho$.
	
	\begin{proposition}(Coefficients)\label{propCoeff}
		Let $K$ be the Laplacian kernel with width $\sigma$. Let $p(x)$ be \eqref{pdf} and the true function $f^*$ \eqref{trueFun}. Consider a small eigenvalue $\lambda_{\rho}\ll 1$. Let $\phi_\rho$ be the solution of \eqref{pde}. We impose that $\phi_\rho(x)\rightarrow0$ for $|x|\rightarrow\infty$. Then the following holds for the coefficient $|c_\rho|$ defined in \eqref{coeff}, in the limit $\lambda_{\rho}\ll 1$:
		\begin{equation}
			\begin{aligned}
				|c_{\rho}|\sim& \lambda_{\rho}^{\frac{\frac{3}{4}\chi+1-\xi}{\chi+2}}& \text{if }&\phi_{\rho} \text{ is odd}\\
				|c_{\rho}| = & 0 &\text{   if }&\phi_{\rho} \text{ is even}.
			\end{aligned}
			\label{scal1coeff}
		\end{equation}
	\end{proposition}
	
	To get the scaling of the coefficients $c_\rho$ with respect to the eigenvalue rank $\rho$, we need to compute the eigenvalues $\lambda_\rho$ at the leading order in $\rho$.
	
	We first find a close formula satisfied by the eigenvalues $\lambda_\rho$, requiring that they are such that the eigenvectors $\phi_\rho$ respect the boundary condition $|\phi_\rho(x)|\rightarrow 0$ for $|x|\rightarrow\infty$. In other words, we find the eigenvalues $\lambda_\rho$ such that any not-decaying exponential contribute in the WKB solution \eqref{wkb4} is identically zero for large $x$.
	
	In particular, for odd $\phi_\rho$ and $\chi>0$ we find the following self-consistent relation satisfied by $\lambda_{\rho}\ll 1$, or equivalently by $\rho\gg 1$:
	\begin{equation}
			\lambda_{\rho} =\left(\frac{\int_{x_1}^{x_{2}}dx\sqrt{2\frac{p(x)}{\sigma}-\frac{\lambda_\rho}{\sigma^2}}}{\arctan(-\gamma_{1}^{-1})+\frac{\rho-1}{2} \pi}\right)^2 +o(\rho^{-2})
			\label{sc1}
		\end{equation}
	where  $\gamma_1 = \text{Ai}(\mu)/\text{Bi}(\mu)$, with $\text{Ai}$ and $\text{Bi}$ the Airy function of the first and second kind \cite{Florentin1966} and $\mu=\left(\frac{\chi (\lambda_\rho\Gamma[\frac{1+\chi}{2}])^{\frac{2}{\chi}}}{2^{\frac{2}{\chi}}\sigma^{2(1+\chi)}}\right)^{1/3}$. Similar relations hold for even $\phi_\rho$ and $\chi=0$, as presented in Appendix \ref{Eigenproofs}. 
	The self-consistent relation \eqref{sc1} yields the following asymptotic scaling for the eigenvalues $\lambda_\rho$ for large $\rho$:
	\begin{equation}
			\lambda_{\rho}\sim \rho^{-2}
			\label{scalEval}
		\end{equation}
	Now that we have the scaling of the eigenvalues $\lambda_\rho$, we can get the scaling of the coefficients $c_\rho$ with respect to their ranks $\rho$.

			\begin{figure*}[h]
		\centering
			\includegraphics[width=1\linewidth]{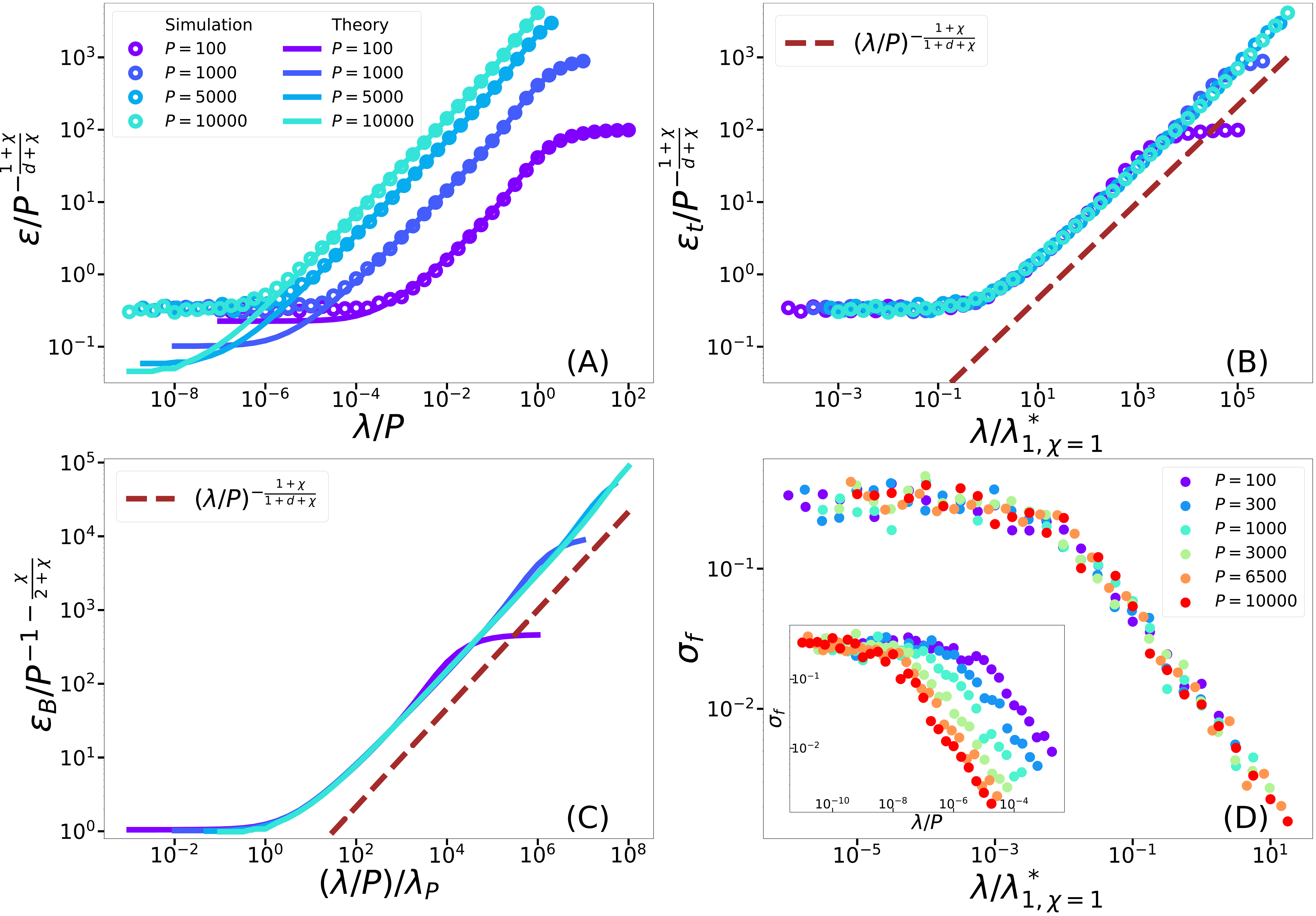}
			
		\caption{ $d=1$, $\chi=1$ and $\xi=0$.
		(A) Open symbols: empirical test error $\varepsilon_t$ (averaging over 200 realisations) rescaled by its ridgeless prediction \eqref{scalTestError}. Full lines: replica prediction $\varepsilon_B$ for fixed training set size $P$ and varying ridge $\lambda/P$. (B): the ridge has been rescaled by $\lambda^*_{1,\chi}$, defined in \eqref{lstar}. Brown line:  asymptotic behavior of $\varepsilon_B$ with $\lambda$ as predicted from Eq. \eqref{predBord}. 
		(C) $\varepsilon_B$, rescaled by its ridgeless prediction \eqref{predBord}, for fixed $P$ and varying rescaled ridge $(\lambda/P)/\lambda_P$, where $\lambda_P$ is defined in the main text. (D):	Inset: Variance of the predictor $\sigma_f$ as defined in \eqref{sigmaf} (averaged over 50 realisations) as a function of the rescaled ridge $\lambda/P$. Main plot: after rescaling the x-axis by $\lambda/\lambda^*_{1,\chi}$, the curves collapse as predicted.
		}
		\label{testError_departure_chi1_bord}
	\end{figure*}
	\begin{theorem} \label{coeffScalTh}
		Let $K$ be the Laplacian kernel with width $\sigma$. Let $p(x)$ be \eqref{pdf} and the true function $f^*$ \eqref{trueFun}. As a consequence of \eqref{scal1coeff} and \eqref{scalEval}, the following aymptotic relation holds for large $\rho$ such that $\phi_{\rho}$ is odd in $x$, for any $\chi\ge 0$:
		\begin{equation}
			c_{\rho}^2\sim \rho^{-\frac{3\chi+4-4\xi}{\chi+2}}.
			\label{coeffScalFinal}
		\end{equation}
	\end{theorem}
	
	Using \eqref{coeffScalFinal}, we are finally able to get the prediction of the test error in the ridgless limit $\lambda\rightarrow 0^{+}$ via the spectral bias theory \eqref{predBord}. This entails summing the coefficients squared $c_\rho^2$ from the $P-$th one onwards:
	\begin{equation}
		\varepsilon_B\sim \sum_{\rho=P}^{\infty}c_{\rho}^2\sim P^{-1-\left(\frac{\chi-4\xi}{\chi+2}\right)}.
		\label{predBord}
	\end{equation}
	Comparing with Thm. \ref{testErrTh}, we thus conclude that the spectral bias prediction \eqref{predBord} is incorrect.
	
	
	
	\section{Role of ridge $\lambda$}\label{phaseDiagr}
	
	
	The replica method \cite{Bordelon2020,Canatar2021} assumes that the predictor is a self-averaging quantity. Approaches based on random matrix theory  only apply under the same condition, which can be guaranteed only for a finite ridge \cite{Jacot2020} (under the Gaussian assumption). In our model in the ridge-less case, the test error is explicitly a function of two data points $x_A$ and $x_B$, and thus cannot be self-averaging. Thus we expect that these methods will work only when the ridge increases past some characteristic value $\lambda^*_{1,\chi}(P)$ to make the test error self-averaging, or equivalently if the training set is larger than some characteristic value $P^*(\lambda)$.
	
	To estimate $P^*(\lambda)$, our strategy is to compute the KRR predictor $f_P$ in the limit of $P\rightarrow\infty$ and $\frac{\lambda}{P}$ finite. This solution will apply for $P\gg P^*(\lambda)$. In the other limit $P\ll P^*(\lambda)$, the KRR predictor must be similar to the case $\lambda=0$ studied above, for which it is piece-wise linear.

	\begin{proposition}
		Let $K$ be the Laplacian kernel with width $\sigma$. The KRR predictor $f_P$ with kernel $K$, in the limit of $P\rightarrow\infty$ and $\frac{\lambda}{P}$ finite, satisfies the following differential equation:
		\begin{equation}
			\sigma^2 \partial_x^2 f_P(x) =\left(\frac{\sigma}{\lambda/P}p(x)+1\right)f_P(x)-\frac{\sigma}{\lambda/P}p(x)f^*(x).
			\label{predPlarge}
		\end{equation}
	\end{proposition}
    
    The equation \eqref{predPlarge} is obtained by noticing that, for the Laplace kernel in one dimension, the kernel norm $||f_P||_K^2$ corresponds to $||f_P||_K^2 = \frac{1}{\sigma}(\int dt f_P(t)^2 + \sigma^2 \int dt f_P'(t)^2)$. Therefore, minimizing the training loss \eqref{min} by taking the functional derivative with respect to $f_P$ yields the linear differential equation \eqref{predPlarge} for $f_P(x)$ (proof in Appendix \ref{scaling-one-d}).
	
	Considering the $p(x)$ introduced in Section \ref{setting}, the relation \eqref{predPlarge} yields the following characteristic scale for the function $f_P$:
	\begin{equation}
		\ell(\lambda,P)\sim \left(\frac{\lambda\sigma}{P}\right)^{\frac{1}{(2+\chi)}}.
		\label{ell}
	\end{equation}
	
	
 	This scale is obtained by noticing that the homogeneous equation of Eq. \eqref{predPlarge} has the same form as the Schroedinger equation \eqref{wkb1}. Therefore, the WKB expansion for small $\lambda/P$ can be used as discussed in Section \ref{sec:kernel_decomposition}, yielding Eq. \eqref{ell} for the characteristic scale $\ell$ at small $x$. The proof is reported in Appendix \ref{scaling-one-d}.
	
	The function $f_P$ is sketched in Fig.\ref{predictors} for fixed $\frac{\lambda}{P}$ and $P$ finite, and compared with the KRR predictor in the limit $P\rightarrow\infty$ .
	For large $P$, the latter limit must be a good approximation of the KRR predictor \eqref{predPlarge}. However, this approximation will break down when  $P$ is small: in that case, the first data point $x_B$ will be much larger than $\ell(\lambda,P)$, and the solution will be approximately piece-wise linear, as in Fig.\ref{predictors}. The cross-over between the two regimes must occur when $x_B\sim \ell(\lambda,P)$, leading to a characteristic ridge:
    	\begin{equation}
		\lambda^*_{1,\chi}\sim P^{-\frac{1}{1+\chi}}.
		\label{lstar}
	\end{equation}
	\begin{figure*}[h]
		\centering
		\includegraphics[width=1\linewidth]{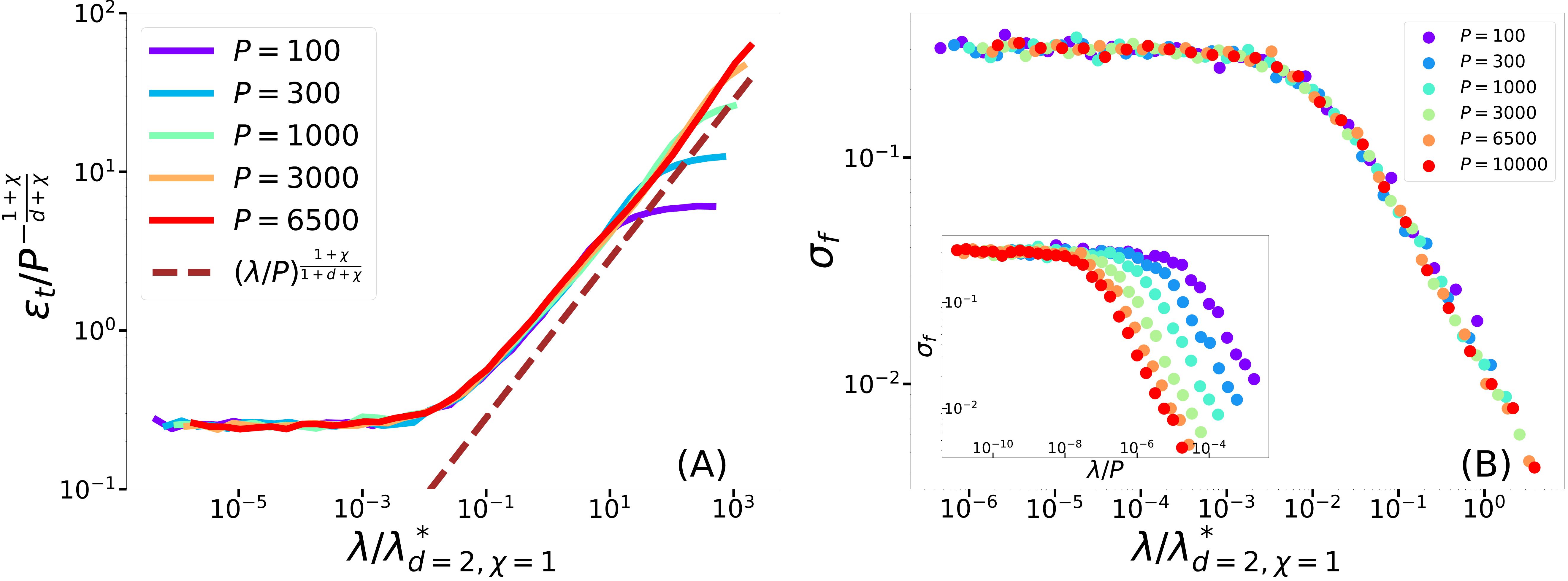}
		\caption{ $d=2$ and $\chi=1$. (A): Empirical test error $\varepsilon_t$ rescaled by its ridgeless prediction \eqref{scalTestError} for fixed training set size $P$ and varying rescaled ridge $\lambda/\lambda^*_{d,\chi}$, with $\lambda^*_{d,\chi}$ defined in Eq. \eqref{lstarD}. Brown line:  predicted scaling of  $\varepsilon_B$ with respect to $\lambda$, as follows from Eq. \eqref{errBlambda}. (B) Inset: $\sigma_f$, defined in \eqref{sigmaf} for fixed $P$ and varying $\lambda/P$. At small ridge, $\sigma_f$ does not decrease with $P$. Main plot: $\sigma_f$ collapses as a function of $\lambda/\lambda^*_{d,\chi}$ as predicted. 
		}
		\label{figd3}
	\end{figure*}
	
	{\bf Numerical test:} To confirm that $\lambda^*_{1,\chi}$  marks the point where replica theory $\varepsilon_B$ breaks down,  we compare it with the empirical test error $\varepsilon_t$ numerically obtained  for $\chi=1$ and $\xi=0$  in Fig. \ref{testError_departure_chi1_bord} (A). 
	For small $\lambda/P$, the prediction $\varepsilon_B$ and the numerical results reach a different plateau, while for large $\lambda/P$ they coincide. Hence there is a crossover in $\lambda$, for fixed $P$, between values of $\lambda$ where the prediction $\varepsilon_B$ works and where it does not.
	 After rescaling $\lambda$ by $P^{-\frac{1}{1+\chi}}$, the empirical curves $\varepsilon_t$ for different $P$ collapse as shown in Fig. \ref{testError_departure_chi1_bord} (B). It is true in particular for the location where $\varepsilon_t$ starts flattening and departs from  $\varepsilon_B$, confirming that replica theory breaks down for  $\lambda\ll \lambda^*_{1,\chi}$.
	In Fig. \ref{testError_departure_chi1_bord} (C), we confirm that in replica theory $\varepsilon_B$ reaches a plateau when $\lambda/P\ll\lambda_P$. In fact, in Appendix \ref{statmec}, we show that $\varepsilon_B$ has small relative changes when the rescaled ridge $\lambda/P$ goes from zero to $\lambda_P$, where $\lambda_P$ is the rank $P$ eigenvalue of the kernel.

	Finally, we confirm that replica theory breaks down when the predictor is not self-averaging near the decision boundary. To do so, we consider the variance of the predictors $f_P$ obtained from different training sets.

	We define $\sigma_f$ as:
	\begin{align}
	    \sigma_f = \frac{1}{N_P} \sum_{i=1}^{N_P} [f_{P,1}(x_i) - f_{P,2}(x_i)]^2
	    \label{sigmaf}
	\end{align}
	where $f_{P,1}$ and $f_{P,2}$ are two different predictors obtained by two different training sets of same size $P$ and $x_{\{i=1,...,N_P\}}$ are the test points where the signs of the two predictors $\text{sign}(f_{P,1}(x_i))$ and $\text{sign}(f_{P,2}(x_i))$ are different.
	In the inset of Fig. \ref{testError_departure_chi1_bord} (D), $\sigma_f$ {\it v.s.} $\lambda/P$ is shown: for small ridges, the variance of the predictors does not decrease for increasing $P$, and the predictor is not self-averaging. We observe in the main plot that the curves collapse if $\lambda$ is rescaled by $\lambda^*_{1,\chi}$ as predicted in Eq. \eqref{lstar}.

   \section{Higher dimension setting and real data}\label{dGen}

    We generalize the previous results to higher dimension $d$ using scaling (non-rigorous) arguments, that make stringent predictions that we test numerically. 
    
    {\bf Ridgless case:} The typical distance $r_{min}$ between training points at the decision boundary can be estimated as the size of the ball in which in average one data point lies. It leads to:
    \begin{align}
        r_{min}\sim P^{-\frac{1}{d+\chi}}
    \end{align}
    In the absence of ridge, $f_P(x)$ will display fluctuations of order one for $|x_1|\sim r_{min}$. Thus the test error must be of order of the probability for a test point to fall within a distance $r_{min}$ from the interface:
    \begin{align}
        \varepsilon_t \sim r_{min}^{1+\chi} \sim P^{-\frac{1+\chi}{d+\chi}}
        \label{errT-scaling}
    \end{align}
    
    {\bf Finite  ridge:}  In the limit $\lambda/P$ fixed and large $P$, the predictor will vary near the decision boundary on some length scale $\ell(\lambda,P)$. In Appendix \ref{scaling-high-d} we argue that:
    \begin{align}
        \ell(\lambda,P) \sim \left(\frac{\lambda}{P}\right)^{\frac{1}{1+d+\chi}}
        \label{ell-d}
    \end{align}
    The test error predicted by the replica method then follows: 
    \begin{align}
        \varepsilon_B \sim \ell(\lambda,P)^{1+\chi}\sim \left(\frac{\lambda}{P}\right)^\frac{1+\chi}{1+d+\chi}
        \label{errBlambda}
    \end{align}
    In Appendix \ref{statmec}, we show that the replica solution has only mild relative changes when the rescaled ridge $\lambda/P$ goes from zero to $\lambda_P$, where $\lambda_P$ is the rank $P$ eigenvalue of the covariant operator. For a Laplace kernel, $\lambda_P\sim P^{-1-\frac{1}{d}} $ \footnote{Using the Fourier variable $q$, we have in that case $\lambda_P\sim q_{max}^{-1-d}$ and $q_{max}\sim P^{1/d}$, see e.g. \cite{Spigler2020}.}.
   Substituting $\lambda/P$ by $\lambda_P$ in Eq. \eqref{errBlambda}, we obtain the spectral bias prediction:
    \begin{align}
        \varepsilon_B \sim \lambda_P^{\frac{1+\chi}{1+d+\chi}}\sim P^{-(1+\frac{1}{d})\frac{1+\chi}{1+d+\chi}}
        \label{errB-scaling}
    \end{align}
    Comparing \eqref{errT-scaling} and \eqref{errB-scaling}, we obtain the following key results:
    (i) for $\chi=0$, the spectral bias predicts the correct asymptotic training curve exponent. (ii)
     For $\chi>0$,  the spectral bias predicts a wrong exponent. However,  the prediction is correct in the limit $d\rightarrow \infty$, and is already excellent at intermediary dimensions (say $d=10$). (iii) The replica prediction breaks down when $\ell(\lambda,P) \sim r_{min}$, which implies a cross-over ridge:
    \begin{align}
        \lambda^*_{d,\chi}\sim P^{-\frac{1}{d+\chi}}
        \label{lstarD}
    \end{align}
    
       	\begin{figure*}[h]
		\centering
		\includegraphics[width=1\linewidth]{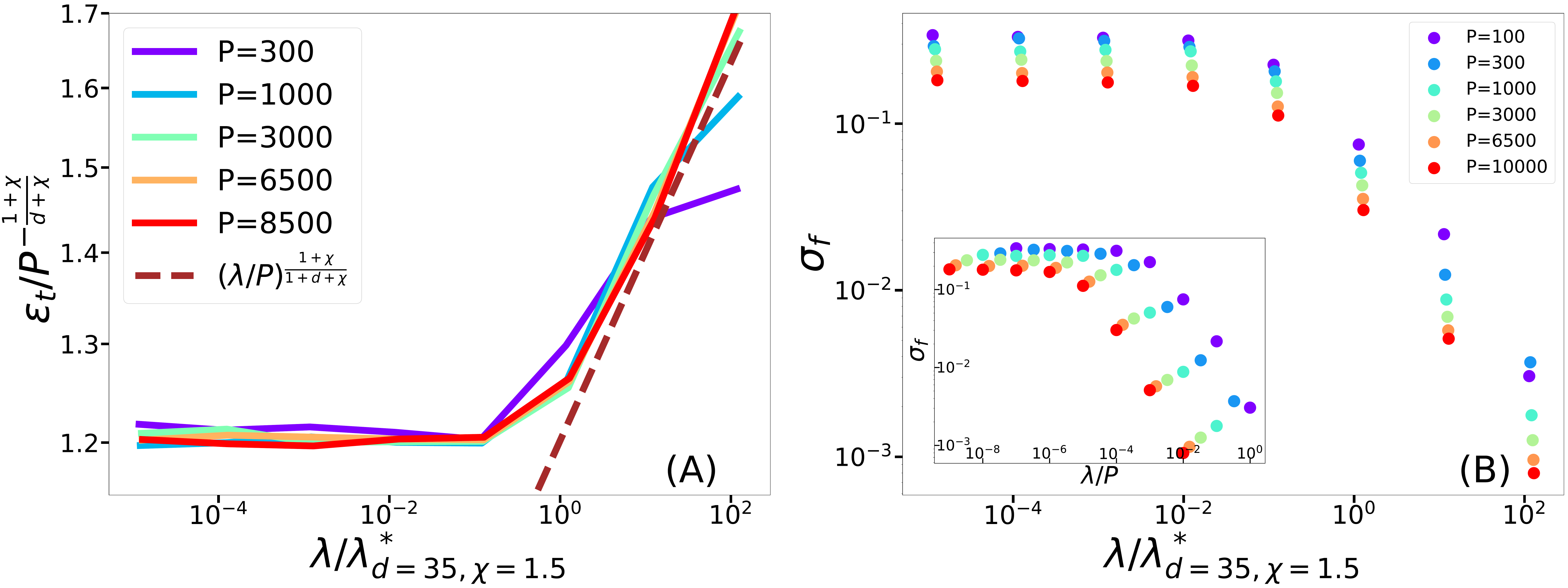}
		\caption{Binary CIFAR10. (A): Empirical test error $\varepsilon_t$ {\it v.s. } ridge. Each quantity is rescaled by our predictions  \eqref{errT-scaling} and \eqref{lstarD} for $d=35$ and $\chi=1.5$. The dashed brown line is the scaling prediction of the test error with respect to $\lambda$ of \eqref{errBlambda}. (B) Inset: variance of the predictor $\sigma_f$ {\it v.s.}   re-scaled ridge $\lambda/P$. Main plot: After rescaling the ridge by $\lambda^*_{d=35,\chi=1.5}$, curves nearly collapse.
		}
		\label{realdata_cifar10}
	\end{figure*}
    
    {\bf Numerical tests:} We consider the case $d=2$  and $\chi=1$. Fig. \ref{figd3} (A) shows the test error {\it v.s.} the ridge, both rescaled by our predictions  Eqs. \eqref{lstarD}, \eqref{errT-scaling}. The collapse is excellent, supporting the validity of both predictions. The prediction of Eq. \eqref{errB-scaling} is also indicated, and still shows an excellent agreement with observation. Fig. \ref{figd3} (B) reveals that once again, the cross-over ridge $\lambda^*_{d,\chi}$ where the replica method breaks down corresponds to a predictor $f_P$ that does not self-average near the decision boundary.


    {\bf Real data:} In Fig. \ref{realdata_cifar10}, we show the same quantities for the binary CIFAR-10 dataset (the 10 classes are grouped in two). The behavior of the test error as a function of the ridge is well-fitted by our model of decision boundaries,  taking $d=35$ (the intrinsic dimension of CIFAR \cite{Spigler2020}) and $\chi=1.5$ as shown in Fig. \ref{realdata_cifar10} (A). Remarkably, as shown in the inset of Fig. \ref{realdata_cifar10} (B), we also find that there exists a ridge-less regime where relative fluctuations of the predictor near  decision boundaries remain large ($\sigma_f>0.1$) for all $P$, from a regime where these fluctuations decay rapidly with increasing $P$. The curves $\sigma_f$ for different $P$ all collapse when the ridge is rescaled by $\lambda^*_{d,\chi}$ as predicted.  Note that in the ridge-less regime, we observe a very slight decay of $\sigma_f$ (twofold) as $P$ increases 100 folds, which signals that the geometry of decision boundaries is presumably more complex than in our model (which assumes, for example, that its properties are invariant when moving along them). 
    A similar behaviour is shown for the binary MNIST dataset in Appendix \ref{additionalFigures}.

	\section{Conclusion}
    We have shown that recent results based on replica or random matrix theory \cite{Bordelon2020, Jacot2020, Loureiro2021} can give excellent results even if data lie in low-dimension if the ridge is large enough.  However, together with other approaches \cite{Spigler2020} in the ridge-less case they lead to a spectral bias prediction. We showed that the latter does not apply for classification if the density of data between classes vanishes, except for $d\rightarrow\infty$. Ultimately,  these methods fail because the predictor is not self-averaging near the decision boundaries.  Quantitatively, however, predictions are already accurate for moderate dimensions.

	
	Finally, it is interesting to note that a vanishing density of data points implies a significant departure from the Gaussian assumption used in these approaches. Following \eqref{wkb4}, in $d=1$ the eigenvectors $\phi_\rho(x)$ are oscillating functions with envelope $\sim |x|^{-\chi/4}$ for small $x$. Thus, the probability distribution $P(\phi_\rho(x) =\phi)$ behaves as a power law $\sim \phi^{-5-\frac{4}{\chi}}$. Moreover, the eigenvectors are not independent for different $\rho$: they all have large values for small $x$, since their envelope is $|x|^{-\chi/4}$ for any $\rho$.
	
	
 \section*{Acknowledgements}
 We thank Francesco Cagnetta, Alessandro Favero, Mario Geiger, Bastien Olivier Marie Göransson, Leonardo Petrini and Lenka Zdeborová for helpful discussions. This work was supported by a grant from the Simons Foundation (\#454953 Matthieu Wyart).

\printbibliography


\newpage
\onecolumn
\appendix

\AtAppendix{\counterwithin{lemma}{section}}
\AtAppendix{\counterwithin{theorem}{section}}
\AtAppendix{\counterwithin{proposition}{section}}
\AtAppendix{\counterwithin{corollary}{section}}
\AtAppendix{\counterwithin{conjecture}{section}}

\section{Statistical mechanics of generalisation: spectral bias}\label{statmec}

	In \cite{Bordelon2020} a general formula for the test error \eqref{mse} has been derived , which requires the exact eigendecomposition of the kernel $K$ \eqref{eigenproblem}. To obtain a prediction $\varepsilon_B$ for the generalization error \eqref{mse}, the authors make two assumptions. First, they assume the test error $\varepsilon_t$ to be a self-averaging quantity with respect to the sampling of the training set. Second, they assume the probability distribution for the values of the eigenvectors $\phi_{\rho}$ over the training to be Gaussian. Given these assumptions, they derive via the replica method the following prediction for the test error:
	\begin{equation}
		\varepsilon_B =\sum\limits_{\rho=1}^{\infty} \frac{c_{\rho}^2}{\lambda_{\rho}^2}\left(\frac{1}{\lambda_{\rho}}+\frac{P}{\lambda+t(P)}\right)^{-2}\left(1-\frac{P\gamma(P)}{(\lambda+t(P))^2}\right)^{-1},
		\label{fullBordelon}
	\end{equation}
	where $\lambda$ is the ridge and:
	\begin{equation}
		t(P)=\sum\limits_{\rho}\left(\frac{1}{\lambda_{\rho}}+\frac{P}{\lambda+t(P)}\right)^{-1}, \qquad \gamma(P) = \sum\limits_{\rho}\left(\frac{1}{\lambda_{\rho}}+\frac{P}{\lambda+t(P)}\right)^{-2}
		\label{Bordelon-tP}
	\end{equation}
	
	It is important to notice that this prediction in the ridge-less case $\lambda=0$ is equivalent, for the scaling at large $P$, to choosing a ridge $\lambda/P$ that is of the same order of magnitude of the smallest eigenvalue $\lambda_P$ of the Gram matrix. To see this from Eq. \eqref{fullBordelon} it is sufficient to show that $t(P)/P \sim \lambda_P$ when $\lambda=0$. In this case, calling $\tilde{t}(P) = t(P)/P$, we can rewrite the definition of $t(P)$ in Eq. \eqref{Bordelon-tP} as
	\begin{align}
	    P = \sum\limits_{\rho} \frac{1}{1+\frac{\tilde{t}(P)}{\lambda_{\rho}}}
	\label{t-tilde}
	\end{align}
	The sum in the right hand side of Eq. \eqref{t-tilde} takes contributions of $O(1)$ for $\lambda_{\rho}\gg \tilde{t}(P)$ and contributions of $O(\lambda_{\rho}/\tilde{t}(P))$ for $\lambda_{\rho}\ll \tilde{t}(P)$. Since $\lambda_{\rho}$ decreases with $\rho$, this suggests that, for large $P$, $\tilde{t}(P)$ should be of the same order of $\lambda_{P}$ to have a sum of order $P$ at the right hand side of Eq. \eqref{t-tilde}. To see it more explicitly, we can consider an eigenvalue spectrum decaying as $\lambda_{\rho}\sim \rho^{-a}$ and we can approximate the sum in Eq. \eqref{t-tilde} with an integral:
	\begin{align}
	    P = \sum\limits_{\rho} \frac{1}{1+\frac{\tilde{t}(P)}{\lambda_{\rho}}} 
	    \sim \int_0^{\infty} d\rho \frac{1}{1+\tilde{t}(P) \rho^a} \propto \tilde{t}(P)^{-\frac{1}{a}}
	\end{align}
	which gives $\tilde{t}(P)\sim P^{-a}$, that is $\tilde{t}(P)\sim \lambda_P$.
	
	We have seen that the spectral bias prediction \eqref{predBord} does not work for vanishing $\lambda$. We may wonder about what happens for larger ridges. We compare the empirical test error $\varepsilon_t$ obtained from the experiments and the full prediction provided by \eqref{fullBordelon} for a large range of ridges $\lambda$. To compute the prediction \eqref{fullBordelon} we need:
	\begin{enumerate}[(i)]
		\item The exact eigenvalues $\lambda_\rho$, found via the self-consistent numerical scheme \eqref{sc1}. We computed them for ranks $\rho$ up to $5.1\cdot 10^4$. Since the scheme is valid for small $\lambda_\rho$, we replaced the first $10^3$ eigenvalues with the ones obtained diagonalising a large Gram matrix.
		\item The exact coefficients $c_\rho$, found projecting the solution $\phi_\rho$ of the differential equation \eqref{pde} onto the true function $f^*$ \eqref{trueFun}. We found them exactly for ranks $\rho<10^4$, then for ranks between $10^4$ and $5.1\cdot 10^4$ we extrapolated the value of $c^2_\rho$ doing a linear fit of $c_\rho^2$ with respect to $\rho$ for the first $10^4$ rank. 
	\end{enumerate}
	Once we have these ingredients, we can compute the prediction $\varepsilon_B$ provided by the full formula in \eqref{fullBordelon} for different training set sizes $P$, and compare it with the the empirical test error  with respect to the ridge $\lambda$, as we do for $\xi=0$ iFig. \ref{testError_departure_chi1_bord} (A) in the main text for $\chi =1$ and in Fig. \ref{testError_departure_chi0_bord} (A) for $\chi =0$. We can notice that:
	\begin{enumerate}[(1)]
		\item The prediction \eqref{fullBordelon}  for the scaling of $\varepsilon_t$ for fixed $P$ works for large ridge $\lambda$ and it breaks down lowering it. In section \ref{phaseDiagr} we argue that the crossover happens at $\lambda_{1,\chi}^*\sim P^{-\frac{1}{1+\chi}}$, as shown in Fig. \ref{testError_departure_chi1_bord} (B) in the main text for $\chi =1$ and in Fig. \ref{testError_departure_chi0_bord} (B) for $\chi =0$.
		\item The scaling of the prediction $\varepsilon_B$ with respect to $P$, given by \eqref{predBord}, captures the behaviour of the numerical results of $\varepsilon_B$ for small $\lambda$.
	\end{enumerate}

	
	\begin{figure*}
		\centering
			\includegraphics[width=1.\linewidth]{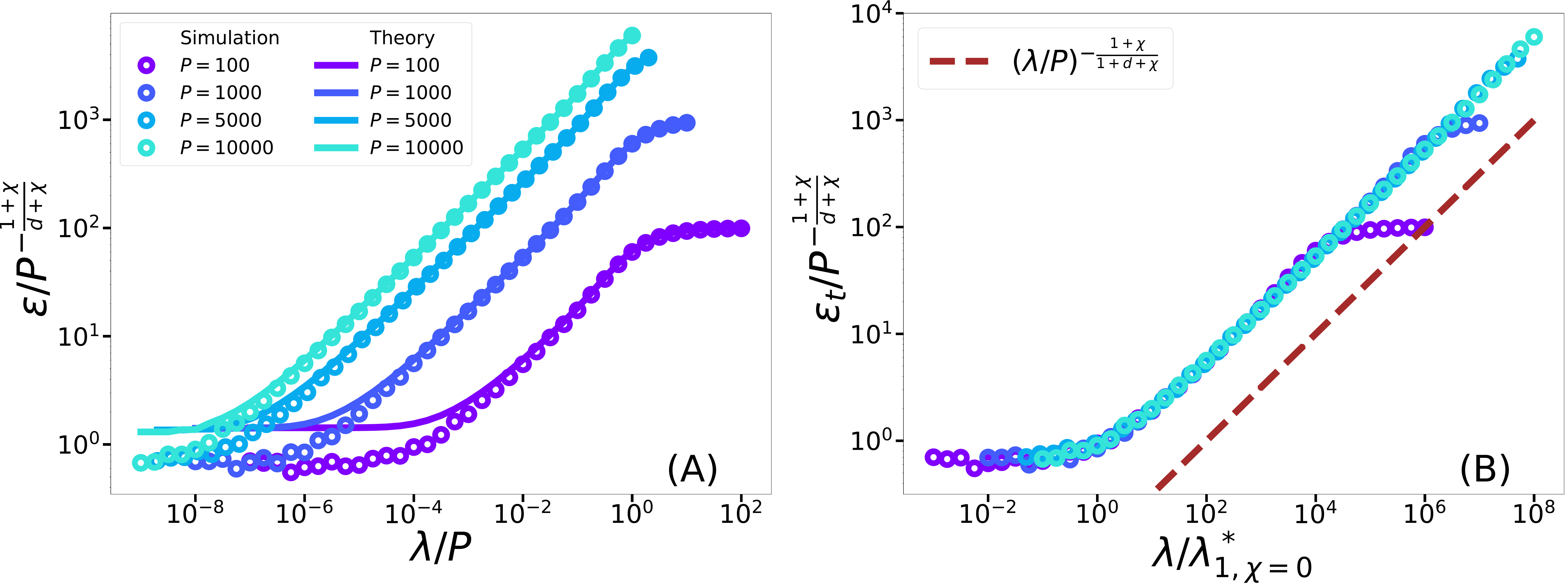}
		\caption{(A) Open symbols: empirical test error $\varepsilon_t$ (averaging over 200 realisations) rescaled by its ridgeless prediction \eqref{scalTestError}. Full lines: replica prediction $\varepsilon_B$ for fixed training set size $P$ and varying ridge $\lambda/P$. (B): the ridge has been rescaled by $\lambda^*_{1,\chi}$, defined in \eqref{lstar}. Brown line:  asymptotic behavior of $\varepsilon_B$ with $\lambda$ as predicted from Eq. \eqref{predBord}. }
		\label{testError_departure_chi0_bord}
	\end{figure*}

	\section{Kernel Alignment Risk Estimator} \label{kareApp}
	In this section we look at the results shown in \cite{Jacot2020}. In their work, the authors assume that, as far as one is interested in just the first two moments of the predictor $f_P$ \eqref{pred}, for any tuple of functions $(f_1,...,f_P)$ the vector of observations of these functions $(f_1(x_1),...,f_P(x_P))$ over $P$ points $\{x_i\}_{i=1,...,P}$ is a Gaussian vector. This Gaussianity Assumption includes also the eigenvectors $\phi_\rho$. As a consequence, it is possible to use rigorous Random Matrix Theory techniques for Gaussian matrices to obtain an estimate, called Kernel Alignment Risk Estimator (KARE), of the test error \eqref{mse} which depends just on the training data:
	\begin{equation}
		\varepsilon_{K} \sim \left\langle \frac{\frac{1}{P}(\vec{y})^T(K+\lambda \mathbb{1})^{-2}\vec{y}}{\left(\frac{1}{P}\text{Tr}\left[(K+\lambda \mathbb{1})^{-1}\right]\right)^2} \right\rangle,
		\label{kare}
	\end{equation}
	where the pedex $K$ stands for "KARE", the average is over different sampled sets, $\vec{y}$ is the vector of the labels in the training set and $K$ is the Gram matrix related to the $P$ samples $
	\{x_i\}$. The relation \eqref{kare} has a different prefactor in front of the Gram matrix with respect to the formula in \cite{Jacot2020}, which is due to our different definition of the training loss \eqref{min}. To obtain the relation \eqref{kare} they rely on some concentration results, whose fluctuations are controlled for values of the ridge $\lambda\rightarrow 0^+$ and training set size $P\rightarrow \infty$ such that $1/(\lambda\sqrt{P})\rightarrow 0^{+}$.
	
	We then test the prediction $\varepsilon_K$, comparing it with the empirical test error $\varepsilon_t$ with respect to the ridge $\lambda$ for fixed training set size $P$ in Fig. \ref{testError_departure_chi0_jacot} (A) for $\chi=0$ and in Fig. \ref{testError_departure_chi1_jacot} (A) for $\chi =0$. Both $\varepsilon_K$ and $\varepsilon_t$ are obtained averaging over 200 sampling realisations. We can see that the KARE prediction works for large $\lambda$, then it breaks down for small ridges, for fixed $P$. In section \ref{phaseDiagr} we argue that the crossover between the ridges where the KARE prediction works and where it does not is at $\lambda_{1,\chi}^*\sim P^{-\frac{1}{1+\chi}}$, as shown in Fig. \ref{testError_departure_chi0_jacot} (B) and in Fig. \ref{testError_departure_chi1_jacot} (B).
	

    \begin{figure*}
		\centering
			\includegraphics[width=1.\linewidth]{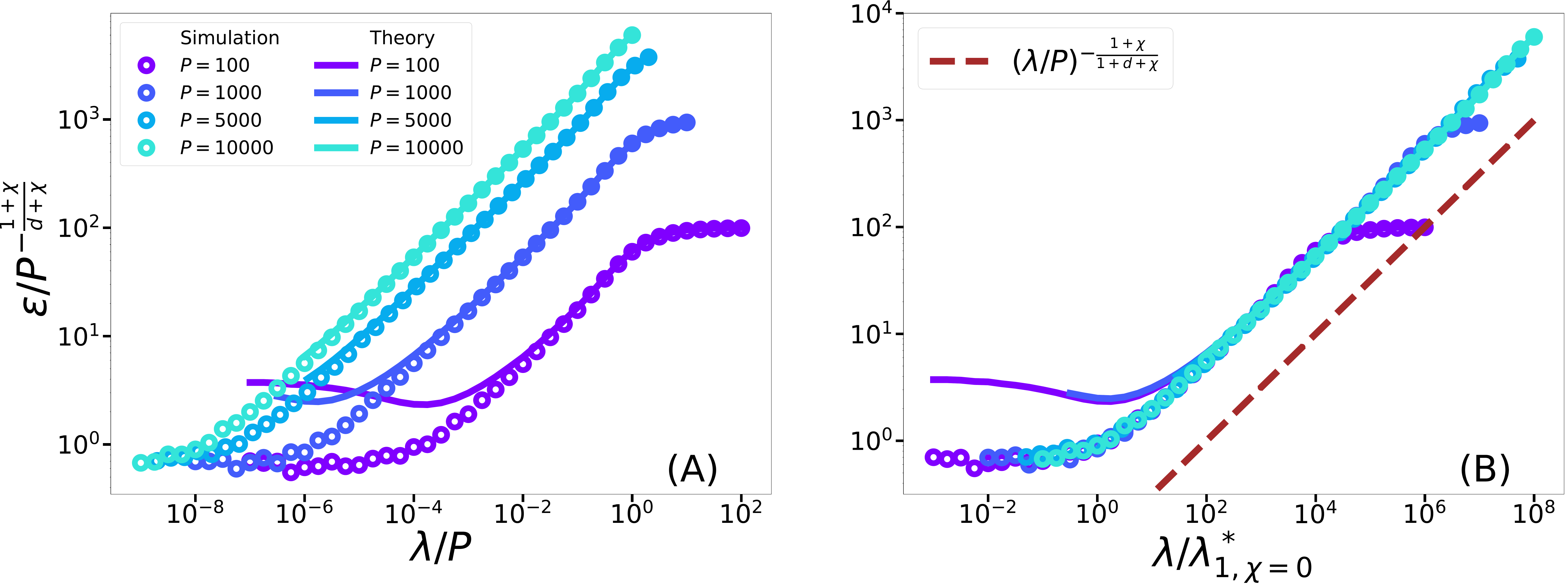}
		\caption{$d=1$, $\xi=0$, $\chi=0$. (A) Open symbols: empirical test error $\varepsilon_t$ (averaging over 200 realisations) rescaled by its ridgeless prediction \eqref{scalTestError}. Full lines: prediction $\varepsilon_K$ by \cite{Jacot2020} for fixed training set size $P$ and varying ridge $\lambda/P$. (B): the ridge has been rescaled by $\lambda^*_{1,\chi}$, defined in \eqref{lstar}. Brown line:  asymptotic behavior of $\varepsilon_t$ with $\lambda$ as predicted by replica prediction in Eq. \eqref{predBord}. }
		\label{testError_departure_chi0_jacot}
	\end{figure*}
	
	\begin{figure*}
		\centering
			\includegraphics[width=1.\linewidth]{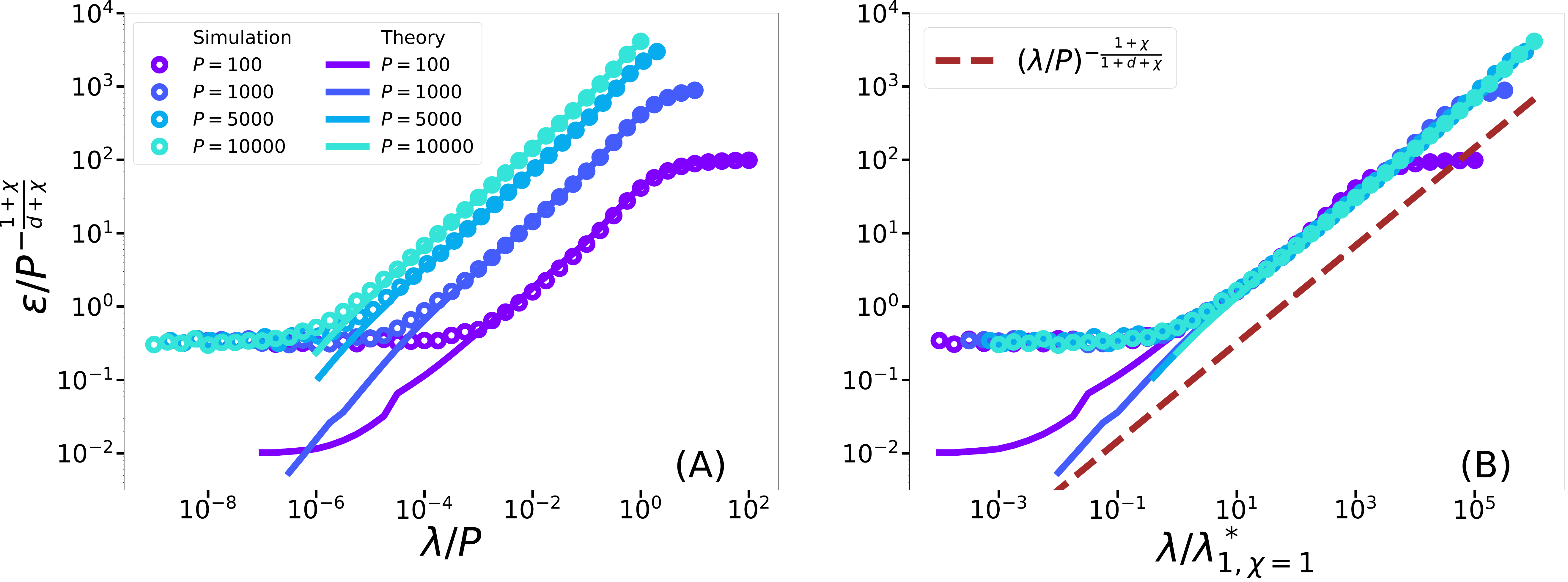}
		\caption{$d=1$, $\xi=0$, $\chi=1$. (A) Open symbols: empirical test error $\varepsilon_t$ (averaging over 200 realisations) rescaled by its ridgeless prediction \eqref{scalTestError}. Full lines: prediction $\varepsilon_K$ by \cite{Jacot2020} for fixed training set size $P$ and varying ridge $\lambda/P$. (B): the ridge has been rescaled by $\lambda^*_{1,\chi}$, defined in \eqref{lstar}. Brown line:  asymptotic behavior of $\varepsilon_t$ with $\lambda$ as predicted by replica prediction in Eq. \eqref{predBord}.}
		\label{testError_departure_chi1_jacot}
	\end{figure*}
	
	\section{No ridge test error Proofs and Numerics} \label{0ridgeproofs}
	
	\subsection{Proofs}
	
	\begin{theorem}[Test error]
		Consider a training set $\{x_i,f^*(x_i)\}_{i=1...P}$, where the samples $x_i$ are i.i.d. with respect to the PDF \eqref{pdf} and the true function $f^*_\xi$ is \eqref{trueFun}. In the limit of $P\rightarrow\infty$, the following asymptotic relation for the test error \eqref{mse} of KRR with Laplacian kernel $K(|x-y|)=\exp(-||x-y||_2 /\sigma)$ and ridge $\lambda\rightarrow0^+$ holds:
		\begin{equation}
			\begin{aligned}
				\varepsilon_t \sim& P^{-1+\frac{2\xi}{\chi+1}}
			\end{aligned}
		\end{equation}
	\end{theorem}
	\begin{proof}
		The sketch of the proof is the following.  We first find the form of the KRR predictor $f_P$ between any couple of sampled points $\{x_i,x_{i+1}\} $ which are neighbours. Then we estimate the amount of test error in the interval $[x_i,x_{i+1}]$, in the limit of large P. Lastly, we get the test error as the sum of all the contributions in such intervals. 
		
		Without loss of generality, we say that the sampled points are such that $x_1<x_2<...<x_{M}<0$ and $0<x_{M+1}<...<x_{P}$. Note that, in the asymptotic limit of large $P$, we expect:
		\begin{equation}
			M\sim P/2 \sim P, 
			\label{asymptoticM}
		\end{equation}
		due to the symmetry of the PDF \eqref{pdf}. It is now relevant to look at the scaling of the typical value of $x_M$ and $x_P$ with $P$, since it will be important to get the scaling of $\varepsilon_t$.
		
		\begin{lemma}
			Let's consider the sampled point $x_M$ and $x_{M+1}$, which are the closest to $x=0$. In the asymptotic limit of large $P$, the following scaling holds for their averages over the sampling:
			\begin{equation}
				\langle |x_M| \rangle \sim \langle |x_{M+1}| \rangle \sim P^{-\frac{1}{\chi+1}}. 
				\label{xm}
			\end{equation}
			Let's consider the sampled points $x_1$ and $x_{P}$, which are the most far from $x=0$. In the asymptotic limit of large $P$, it holds for their averages:
			\begin{equation}
				\langle |x_1| \rangle \sim \langle |x_{P}| \rangle \sim \sqrt{\log P}
				\label{xp}
			\end{equation}
		\end{lemma}
		\begin{proof}
			Let's start from $x_{M+1}$. We have that its typical value $\langle x_{M+1} \rangle$ is such that:
			\begin{equation}
				\frac{1}{P}\sim \int_{0}^{\langle x_{M+1} \rangle} dx\, p(x) \sim \int_{0}^{\langle x_{M+1} \rangle} dx\, x^{\chi}e^{-x^2},
				\label{a11}
			\end{equation}
			where $p(x)$ is given by \eqref{pdf}. The relation \eqref{a11} can be interpreted as if we expect to sample on average one point out of $P$ in the interval $[0,\langle x_{M+1}\rangle]$. For large $P$, we have $\langle x_{M+1}\rangle\ll 1$, hence we can write:
			\begin{equation}
				\frac{1}{P}\sim \int_{0}^{\langle x_{M+1} \rangle} dx\, x^{\chi} \propto \langle x_{M+1} \rangle^{\chi+1},
			\end{equation}
			which gives $\langle x_{M+1} \rangle = P^{-1/(\chi+1)}$, as in \eqref{xm}. The same logic can be applied to $x_M$.
			
			Now we consider $x_P$. Its typical value $\langle x_{P} \rangle$ will be such that:
			\begin{equation}
				\frac{1}{P}\sim \int_{\langle x_P\rangle }^{\infty}p(x)\,dx.
				\label{a20}
			\end{equation}
			Since we expect $\langle x_P\rangle$ to increase with $P$, we can rewrite the quantity $\int_{\langle x_P\rangle }^{\infty}p(x)\,dx$ as:
			\begin{equation}
				\int_{\langle x_P\rangle }^{\infty}p(x)\,dx \sim \int_{\langle x_P\rangle }^{\infty}e^{-x^2}\,dx,
			\end{equation}
			which is the $\text{erfc}$ function evaluated in $\langle x_P\rangle$. We can now make use of the following asymptotic expansion of the erfc function for large $x$:
			\begin{equation}
				\text{erfc}(x)=\frac{e^{-x^2}}{x\sqrt{\pi}}\left[1+\sum\limits_{n=1}^{\infty}(-1)^n\frac{(2n)!}{n!(2x)^{2n}}\right],
				\label{erfc}
			\end{equation}
			which at the leading order for large $x$ gives $\text{erfc}(x)\sim e^{-x^2}/x$. We can then rewrite \eqref{a20} as
			\begin{equation}
				\frac{1}{P}\sim \frac{e^{-\langle x_P\rangle ^2}}{\langle x_P\rangle}.
			\end{equation}
			Looking at the leading behaviour of $\langle x_P\rangle$ with respect to $P$, we find:
			\begin{equation}
				\langle x_P\rangle\sim \sqrt{\log P},
			\end{equation}
			which is \eqref{xp}. The same procedure can be applied to $x_1$.
		\end{proof}

		Now we consider the form of the predictor $f_P$.
		\begin{lemma}
			Given a couple of neighboring points $\{x_i,x_{i+1}\}$, we define $\Delta x_i = x_{i+1} - x_{i}$. The form of the predictor $f_P(x)$ \eqref{pred} in the interval $x \in [x_i,x_{i}+\Delta x_i]$ depends on $x_i$:
			\begin{itemize}
				\item For $x_i>0$:
				\begin{equation}
					\begin{aligned}
						f_P(x)= |x_i|^{-\xi}\left[A(x_i)e^{\frac{(x-x_i)}{\sigma}}+ B(x_i)e^{-\frac{(x-x_i)}{\sigma}}\right],
					\end{aligned}
					\label{a3}
				\end{equation}
				where:
				\begin{equation}
					A(x_i)=\left[\frac{\left(1-e^{-\frac{\Delta x_i}{\sigma}}\right)}{\left(1-e^{\frac{\Delta x_i}{\sigma}}\right)}\frac{\left((1+\frac{\Delta x_i}{x_i})^{-\xi}-e^{\frac{\Delta x_i}{\sigma}}\right)}{2\sinh\left(\frac{\Delta x_i}{\sigma}\right)}+\frac{1-(1+\frac{\Delta x_i}{x_i})^{-\xi}}{\left(1-e^{\frac{\Delta x_i}{\sigma}}\right)}\right]
					\label{y1}
				\end{equation}
				and:
				\begin{equation}
					B(x_i) = 			
					-\frac{\left((1+\frac{\Delta x_i}{x_i})^{-\xi}-e^{\frac{\Delta x_i}{\sigma}}\right)}{2\sinh\left(\frac{\Delta x_i}{\sigma}\right)}
					\label{y2}
				\end{equation}
				\item For $x_i<0$ and $x_i\neq x_M$, the predictor $f_P(x)$ has the form \eqref{a3}, with the functions $A(x_i)$ and $B(x_i)$ defined as in \eqref{y1} and \eqref{y2} with changed sign.
				\item For $x_i=x_M$, the predictor $f_P(x)$ has the same form of \eqref{a3} but with the functions $A(x_i)$ and $B(x_i)$ defined as:
				\begin{equation}
					A(x_i)=\left[\frac{\left(1-e^{-\frac{\Delta x_i}{\sigma}}\right)}{\left(1-e^{\frac{\Delta x_i}{\sigma}}\right)}\frac{\left((1+\frac{\Delta x_i}{x_i})^{-\xi}+e^{\frac{\Delta x_i}{\sigma}}\right)}{2\sinh\left(\frac{\Delta x_i}{\sigma}\right)}+\frac{1+(1+\frac{\Delta x_i}{x_i})^{-\xi}}{\left(1-e^{\frac{\Delta x_i}{\sigma}}\right)}\right]
					\label{y3}
				\end{equation}
				and:
				\begin{equation}
					B(x_i) = 			
					-\frac{\left((1+\frac{\Delta x_i}{x_i})^{-\xi}+e^{\frac{\Delta x_i}{\sigma}}\right)}{2\sinh\left(\frac{\Delta x_i}{\sigma}\right)}
					\label{y4}
				\end{equation}

			\end{itemize}
			
		\end{lemma}
		
		\begin{proof}
			The general form of the KRR predictor is:
			\begin{equation}
				f_P(x)=\sum_{i=1}^P\alpha_i K(x,x_i)=\sum_{i=1}^P\alpha_i e^{-||x-x_i||_2 /\sigma}, 
			\end{equation}
			
			with the coefficients $\alpha_i$ fixed by minimizing the training loss \eqref{min}. Since we are in the ridgeless limit $\lambda\rightarrow0^+$, the minimisation problem boils down to having the predictor $f_P$ fit the training set $\{x_i\}_{i=1...P}$.
			
			Let's consider $x\in]x_i,x_i+\Delta x_i[$. The predictor $f_P$ can be then rewritten as follows:
			\begin{equation}
				f_P(x) = \sum_{j=1}^{i}\alpha_j e^{-\frac{(x-x_j)}{\sigma}} + \sum_{j=i+1}^{P}\alpha_j e^{\frac{-(x_j-x)}{\sigma}}.
				\label{a1}
			\end{equation}
			If we derive \eqref{a1} two times with respect to $x$, we find the following differential equation satisfied by $f_P$:
			\begin{equation}
				\begin{aligned}
					f_P^{''}(x)& =  \frac{1}{\sigma^2}\left(\sum_{j=1}^{i}\alpha_j e^{-\frac{(x-x_j)}{\sigma}} + \sum_{j=i+1}^{P}\alpha_j e^{\frac{-(x_j-x)}{\sigma}}\right)\\
					& = \frac{1}{\sigma^2} f_P(x)
				\end{aligned}
				\label{a2}
			\end{equation}
			The solution of \eqref{a2} is given by the sum of two exponential functions, with coefficients $A_i$ and $B_i$:
			\begin{equation}
				f_P(x) = A_i e^{\frac{x}{\sigma}}+B_i e^{-\frac{x}{\sigma}}.
			\end{equation}
			The coefficients $A_i$ and $B_i$ are fixed by requesting that the predictor $f_P$ perfectly fits the true function $f^*_\xi(x)$ \eqref{trueFun} on the training set $\{x_i\}$. This requirement amounts to imposing the following boundary conditions:
			\begin{equation}
				f_P(x_i)=f^*_\xi(x_i),\qquad f_P(x_i+\Delta x_i)=f^*_\xi(x_i+\Delta x_i).
			\end{equation}
			Imposing these boundary conditions, the previously stated relations are found.
		\end{proof}
		
		We now look at the amount of test error \eqref{mse} done by the predictor given by \eqref{a3} in a generic interval $[x_i,x_i+\Delta x_i]$.
		
		\begin{lemma}
			We define the amount of test error in the interval $[x_i,x_i+\Delta x_i]$, for $i<P$, as follows:
			\begin{equation}
				\varepsilon_{x_i}=\int_{x_i}^{x_i+\Delta x_i} dx p(x)|f_P(x)-f^*(x)|^2.
				\label{contribute}
			\end{equation}
			where $p(x)$ is given by \eqref{pdf}. In the asymptotic limit of small $\Delta x_i$ (which is equivalent to the asymptotic limit of large $P$), we have for $\xi>0$ that:
			\begin{equation}
				\varepsilon_{x_i}\sim p(x_i) \frac{(\Delta x_i)^3}{x_i^2}|x_i|^{-2\xi},
				\label{a91}
			\end{equation}
			while for $\xi=0$:
			\begin{itemize}
				\item For $x_i\neq x_M$:
				\begin{equation}
					\varepsilon_{x_i}\sim p(x_i)\frac{(\Delta x_i)^5}{\sigma^4}|x_i|^{-2\xi}.
					\label{a9}
				\end{equation}
				\item For $x_i= x_M$:
				\begin{equation}
					\varepsilon_{x_M}\sim |x_M|^{\chi+1-\xi}.
					\label{a15}
				\end{equation}
			\end{itemize}
		\end{lemma}
		\begin{proof}
			Let's first consider the $\xi>0$ case. The contribute $\varepsilon_{x_i}$ \eqref{contribute}  for $x_i>0$ can be written as follows, using \eqref{a3}:
			\begin{equation}
				\begin{aligned}
					\varepsilon_{x_i}&=\int_{x_i}^{x_i+\Delta x_i} dx \,p(x)|f_P(x)-f^*_\xi(x)|^2\\
					&\sim \int_{x_i}^{x_i+\Delta x_i} dx \,p(x)\,|x_i|^{-2\xi}\,\left| \left[ \frac{\left((1+\frac{\Delta x_i}{x_i})^{-\xi}-1\right)}{\frac{\Delta x_i}{\sigma}}\right]\sinh{\left(\frac{x-x_i}{\sigma}\right)}+\right.\\
					&\qquad\qquad\qquad\qquad\qquad\qquad\qquad\qquad\left.+\cosh{\left(\frac{x-x_i}{\sigma}\right)}-\left|\frac{x}{x_i}\right|^{-\xi}\right|^2,
				\end{aligned}
			\end{equation}
			where the second equation has been obtained expanding the function \eqref{a3} for small $\Delta x_i$ with respect to $\sigma$. We now change variable $y=x-x_i$, obtaining:
			\begin{equation}
				\begin{aligned}
					\varepsilon_{x_i}&\sim \int_{0}^{\Delta x_i} dy \,p(y+x_i)\,|x_i|^{-2\xi}\,\left|\frac{y}{\sigma}\left[ \frac{\left((1+\frac{\Delta x_i}{x_i})^{-\xi}-1\right)}{\frac{\Delta x_i}{\sigma}}\right]+1+\frac{y^2}{\sigma^2}-\left|\frac{y+x_i}{x_i}\right|^{-\xi}\right|^2
				\end{aligned}
				\label{a61}
			\end{equation}
			where we have expanded the $\sinh$ and the $\cosh$ for small $y$ with respect to $\sigma$. Now we have two cases:
			\begin{itemize}
				\item One case where the increment $\Delta x_i$ is small with respect to $x_i$. This happens for a number of order $P$ of sampled points in the training set. In this case we can expand $p(y+x_i)$, $(1+\frac{\Delta x_i}{x_i})^{-\xi}$ and $(1+\frac{y}{x_i})^{-xi}$ in $y$ or $\Delta x_i$ with respect to $x_i$ in \eqref{a61}, obtaining at the leading order in $\Delta x_i$:
				\begin{equation}
					\begin{aligned}
						\varepsilon_{x_i}&\sim \int_{0}^{\Delta x_i} dy \,p(x_i)\,|x_i|^{-2\xi}\,\left(\xi\frac{y}{x_i}+\frac{y^2}{\sigma^2}\right)^2\\
						&\sim p(x_i) \frac{(\Delta x_i)^3}{x_i^2} \,|x_i|^{-2\xi}
					\end{aligned}
				\end{equation}
				which is the relation \eqref{a91}.
				\item Another case where the increment $\Delta x_i$ is of the same order in $P$ with respect to $x_i$. This happens just for a few points around $x=0$, and the number of these points is of order $P^0= \mathcal{O}(1)$. Since these points $x_i$ are close to 0, we have:
				\begin{equation}
					\begin{aligned}
						p(x_i+y)&\sim |x_i+y|^{\chi}= x_i^{\chi}\left|1+\frac{y}{x_i}\right|^{\chi}\\
						&\sim p(x_i)\left|1+\frac{y}{x_i}\right|^{\chi}
					\end{aligned}
					\label{a81}
				\end{equation}
				Plugging \eqref{a81} in \eqref{a61} and noticing that the quantity $\left((1+\frac{\Delta x_i}{x_i})^{-\xi}-1\right)$ is a constant in $P$, we rewrite \eqref{a61} at the leading order in $\Delta x_i\sim x_i$:
				\begin{equation}
					\varepsilon_{x_i}\sim \int_{0}^{\Delta x_i} dy \,p(x_i)\,|x_i|^{-2\xi}\,\left|1+\frac{y}{\Delta x_i}-(1+\frac{y}{x_i})^{-\xi}\right|^2,
				\end{equation}
				which yields:
				\begin{equation}
					\varepsilon_{x_i}\sim p(x_i)\,|x_i|^{-2\xi} \Delta x_i,
				\end{equation}
				which is consistent with the wanted relation \eqref{a91}, since $\Delta x_i\sim x_i$.
			\end{itemize}
			For $\xi>0$ and $x_i<0$ we can repeat the same procedure, getting relation \eqref{a91} for the contribute $\varepsilon_{x_i}$. 
			
			For $\xi=0$ and $x_i\neq x_M$ we can repeat the procedure done above for $\xi>0$ up to \eqref{a61}. Then, we notice that imposing $\xi=0$ in \eqref{a61}, we get:
			\begin{equation}
				\varepsilon_{x_i}\sim \int_{0}^{\Delta x_i} dy \,p(y+x_i)\frac{y^4}{\sigma^4}.
			\end{equation}
			We can then repeat the study of the two different cases done for $\xi>0$, depending on whether the increment $\Delta x_i$ is of the same order or not with respect to $x_i$. We then get \eqref{a9}.
			
			For $\xi\ge0$ and $x_i= x_M$ the contribute $\varepsilon_{x_M}$ \eqref{contribute} can be rewritten using the form \eqref{y3} and \eqref{y4} of the predictor and expanding for small $\Delta x_M/\sigma$:
			\begin{equation}
				\begin{aligned}
					\varepsilon_{x_M}&\sim \int_{x_M}^{x_M+\Delta x_M} dx \,p(x)\,|x_M|^{-2\xi}\,\left| \left[ \frac{\left((1+\frac{\Delta x_i}{x_i})^{-\xi}+1\right)}{\frac{\Delta x_i}{\sigma}}\right]\sinh{\left(\frac{x-x_i}{\sigma}\right)}+\right.\\
					&\qquad\qquad\qquad\qquad\qquad\qquad\qquad\qquad\left.-\cosh{\left(\frac{x-x_i}{\sigma}\right)}-f_\xi^*(x)\right|^2.
				\end{aligned}
				\label{a32}
			\end{equation}
			We can split the integral in \eqref{a32} in two parts, one from $x_M$ to 0 and another one from 0 to $x_{M+1}$. We analyze the first part, the second part can be analysed similarly. Since $x_M\sim P^{-\frac{1}{\chi+1}}\ll 1$, we can write \eqref{a32} in the following form, with the change of variable $y=x-x_M$:
			\begin{equation}
				\varepsilon_{x_M}\sim \int_0^{-x_M}dy\,  (y+x_M)^{\chi}\left|2\frac{y}{\Delta x_i}+\left|\frac{y+x_M}{x_M}\right|^{-\xi}-1\right|^2,
			\end{equation}
			which yields at the leading order in $P$:
			\begin{equation}
				\varepsilon_{x_M}\sim |x_M|^{\chi-2\xi}\frac{|x_M|^3}{(\Delta x_M)^2},
			\end{equation}
			which is consistent with both \eqref{a91} and \eqref{a15}.
		\end{proof}
		
		Now that we have the contributes \eqref{contribute} to the test error in a given interval $[x_i,x_i+\Delta x_i]$, we want to sum over them to get the behaviour of the full test error with respect to $P$. Before getting to that, we prove an intermediate result about the average spacing between two neighbouring points $x_i$ and $x_i+\Delta x_i$:
		\begin{lemma}
			Given a couple of neighbouring points $x_i$ and $x_i+\Delta x_i$, for $i<P$ and $i\neq M$, the average distance between them in the asymptotic limit of large $P$ is given by:
			\begin{equation}
				\langle \Delta x_i \rangle \sim \frac{1}{P p(x_i)},
				\label{spacing}
			\end{equation}
			where the average is over the sampling of the training set.
		\end{lemma}
		\begin{proof}
			On average, we expect that between $\langle x_i \rangle $ and $\langle x_i +\Delta x_i\rangle $ there is one sampled point out of $P$:
			\begin{equation}
				\frac{1}{P}\sim \int_{\langle x_i \rangle}^{\langle x_i +\Delta x_i\rangle}p(x)\,dx.
			\end{equation}
			Since we are considering large $P$, we have:
			\begin{equation}
				\frac{1}{P}\sim p(x_i) \langle \Delta x_i \rangle,
			\end{equation}
			which gives \eqref{spacing}.
		\end{proof}

		\begin{lemma}
			In the asymptotic limit of large $P$, the test error \eqref{mse} can be rewritten as follows:
			\begin{equation}
				\varepsilon_t = \sum_{i=1}^{P-1}\varepsilon_{x_i} + \int_{x_P}^{\infty}p(x)|f_P(x)-f_\xi^*(x)|^2\,dx + \int_{-\infty}^{x_1}p(x)|f_P(x)-f_\xi^*(x)|^2\,dx,
				\label{full}
			\end{equation}
			where $\varepsilon_{x_i}$ is defined in \eqref{contribute}.
			Then, the following holds:
			\begin{equation}
				\begin{aligned}
					\varepsilon_t \sim  P^{-\left(\frac{\chi+1-2\xi}{\chi+1}\right)}.
				\end{aligned}
				\label{result1}
			\end{equation}
		\end{lemma}
		\begin{proof}
			The relation \eqref{full} is immediate from the the definition of the test error \eqref{mse} and of the contributes \eqref{contribute}. 
			
			Consider the dependence on $P$ of the second term in \eqref{full}. It is always possible to bound from above $|f_P(x)-f_\xi^*|^2$ with a positive constant $C$. Then:
			\begin{equation}
				\int_{x_P}^{\infty}p(x)|f_P(x)-f_\xi^*(x)|^2\,dx \le C  \int_{x_P}^{\infty}p(x)dx.
			\end{equation}
			The right hand side of this relation is exactly (up to the constant $C$) the definition \eqref{a20} of the typical value of $x_P$. As a consequence, we have that the contribute of the second term to the test error is of order smaller or equal to $P^{-1}$. The same applies for the third term in \eqref{full}:
			\begin{equation}
				\int_{x_P}^{\infty}p(x)|f_P(x)-f_\xi^*(x)|^2\,dx + \int_{-\infty}^{x_1}p(x)|f_P(x)-f_\xi^*(x)|^2\,dx \le \frac{C_1}{P},
				\label{x1}
			\end{equation}
			where $C_1>0$.
			
			Now we consider the contributes corresponding to the first term in \eqref{full}. We start from $\xi=0$. The scaling with respect to $P$ of the contributes $\varepsilon_{x_i}$ can be of four different kinds. 
			
			\begin{itemize}
				\item We have a number of order $\mathcal{O}(P)$ of contributes $\varepsilon_{x_i}$ such that $p(x_i)$ does not scale with $P$. These contributes correspond to points $x_i$ sampled in the bulk of the distribution \eqref{pdf}. Combining \eqref{a9} and \eqref{spacing} we have then that the contribution of these terms to the full test error in \eqref{full} is:
				\begin{equation}
					P\cdot\frac{1}{P^5 \sigma^4}\sim  \frac{1}{P^4 \sigma^4},
					\label{x2}
				\end{equation}
				where the first factor $P$ stands for the number of contributes we are considering. 
				
				\item We have a number of order $\mathcal{O}(1)$ of contributes $\varepsilon_{x_i}$ related to points $x_i$ sampled with \eqref{pdf} close to $x=0$, and they are different from $x_M$. We expect that their typical value scales like $\langle x_i \rangle \sim P^{-\frac{1}{\chi+1}}$, as $x_M$ in \eqref{xm}. As a consequence, we have that for these points:
				\begin{equation}
					p(x_i)\sim |x_i|^{\chi}\sim P^{-\frac{\chi}{\chi+1}}.
					\label{a22}
				\end{equation}
				Combining \eqref{a9}, \eqref{spacing} and \eqref{a22} we obtain that the contribution of these terms to the full test error in \eqref{full} is:
				\begin{equation}
					\frac{1}{P^5 \sigma^4 p^4(x_i)}\sim  \frac{1}{P^5 \sigma^4 P^{-\frac{4\chi}{\chi+1}}}\sim \frac{1}{\sigma^4 P^{\left(1+\frac{4}{\chi+1}\right)}}.
					\label{x3}
				\end{equation}
				
				\item There is a number of order $\mathcal{O}(1)$ of contributes $\varepsilon_{x_i}$ related to points $x_i$ sampled with \eqref{pdf} in the tail of the Gaussian. Their typical value will scale with $P$ as $\sqrt{\log P}$, as $x_P$ in \eqref{xp}. Then, disregarding logarithmic factors in $P$:
				\begin{equation}
					p(x_i)\sim e^{-x_i^2}\sim\frac{1}{P}.
					\label{a23}
				\end{equation}
				Combining \eqref{a9}, \eqref{spacing} and \eqref{a23} we have then that the contribution of these terms to the full test error in \eqref{full} is, disregarding logarithmic factors in $P$:
				\begin{equation}
					\frac{1}{P^5 \sigma^4 p^4(x_i)}\sim  \frac{1}{P^5 \sigma^4 P^{-4}}\sim \frac{1}{\sigma^4 P}.
					\label{x4}
				\end{equation}
				
				\item Now we look at the contribution $\varepsilon_{x_M}$. Combining \eqref{a15} and \eqref{xm} we have:
				\begin{equation}
					\varepsilon_{x_M}\sim \frac{1}{P}
					\label{x5}
				\end{equation}
			\end{itemize}
			Summing over the contributions \eqref{x1}, \eqref{x2}, \eqref{x3}, \eqref{x4} and \eqref{x5}, we obtain that the leading contribute in $P$ to the test error is given by the relation \eqref{result1}. 
			
			Now we look at the $\xi>0$ case. In this case we have three different types of contributes $\varepsilon_{x_i}$.
			\begin{itemize}
				\item For the number $\mathcal{O}(P)$ of contributes $\varepsilon_{x_i}$ where $p(x_i)$ and $x_i$ do not scale with $P$, we obtain combining \eqref{a91}  and \eqref{spacing} the following contribution to $\varepsilon_t$:
				\begin{equation}
					P\cdot \frac{1}{P^3}\sim \frac{1}{P^2}.
					\label{a111}
				\end{equation}
				
				\item For the number $\mathcal{O}(1)$ of contributes $\varepsilon_{x_i}$ where $\langle x_i \rangle \sim P^{-\frac{1}{\chi+1}}$, we have that the contribution to the test error is given by:
				\begin{equation}
					\frac{1}{P^3}\frac{1}{x_i^{2\chi+2+2\xi}}\sim P^{-\frac{\chi+1-2\xi}{\chi+1}}.
					\label{a112}
				\end{equation}
				
				\item For the number $\mathcal{O}(1)$ of contributes $\varepsilon_{x_i}$ where $\langle x_i\rangle\sim \sqrt{\log P}$, the contribute to $\varepsilon_t$ is (disregarding logarithmic factors in $P$):
				\begin{equation}
					\frac{p(x_i)}{P^3 p^3(x_1)}\sim \frac{1}{P},
					\label{a113}
				\end{equation}
				using \eqref{a23}.
				
			\end{itemize}
			Combining the contributes \eqref{a111}, \eqref{a112} and \eqref{a113} we get that the leading contribute in $P$ to the test error $\varepsilon_t$ is given by the  relation \eqref{result1}. 
			
		\end{proof}

	\end{proof}
	\vspace{-1em}
	\subsection{Numerics}
	
	
	Considering $\frac{\chi+1}{2}>\xi>0$ and $\chi=2$ and $\chi=4$, the prediction \eqref{scalTestError} still holds, as shown in Fig. \ref{empTestError_xiMagg0}, realised for $\sigma=100$ and $\lambda = 10^{-12}$. e use as a ridge $\lambda = 10^{-12}$ and not exactly 0 to avoid numerical instabilities due to the inversion of the matrix in \eqref{pred}. We choose $\lambda = 10^{-12}$ because it is smaller than the values of the eigenvalues of the Gram matrices used in \eqref{pred}, but it is large enough to avoid the instabilities. In Appendix \ref{num} there are further details for the sampling scheme used for the training and test sets.

	\begin{figure}
		\centering
		\includegraphics[width=0.49\linewidth]{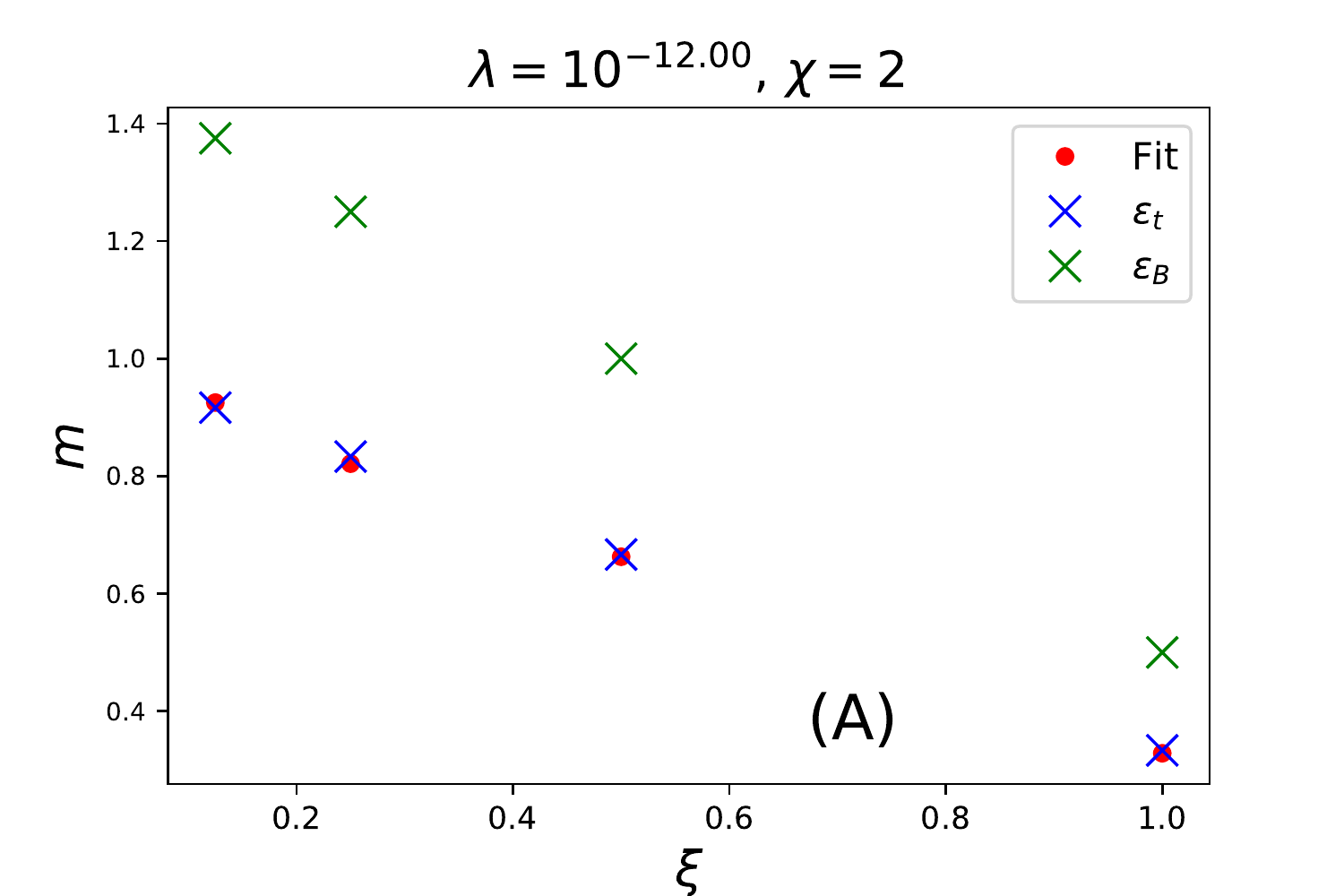}
		\includegraphics[width=0.49\linewidth]{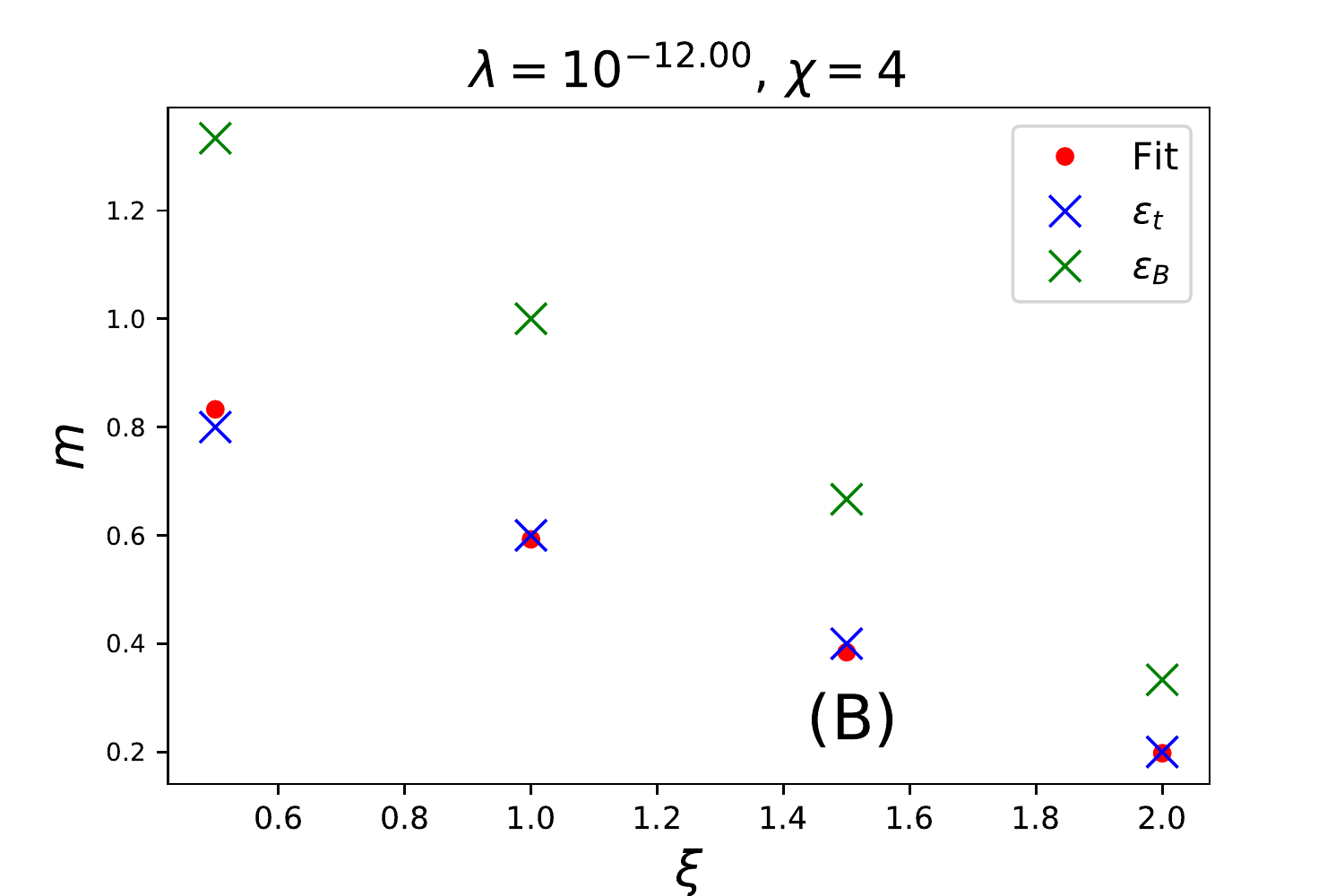}
		\caption{$d=1$, ridge $\lambda=10^{-12}$, $\chi=2$ (A) and $\chi=4$ (B). Exponent $m$ of the relation $\varepsilon\sim P^{-m}$, where $\varepsilon$ is the test error and $P$ the training set size, with respect to $\xi$ of the true function \eqref{trueFun}. Red points: exponent from a fit of the simulations, for $P\in[10,10^4]$, averaged over 20 realizations. Blue crosses: exponent of $\varepsilon_t$ from the prediction \eqref{scalTestError}. Green crosses: exponent of $\varepsilon_B$ from the spectral bias prediction \eqref{predBord}.}
		\label{empTestError_xiMagg0}
	\end{figure}
	
	\section{Eigendecomposition Proofs and Numerics}\label{Eigenproofs}
	\subsection{Proofs}
	We remind, as presented in Th. \ref{thmpde} in the main text, that the eigenvectors $\phi_\rho$ of the Laplace kernel satisfy the following differential equation:
	\begin{equation}
			\phi_{\rho}^{''}(x)=\left(-2\frac{p(x)}{\lambda_{\rho} \sigma} +\frac{1}{\sigma^2}\right)\phi_{\rho}(x).
			\label{pdeApp}
		\end{equation}
	
	We proceed to solve the differential equation \eqref{pde} to get an explicit form of the $\phi_\rho$. Then we compute the coefficients $c_\rho$ obtained decomposing the true function $f^*$ onto the eigenvectors:
    \begin{equation}
		c_{\rho} = \int dx\, p(x) f^*_\xi(x)\phi_{\rho}(x)
		\label{coeff}
	\end{equation}
	Lastly, we obtain a numerical scheme to get small eigenvalues $\lambda_\rho$.

	\begin{proposition}(Coefficients) \label{coffPropApp}
		Let $K$ be the Laplacian kernel with width $\sigma>0$: $K(x,y)=K(|x-y|)=\exp(-||x-y||_2 /\sigma)$. Let $p(x)$ be \eqref{pdf}. Consider a small eigenvalue $\lambda_{\rho}\ll 1$. Let $\phi_\rho$ be the solution of \eqref{pdeApp}. We impose that $\phi_\rho(x)\rightarrow0$ for $|x|\rightarrow\infty$. Then the following holds for the coefficient $|c_\rho|$ defined in \eqref{coeff}, in the limit $\lambda_{\rho}\ll 1$:
		\begin{equation}
			\begin{aligned}
				|c_{\rho}|\sim& \lambda_{\rho}^{\frac{\frac{3}{4}\chi+1-2\xi}{\chi+2}}& \text{if }&\phi_{\rho} \text{ is odd}\\
				|c_{\rho}| = & 0 &\text{   if }&\phi_{\rho} \text{ is even}.
			\end{aligned}
			\label{scal1coeffApp}
		\end{equation}
	\end{proposition}
	\begin{proof}
		Let's consider the differential equation \eqref{pdeApp}, satisfied by the eigenvectors $\phi_{\rho}$. We will solve that differential equation for small $\lambda_{\rho}$, then we will compute the integral which defines $c_{\rho}$ \eqref{coeff} at the leading order in $\lambda_{\rho}$.

		\begin{lemma}[Eigenvectors for $\chi>0$] \label{evecChiMagg0}
			Let $\phi_\rho$ be a solution of \eqref{pdeApp} for a given $\lambda_\rho$. For $\chi>0$ and $x>0$, let $x_1$ and $x_2$ be the roots of the function:
			\begin{equation}
				\Gamma^{2}(x)=\frac{2}{\lambda_{\rho}\sigma}p(x)-\frac{1}{\sigma^2},
				\label{gamma}
			\end{equation}
			where $p(x)$ is given by \eqref{pdf}.  Let $\text{Ai}$ and $\text{Bi}$ be the Airy functions of first and second kind \cite{Florentin1966}. We impose that $\phi_\rho(x)\rightarrow0$ for $|x|\rightarrow\infty$. If $\phi_\rho$ is odd in $x$, then some positive coefficients $\alpha,\, \beta,\, \zeta,\, \delta_1,\,$ $\delta_2$ exist, independent on $\lambda_\rho$, such that $\phi_{\rho}$ is approximated at the leading order in $\lambda_{\rho}$ by the following definition by parts, for $x>0$:
			\begin{align}
			    \phi_{\rho}(x) \simeq
			    \begin{cases}
			        \phi_{\rho}^{(I)}(x), \qquad \text{for }\  x\in [0,\beta \lambda_{\rho}^{\frac{1}{\chi}}]\\
			        \phi_{\rho}^{(II)}(x), \qquad \text{for } x\in [\beta \lambda_{\rho}^{\frac{1}{\chi}},x_2-\frac{\delta_1}{\sqrt{-\log \lambda_{\rho}}}]\\
			        \phi_{\rho}^{(IV)}(x), \qquad \text{for } x\in [x_2-\frac{\delta_1}{\sqrt{-\log \lambda_\rho}},x_2+\frac{\delta_2}{\sqrt{-\log \lambda_\rho}}]\\
			        \phi_{\rho}^{(V)}(x), \qquad \text{for } x\in [x_2+\frac{\delta_2}{\sqrt{-\log \lambda_\rho}},\infty] 
			    \end{cases}
			\end{align}
			
			with:
			
			\begin{align}
				\phi_{\rho}^{(I)}(x)=
				\frac{\alpha}{\lambda_{\rho}^{1/12}}\left(\text{Ai}(\mu-\nu x)-\gamma_1\text{Bi}(\mu-\nu x) \right)
			\label{c1}
			\end{align}
			
			\begin{align}
				\phi_{\rho}^{(II)}(x)\simeq
                \frac{\alpha}{\lambda_{\rho}^{1/4}}\left(\text{Ai}(\xi(x))-\gamma_1\text{Bi}(\xi(x)) \right)\frac{|\xi(x)|^{1/4}}{|\Gamma^2(x)|^{1/4}}
				\label{c2}
			\end{align}			

			\begin{align}
				\phi_{\rho}^{(IV)}(x)=W_1 \text{Ai}\left[\left(\frac{2 x_2}{\sigma^2}\right)^{1/3}(x-x_2)\right] + W_2 \text{Ai}\left[\left(\frac{2 x_2}{\sigma^2}\right)^{1/3}(x-x_2)\right]
				\label{c4}
			\end{align}

			\begin{align}
					\phi_{\rho}^{(V)}(x)\sim\frac{\alpha(\sin\theta -\gamma_1\cos\theta)}{2\sqrt{\pi}(-p(x)+\lambda_{\rho})^{1/4}}\exp\left(-\int_{x_2}^x \sqrt{-\Gamma^2(z)}dz\right)
				\label{c5}
			\end{align}

			where we have introduced the notation
			$\mu=\left(\frac{\chi (\lambda_\rho\Gamma[\frac{1+\chi}{2}])^{\frac{2}{\chi}}}{2^{\frac{2}{\chi}}\sigma^{2(1+\chi)}}\right)^{1/3}$,
			$\nu=\left(\frac{2^{1-\frac{1}{\chi}}\chi}{\sigma^{2-\frac{1}{\chi}}\lambda_\rho^{\frac{1}{\chi}}}\right)^{1/3}$,
			$\gamma_1 = \text{Ai}(\mu)/\text{Bi}(\mu)$,
			$\xi(x)=-\left[\frac{3}{2}\int_{x_1}^x\sqrt{\Gamma^2(z)}dz\right]^{2/3}$,
			$\Gamma^2(x)$ defined in \eqref{gamma},
			$\theta = \int_{x_1}^{x_2}\sqrt{\Gamma^2(z)}dz+\frac{\pi}{4}$.
			The coefficients $W_1$ and $W_2$ are found matching the solution parts \eqref{c2}, \eqref{c4} and \eqref{c5}. They are such that $W_1\sim W_2 \sim \lambda_{\rho}^{-1/4}$.\\
			This definition by parts is to be interpreted as made of two matching parts $\phi_{\rho}^{(I)}$ and $\phi_{\rho}^{(IV)}$ around the roots $x_1$ and $x_2$ of the function $\Gamma^2(x)$, a bulk part $\phi_{\rho}^{(II)}$ for $x_1<x<x_2$, and a tail part $\phi_{\rho}^{(V)}$ for $x\gg x_2$. In particular, the bulk part $\phi_{\rho}^{(II)}$ \eqref{c2} can be simplified for $x$ far from the roots $x_1$ and $x_2$. More precisely, in the interval $[\zeta \lambda_{\rho}^{\frac{1}{(2+\chi)}}, x_2-\frac{\delta_1}{\sqrt{-\log \lambda_\rho}}]$, it holds $\phi_{\rho}^{(II)}(x)\simeq \phi_{\rho}^{(III)}(x)$ with:
			
			\begin{align}					   
			    \phi_{\rho}^{(III)}(x)\sim \frac{\alpha}{(p(x)-\lambda_{\rho})^{1/4}}\left(\sin\left(\int_{x_1}^x \sqrt{\Gamma^2(z)}dz+\frac{\pi}{4}\right)-\right.
				\left. +\gamma_1\cos\left(\int_{x_1}^x \sqrt{\Gamma^2(z)}dz+\frac{\pi}{4}\right) \right).
			\label{c3}
			\end{align}

			If $\phi_\rho$ is instead even in $x$, and such that $\phi^{'}_\rho(0)=0$, then it has the same form as the equations \eqref{c1}-\eqref{c5} for $x>0$, but the coefficient $\gamma_1$ is defined as $\text{Ai}'(\mu)/\text{Bi}'(\mu)$.
			
			\begin{proof}
				In this proof we will consider for simplicity the following form of the differential equation \eqref{pdeApp}:
				\begin{equation}
					\phi_{\rho}^{''}(x)+\tilde{\Gamma}^{2}(x)\phi_{\rho}(x)=0,
					\label{pde2}
				\end{equation}
				with $\tilde{\Gamma}^2$ defined as follows:
				\begin{equation}
					\tilde{\Gamma}^2(x)=\frac{1}{\lambda_{\rho}}|x|^{\chi}e^{-x^2}-1,
					\label{gamma2}
				\end{equation}
				where we get rid of the numerical coefficients 2, $\sigma$ and $\Gamma[\frac{1+\chi}{2}]$ present in \eqref{pdeApp} and \eqref{gamma}. Then in the final results \eqref{c1}-\eqref{c5} we put back these coefficients. 
				
				We will start with the odd eigenvectors $\phi_{\rho}$. The first thing we remark is that the function $\tilde{\Gamma}^2$ has two roots for $\chi>0$:
				\begin{equation}
					x_{1}\sim \left(\lambda_\rho \right)^{1/\chi}, \qquad x_2 \sim \sqrt{-\log\lambda_\rho}.
					\label{x1x2full}
				\end{equation}   
				To obtain the first piece of the eigenvector \eqref{c1}, we expand the function $\tilde{\Gamma}^2(x)$ in $x$ around $x_1$:
				\begin{equation}
					\begin{aligned}
						\tilde{\Gamma}^2(x) \sim & e^{-\lambda_{\rho} ^{2/\chi }}  \left(\chi -2 \lambda_{\rho} ^{2/\chi }\right) \lambda_{\rho} ^{-\frac{1}{\chi }} \left(x-\lambda_{\rho} ^{1/\chi }\right)+\\
						&+\frac{1}{2} e^{-\lambda_{\rho} ^{2/\chi }}  \left(-2 (2 \chi +1) \lambda_{\rho} ^{2/\chi }+4 \lambda_{\rho} ^{4/\chi }+(\chi -1) \chi \right) \lambda_{\rho} ^{-\frac{ 2}{\chi }} \left(x-\lambda_{\rho} ^{1/\chi }\right)^2
					\end{aligned}
					\label{trunc1}
				\end{equation}
				We want to truncate this expansion at first order in $\left(x-\lambda_{\rho} ^{1/\chi }\right)$. The solution of \eqref{pde2} with the truncated expansion will be equal to the full solution form 0 up to a certain $x^*$. We find $x^*$ as the point such that the second order of the expansion \eqref{trunc1} is of the same order of the first order, at the leading order in $\lambda_{\rho}$. For $\chi\neq1$, the comparison of the two orders in \eqref{trunc1} yields:
				\begin{equation}
					\lambda_{\rho} ^{-1/\chi }(x-\lambda_{\rho} ^{1/\chi })\sim \lambda_{\rho} ^{-2/\chi }(x-\lambda_{\rho} ^{1/\chi })^2,
				\end{equation}
				from which we have:
				\begin{equation}
					x^*\sim \lambda_{\rho} ^{1/\chi }
					\label{xstar}
				\end{equation}
				For $\chi=1$ the relation \eqref{xstar} holds again, since $x^*$ depends continuously on $\chi$. In other words, it exists a constant $\beta$, independent on $\lambda_\rho$, such that $x^*=\beta\lambda_{\rho} ^{1/\chi } > x_1$.
				
				For $x\in[0,x^*]$, the equation \eqref{pde2} becomes, at the leading order in $\lambda_\rho$:
				\begin{equation}
					\phi_{\rho}^{''}(x)+\lambda_{\rho} ^{-\frac{1}{\chi }} \left(x-\lambda_{\rho} ^{1/\chi }\right)\phi_{\rho}(x)=0,
				\end{equation}
				whose solution $\phi_\rho^{(I)}$, with boundary condition $\phi_{\rho}(0)=0$, is given by \eqref{c1}. The factor $\lambda_{\rho}^{-1/12}$ is due to the normalisation to 1 of the full $\phi_{\rho}$, as we will see later in the proof.
				
				The solution of \eqref{pde2}, for $x$ distant from the roots $x_1$ and $x_2$ of $\tilde{\Gamma}^2$, can be found by means of the technique of the Modified Airy Function (MAF) \cite{Ghatak1991}. The solution of \eqref{pde2} is given by:
				\begin{equation}
					\phi_\rho^{(II)}(x)\sim
					\begin{aligned}
						\left(Q_1 \text{Ai}(\xi(x))+Q_2\text{Bi}(\xi(x)) \right)\frac{|\xi(x)|^{1/4}}{|\Gamma^2(x)|^{1/4}}, 
					\end{aligned}
					\label{cc2}
				\end{equation}
				where $Q_1$ and $Q_2$ are constants to be determined and $\xi(x)=-\left[\frac{3}{2}\int_{x_1}^x\sqrt{\tilde{\Gamma}^2(z)}dz\right]^{2/3}$. The solution \eqref{cc2} holds up for the points $x$ such that the following inequality is valid:
				\begin{equation}
					\left|(\xi')^{-3}\left(\frac{3(\xi'')^2}{4\xi'}-\frac{1}{2}\xi'''\right)\right|\ll |\xi|.
					\label{controlled}
				\end{equation}
				The relation \eqref{controlled} holds for $x\in[x^*,x_2-\frac{\delta_1}{\sqrt{\log \lambda_\rho}}]$, since in that interval we have, at the leading order in $\lambda_\rho$:
				\begin{equation}
					\begin{aligned}
						\xi(x)\sim \xi'(x) \sim \xi''(x) \sim \xi'''(x)\sim  \lambda_{\rho}^{-1/3}.
					\end{aligned}
				\end{equation}
				
				The constants $Q_1$ and $Q_2$ can be found matching the solution \eqref{cc2} with \eqref{c1}, taking care of the fact that:
				\begin{equation}
					\frac{\xi^{1/4}(x)}{\left(\tilde{\Gamma}^2(x)\right)^{1/4}}\xrightarrow[x\rightarrow (x^*)^+]{}\lambda_\rho^{1/6},
				\end{equation}
				then obtaining:
				\begin{equation}
					\phi_\rho^{(II)}(x)\sim
					\begin{aligned}
						\frac{\alpha}{\lambda_{\rho}^{1/4}}\left(\text{Ai}(\xi(x))-\gamma_1\text{Bi}(\xi(x)) \right)\frac{|\xi(x)|^{1/4}}{|\Gamma^2(x)|^{1/4}},
					\end{aligned}
					\label{ccc2}
				\end{equation}
				The relation \eqref{ccc2} can be approximated in the following way. The Airy function $\text{Ai}$ has the following asymptotic approximation for large negative $y$ \cite{Ghatak1991}:
				\begin{equation}
					\begin{aligned}
						\text{Ai}(y)\xrightarrow[y\rightarrow-\infty]{}&\frac{1}{\sqrt{\pi}y^{1/4}}\left[\sin\left(y^{\frac{3}{2}}+\frac{\pi}{4}\right)\sum\limits_{k=0}^{\infty}(-1)^k c_{2k}(y^{3/2})^{-2k}+\right.\\
						&\left. -\cos\left(y^{\frac{3}{2}}+\frac{\pi}{4}\right)\sum\limits_{k=0}^{\infty}(-1)^k c_{2k+1}(y^{3/2})^{-2k-1} \right],
					\end{aligned}
					\label{d3}
				\end{equation}
				where:
				\begin{equation}
					c_0=1,\quad c_k =\frac{\Gamma(3k+\frac{1}{2})}{54^k k! \Gamma(k+\frac{1}{2})},\, k\ge 1.
				\end{equation}
				As a consequence, for large $y$:
				\begin{equation}
					\text{Ai}(y)\sim \frac{1}{\sqrt{\pi}y^{1/4}}\sin\left(y^{\frac{3}{2}}+\frac{\pi}{4}\right).
					\label{d1}
				\end{equation}
				Similarly, it holds for $\text{Bi}$:
				\begin{equation}
					\text{Bi}(y)\sim \frac{1}{\sqrt{\pi}y^{1/4}}\cos\left(y^{\frac{3}{2}}+\frac{\pi}{4}\right).
					\label{d2}
				\end{equation}
				Using these asymptotic relations with $y=\xi(x)$, we find that when they are valid the solution $\phi_\rho^{(II)}$ \eqref{ccc2} can be approximated by $\phi_\rho^{(III)}$ \eqref{c3}. We remark that the solution \eqref{c3} can be found in the literature under the name of WKB approximation \cite{Ghatak1991} and it is of interest in the field of quantum mechanics.

				We want to define better the interval where we can use the approximation \eqref{c3} of $\phi_{\rho}$. The upper bound of that integral will be given by the limit of validity of \eqref{ccc2}, hence $\frac{\delta_1}{\sqrt{-\log \lambda_{\rho}}}$. The lower bound of that interval will be given by the $x=\hat{x}$ such that the zero and first order in the expansion \eqref{d3} are of the same order, hence when:
				\begin{equation}
					1\sim \frac{1}{(\xi(\hat{x}))^{3/2}}
					\label{e1}
				\end{equation}
				For small $x$ and small $\lambda_{\rho}$, using \eqref{x1x2full}, we have the following asymptotic relation:
				\begin{equation}
					\int_{x_1}^{\hat{x}}\left(\frac{x^\chi}{\lambda_{\rho}}e^{-x^2}-1\right)^{1/2}\sim \frac{\hat{x}^{\frac{\chi}{2}+1}}{\sqrt{\lambda_{\rho}}}.
					\label{e2}
				\end{equation}
				Combining \eqref{e1} and \eqref{e2}, we get:
				\begin{equation}
					\hat{x}\sim \lambda_{\rho}^{\frac{1}{2+\chi}},
				\end{equation}
				hence it exists a $\zeta>0$, independent on $\lambda_{\rho}$, such that $\hat{x}=\zeta \lambda_{\rho}^{\frac{1}{2+\chi}}$.
				
				Since the MAF solution \eqref{c2} is not valid around $x_2$, to get the form of $\phi_\rho$ in that region we linearize $\tilde{\Gamma}^2(x)$ around $x_2$, and then we solve exactly the differential equation in the region where this approximation is valid. Since we are looking at large $x$ (since $x_2\sim\sqrt{-\log\lambda_\rho}$), we can approximate $x^\chi e^{-x^2}$ in \eqref{gamma2} with $e^{-x^2}$. The expansion of $\tilde{\Gamma}^2(x)$ around $x_2$ up to second order then yields:
				\begin{equation}
					\tilde{\Gamma}^2(x) \sim -2 \sqrt{-\log (\lambda_{\rho} )} \left(x-\sqrt{-\log (\lambda_{\rho} )}\right)+(-2 \log (\lambda_{\rho} )-1) \left(x-\sqrt{-\log (\lambda_{\rho} )}\right)^2.
					\label{trunc4}
				\end{equation}
				The truncation at the first order of \eqref{trunc4} is valid in a region around $x_2$ such that the at the boundaries of that region the second order in \eqref{trunc4} is of the same order in $\lambda_{\rho}$ of the first order. This happens for $x$ such that:
				\begin{equation}
					|x - \sqrt{-\log \lambda_{\rho}}|\sim \frac{1}{\sqrt{-\log \lambda_{\rho}}}
				\end{equation}
				Consequently, there exist coefficients $\delta_1$ and $\delta_2$ independent on $\lambda_{\rho}$ such that the differential equation \eqref{pde2} with $\tilde{\Gamma}^2(x)$ approximated up to the first order in \eqref{trunc4} has a solution \eqref{c4}:
				\begin{equation}
					\begin{aligned}
						\phi_{\rho}^{(IV)}(x)=W_1 \text{Ai}\left[\left(\frac{2 x_2}{\sigma^2}\right)^{1/3}(x+x_2)\right] + W_2 \text{Ai}\left[\left(\frac{2 x_2}{\sigma^2}\right)^{1/3}(x+x_2)\right], \\\qquad \text{for } x\in [x_2-\frac{\delta_1}{\sqrt{-\log \lambda_\rho}},x_2+\frac{\delta_2}{\sqrt{-\log \lambda_\rho}}]
					\end{aligned}
					\label{cc4}
				\end{equation}
				The coefficients $W_1$ and $W_2$ are found matching \eqref{cc4} with the solution parts of $\phi_\rho$ before and after its interval of validity. Since $\phi_{\rho}^{(III)}(x)$ in $x=x_2-\frac{\delta_1}{\sqrt{-\log\lambda}}$ has amplitude $\sim e^{\frac{1}{4}x_2^2}\sim \lambda_{\rho}^{-1/4}$, then $W_1\sim W_2 \sim \lambda_{\rho}^{-1/4}$.
				
				To find the solution for $x\ge \left(x_2+\frac{\delta_2}{\sqrt{-\log \lambda_\rho}}\right)$, we can make use of some formulae in the literature, called "connection formulae" \cite{Ghatak1991}, which map the WKB solution for $x< x_2$ (at left of the arrow) into the one for $x> x_2$ (at right):
				\begin{equation}
					\begin{aligned}
						\frac{2}{(\tilde{\Gamma}^{2}(x))^{1/4}}\sin\left[\int_x^{x_2}dx(\tilde{\Gamma}^{2}(x))^{1/2}+\frac{\pi}{4}\right] \rightarrow & \frac{2}{(-\tilde{\Gamma}^{2}(x))^{1/4}} \exp\left[ -\int^x_{x_2}dx(-\tilde{\Gamma}^{2}(x))^{1/2} \right]\\
						\frac{1}{(\tilde{\Gamma}^{2}(x))^{1/4}}\cos\left[\int_x^{x_2}dx(\tilde{\Gamma}^{2}(x))^{1/2}+\frac{\pi}{4}\right] \rightarrow & \frac{2}{(-\tilde{\Gamma}^{2}(x))^{1/4}} \exp\left[ \int^x_{x_2}dx(-\tilde{\Gamma}^{2}(x))^{1/2} \right]
					\end{aligned}
					\label{connection}
				\end{equation}
				The idea behind the proof of these formulae is to take the exact solution of the equation \eqref{pde2} in the interval close to $x_2$, and then expand the Airy functions at left and at right of $x_2$.
				Using the relations \eqref{connection}, the solution for $x\ge \left(x_2+\frac{\delta_2}{\sqrt{-\log \lambda_\rho}}\right)$ is:
				\begin{equation}
					\begin{aligned}
						\phi_{\rho}^{(V)}(x)\sim&\frac{\alpha(\sin\theta -\gamma_1\cos\theta)}{2\sqrt{\pi}(-x^\chi e^{-x^2}+\lambda_{\rho})^{1/4}}\exp\left(-\int_{x_2}^x \sqrt{-\tilde{\Gamma}^2(z)}dz\right)+\\
						&+\frac{\alpha(\gamma_1\sin\theta +\cos\theta)}{\sqrt{\pi}(-x^\chi e^{-x^2}+\lambda_{\rho})^{1/4}}\exp\left(\int_{x_2}^x \sqrt{-\tilde{\Gamma}^2(z)}dz\right),
					\end{aligned}
				\end{equation}
				where $\theta = \int_{x_1}^{x_2}\sqrt{\Gamma^2(z)}dz+\frac{\pi}{4}$. Now we impose that $\phi_{\rho}(x)\rightarrow0$ for $x\rightarrow \infty$, in order not to have exponentially divergent terms which would make the norm of $\phi_\rho$ infinite. This request is equivalent to imposing the condition:
				\begin{equation}
					\gamma_1\sin\theta +\cos\theta=0,
					\label{condition}
				\end{equation}
				which we will see in a different Proposition that it fixes the eigenvalues $\lambda_{\rho}$. Imposing this boundary condition, we find $\phi_{\rho}^{(V)}$ as in \eqref{c5}.
				
				What it is left is the proof of the factor $\lambda_{\rho}^{-1/12}$ present in \eqref{c1}, which then fixes the dependence on $\lambda_{\rho}$ of the normalisation coefficients of the relations \eqref{c2}-\eqref{c5}. We show now that the factor is such that the eigenvector $\phi_\rho$ is normalised to $1$. More specifically, we show that the norm of $\phi_\rho$ does not depend on $\lambda_\rho$ at the leading order in $\lambda_\rho$, and then the constant $\alpha$ in \eqref{c1}-\eqref{c5}, independently on $\lambda_\rho$, fixes the norm of $\phi_\rho$ to 1.
				
				The norm $||\phi_{\rho}^2||_p$ of $\phi_\rho$ is defined as follows:
				\begin{equation}
					||\phi_{\rho}||_p =\int_{-\infty}^{\infty} dx\,p(x) \phi^2_\rho(x).
					\label{norm}
				\end{equation}
				We restrict to $x\ge 0$ since the integrand is even in $x$ and we divide this norm in five pieces, analysing them one by one, disregarding the numerical factors and looking just at the behaviour with respect to $\lambda_\rho$ for clarity purposes.
				\begin{itemize}
					\item For $x\in[0,\beta \lambda_\rho^{\frac{1}{\chi}}]$  we use \eqref{c1}. The contribute to the norm \eqref{norm} is the following :
					\begin{equation}
						\int_{0}^{\beta \lambda_\rho^{\frac{1}{\chi}}}dx\,x^\chi e^{-x^2}\left(\phi_{\rho}^{(I)}(x)\right)^2.
						\label{f1}
					\end{equation}
					Squaring \eqref{c1} we get four terms. We analyze one of them and the same logic can be applied to the other three terms, since at the leading order in $\lambda_\rho$ the factor $\gamma_1$ is of order $\mathcal{O}(1)$.
					\begin{equation}
						\frac{1}{\lambda_\rho^{1/6}} \int_{0}^{\beta \lambda_\rho^{\frac{1}{\chi}}}dx\,x^\chi e^{-x^2} \text{Ai}^2(\lambda_\rho^{\frac{2}{3\chi}}-x\lambda_\rho^{-\frac{1}{3\chi}}) 
					\end{equation}
					We do the change of variable $y = x\lambda_\rho^{-\frac{1}{3\chi}}$:
					\begin{equation}
						\lambda_\rho^{\frac{1}{6}+\frac{1}{3\chi}}\int_{0}^{\beta \lambda_\rho^{\frac{2}{3\chi}}}dy y^{\chi}e^{-y^2 \lambda_\rho^{\frac{2}{3\chi}}}\text{Ai}^2(\lambda_\rho^{\frac{2}{3\chi}}-y).
						\label{f22}
					\end{equation}
					Since we are interested at the leading order of \eqref{f22} in $\lambda_\rho$:
					\begin{equation}
						\lambda_\rho^{\frac{1}{6}+\frac{1}{3\chi}}\int_{0}^{\beta \lambda_\rho^{\frac{2}{3\chi}}}dy y^{\chi}e^{-y^2 \lambda_\rho^{\frac{2}{3\chi}}}\text{Ai}^2(\lambda_\rho^{\frac{2}{3\chi}}-y)\sim 
						\lambda_\rho^{\frac{1}{6}+\frac{1}{3\chi}}\int_{0}^{\beta \lambda_\rho^{\frac{2}{3\chi}}}dy y^{\chi}\text{Ai}^2(-x)
					\end{equation}
					Since the function $\text{Ai}$ is continuous in the integration interval, we can bound \eqref{f1} from above with a constant $F>0$. Then we have that:
					\begin{equation}
						\lambda_\rho^{\frac{1}{6}+\frac{1}{3\chi}}\int_{0}^{\beta \lambda_\rho^{\frac{2}{3\chi}}}dy y^{\chi}\text{Ai}^2(-x)\le F\lambda_\rho^{\frac{5}{6}+\frac{1}{\chi}}.
					\end{equation}
					Hence:
					\begin{equation}
						\int^{\beta \lambda_\rho^{\frac{1}{\chi}}}_{0} dx\,x^\chi e^{-x^2}\left(\phi_{\rho}^{(I)}(x)\right)^2\le F_1\lambda_\rho^{\frac{5}{6}+\frac{1}{\chi}},
						\label{g1}
					\end{equation}
					with $F_1>0$.
					\item For $x\in[\beta \lambda_\rho^{\frac{1}{\chi}},\zeta \lambda_\rho^{\frac{1}{2+\chi}}]$ we use \eqref{c2}. As for $\phi^{(I)}$, we study one of the four terms we get doing the square of $\phi^{(II)}$:
					\begin{equation}
						\begin{aligned}
							\frac{1}{\lambda_{\rho}^{1/2}} &\int_{\beta \lambda_\rho^{\frac{1}{\chi}}}^{\zeta \lambda_\rho^{\frac{1}{2+\chi}}}dx\,x^\chi e^{-x^2}\text{Ai}^2(\xi(x))\frac{|\xi(x)|^{1/2}}{|\Gamma^2(x)|^{1/2}}\sim\\ &\frac{1}{\lambda_{\rho}^{1/6}} \int_{\beta \lambda_\rho^{\frac{1}{\chi}}}^{\zeta \lambda_\rho^{\frac{1}{2+\chi}}} dx\,x^\chi \frac{\left(\int_{x_1}^{x}(z^{\chi}-\lambda_{\rho})^{\frac{1}{2}}dz\right)^{1/3}}{\left(x^\chi -\lambda_\rho\right)^{1/2}} \text{Ai}^2\left(\frac{1}{\lambda_\rho^{1/3}}\left(\int_{x_1}^{x}(z^{\chi}-\lambda_{\rho})^{\frac{1}{2}}dz\right)^{2/3}\right)
						\end{aligned}
						\label{f7}
					\end{equation}
					At the leading order in $\lambda_\rho$ we can do the following approximation, recalling \eqref{x1x2full}:
					\begin{equation}
						\int_{x_1}^{x}(z^{\chi}-\lambda_{\rho})^{\frac{1}{2}}dz \sim x^{\frac{\chi}{2}+1},
						\label{relaz1}
					\end{equation}
					which we insert in \eqref{f7}. Then we perform in \eqref{f7} the substitution $y = x - \beta \lambda_\rho^{\frac{1}{\chi}}$, getting at the leading order in $\lambda_\rho$:
					\begin{equation}
						\frac{1}{\lambda_{\rho}^{1/6}}\int^{\zeta \lambda_\rho^{\frac{1}{2+\chi}}}_{0}dy\, y^{\chi+\frac{1}{3}(\frac{\chi}{2}+1)-\frac{\chi}{2}}\text{Ai}^2\left(\frac{1}{\lambda_\rho^{1/3}} y^{\frac{2}{3}(\frac{\chi}{2}+1)}\right).
						\label{f8}
					\end{equation}
					We then do the substitution $w = \frac{y^{\frac{2}{3}(\frac{\chi}{2}+1)}}{\lambda_\rho^{1/3}}$ in \eqref{f8}, getting:
					\begin{equation}
						\lambda_\rho^{\frac{1}{2}} \int_{0}^{1} dw \,w^{\frac{\chi}{2+\chi}}\text{Ai}^2(w)\sim \lambda_\rho^{\frac{1}{2}},
					\end{equation}
					hence obtaining:
					\begin{equation}
						\int_{\beta \lambda_\rho^{\frac{1}{\chi}}}^{\zeta \lambda_\rho^{\frac{1}{2+\chi}}} dx\,x^\chi e^{-x^2}\left(\phi_{\rho}^{(II)}(x)\right)^2\sim \lambda_\rho^{\frac{1}{2}}.
						\label{g2}
					\end{equation}
					
					\item  For $x\in[\zeta \lambda_\rho^{\frac{1}{2+\chi}},x_2-\frac{\delta_1}{\sqrt{-\log \lambda_\rho}}]$  we use \eqref{c3}. Squaring $\phi_{\rho}^{(III)}$ we get four terms. We focus on one of them, and the logic we will use can be applied also to the other three terms. We consider then:
					\begin{equation}
						\int^{x_2-\frac{\delta_1}{\sqrt{-\log \lambda_\rho}}}_{\zeta \lambda_\rho^{\frac{1}{2+\chi}}}dx\,x^\chi e^{-x^2}\frac{1}{(x^\chi e^{-x^2}-\lambda_{\rho})^{1/2}}\sin^2\left(\int_{x_1}^x \sqrt{\Gamma^2(z)}dz+\frac{\pi}{4}\right).
						\label{f2}
					\end{equation}
					We can replace the $\sin^2(...)$ with $\frac{1}{2}(1-\cos(2...))$, where $(...)$ is the argument of the $\sin^2$ in \eqref{f2}. We look at the first term which comes from this substitution, at the leading order in $\lambda_\rho$:
					\begin{equation}
						\frac{1}{2} \int^{x_2-\frac{\delta_1}{\sqrt{-\log \lambda_\rho}}}_{\zeta \lambda_\rho^{\frac{1}{2+\chi}}}dx\,x^\chi e^{-x^2}\frac{1}{(x^\chi e^{-x^2}-\lambda_{\rho})^{1/2}}\sim \int_{0}^{\infty}x^{\chi/2} e^{-\frac{1}{2}x^2}\sim\mathcal{O}(1).
						\label{f3}
					\end{equation}
					For the second term:
					\begin{equation}
						\begin{aligned}
							\frac{1}{2} \int^{x_2-\frac{\delta_1}{\sqrt{-\log \lambda_\rho}}}_{\zeta \lambda_\rho^{\frac{1}{2+\chi}}}dx &\,x^\chi e^{-x^2}\frac{1}{(x^\chi e^{-x^2}-\lambda_{\rho})^{1/2}}\sin\left(\int_{x_1}^x \sqrt{\Gamma^2(z)}dz+\frac{\pi}{4}\right)\\ 
							&\le \frac{1}{2} \int^{x_2-\frac{\delta_1}{\sqrt{-\log \lambda_\rho}}}_{\zeta \lambda_\rho^{\frac{1}{\chi}}}dx\,x^\chi e^{-x^2}\frac{1}{(x^\chi e^{-x^2}-\lambda_{\rho})^{1/2}} = \mathcal{O}(1).
						\end{aligned}
						\label{f4}
					\end{equation}
					Putting together \eqref{f3} and \eqref{f4}, and repeating the logic for the other terms coming from squaring $\phi_{\rho}^{(III)}$, we have:
					\begin{equation}
						\int^{x_2-\frac{\delta_1}{\sqrt{-\log \lambda_\rho}}}_{\zeta \lambda_\rho^{\frac{1}{2+\chi}}}dx\,x^\chi e^{-x^2} \left(\phi_{\rho}^{(III)}(x)\right)^2= \mathcal{O}(1)
						\label{g3}
					\end{equation}
					
					\item For $x\in[x_2-\frac{\delta_1}{\sqrt{-\log \lambda_\rho}}, x_2+\frac{\delta_2}{\sqrt{-\log \lambda_\rho}}]$ we use \eqref{c4}. As above, we consider just one of the four terms we get doing the square of $\phi_\rho^{(IV)}$, and the same procedure can be applied to the other three:
					\begin{equation}
						W_1^2 \int_{x_2-\frac{\delta_1}{\sqrt{-\log \lambda_\rho}}}^{x_2+\frac{\delta_2}{\sqrt{-\log \lambda_\rho}}} dx\,x^\chi e^{-x^2}\text{Ai}^2\left[x_2^{1/3}(x-x_2)\right].
					\end{equation}
					Since $x_2\sim \sqrt{-\log\lambda_{\rho}}\gg 1$, we can do the following approximation, where we also made the substituion $y=x-x_2$ and used the fact that $W_1\sim \lambda_\rho^{-1/4}$:
					\begin{equation}
						\begin{aligned}
							\frac{1}{\lambda_\rho^{1/2}}\int_{x_2-\frac{\delta_1}{\sqrt{-\log \lambda_\rho}}}^{x_2+\frac{\delta_2}{\sqrt{-\log \lambda_\rho}}} &dx\, e^{-x^2}\text{Ai}^2\left[x_2^{1/3}(x-x_2)\right]\sim\\ &\frac{1}{\lambda_\rho^{1/2}}\int_{-\frac{\delta_1}{\sqrt{-\log \lambda_\rho}}}^{\frac{\delta_2}{\sqrt{-\log \lambda_\rho}}} dy\, e^{-(y+x_2)^2}\text{Ai}^2\left[x_2^{1/3}y\right].
						\end{aligned}
					\end{equation}
					Noticing that $y\ll x_2$ and doing the substitution $w=x_2^{1/3}y$ we get:
					\begin{equation}
						\begin{aligned}
							\lambda_\rho^{1/2}\int_{-(\sqrt{-\log \lambda_\rho})^{2/3}}^{(\sqrt{-\log \lambda_\rho})^{2/3}} dw\, w^2\text{Ai}^2\left[w\right] \sim  \lambda_\rho^{1/2}.
						\end{aligned}
					\end{equation}
					Consequently:
					\begin{equation}
						\int_{x_2-\frac{\delta_1}{\sqrt{-\log \lambda_\rho}}}^{x_2+\frac{\delta_2}{\sqrt{-\log \lambda_\rho}}} dx\,x^\chi e^{-x^2}\left(\phi_{\rho}^{(IV)}(x)\right)^2\sim \lambda_\rho^{1/2}.
						\label{g4}
					\end{equation}
					
					\item For $x\ge x_2+\frac{\delta_2}{\sqrt{-\log \lambda_\rho}}$ we use \eqref{c5}, getting the following:
					\begin{equation}
						\int^{\infty}_{x_2+\frac{\delta_2}{\sqrt{-\log \lambda_\rho}}} dx\,x^\chi e^{-x^2}\left(\phi_{\rho}^{(V)}(x)\right)^2\le
						\int_{x_2+\frac{\delta_2}{\sqrt{-\log \lambda_\rho}}}^{\infty} dx\,x^{\frac{\chi}{2}} e^{-\frac{1}{2}x^2} \sim \int_{x_2}^{\infty} dx\, e^{-\frac{1}{2}x^2}.
					\end{equation}
					Then we use the expansion of the erfc function for large $x$ \eqref{erfc}, recalling that $x_2\sim \sqrt{-\log\lambda_\rho}$ \eqref{x1x2full}, getting:
					\begin{equation}
						\int_{x_2+\frac{\delta_2}{\sqrt{-\log \lambda_\rho}}}^{\infty} dx\,x^\chi e^{-x^2}\left(\phi_{\rho}^{(V)}(x)\right)^2 \le H \lambda_\rho^{1/2},
						\label{g5}
					\end{equation}
					for $H>0$.
				\end{itemize}
				Combining the contributes \eqref{g1}, \eqref{g2}, \eqref{g3}, \eqref{g4} and \eqref{g5} to the norm \eqref{norm}, we get that the norm of $\phi_\rho$ is of order $\mathcal{O}(1)$ independently on $\rho$, as wanted.
				
				For the even eigenvectors $\phi_\rho (x) = \phi_\rho (-x)$ everything in the proof above applies equally, expect for the definition of $\gamma_1$ in \eqref{c1}, which becomes:
				\begin{equation}
					\gamma_2 = \text{Ai}'(\mu)/\text{Bi}'(\mu),
				\end{equation}
				with $\mu=\left(\frac{\chi (\lambda_\rho\Gamma[\frac{1+\chi}{2}])^{\frac{2}{\chi}}}{2^{\frac{2}{\chi}}\sigma^{2(1+\chi)}}\right)^{1/3}$. This change in definition is due to the new the boundary condition we impose on $\phi_\rho$, which is no more $\phi_{\rho}(0)=0$ but $\phi'_{\rho}(0)=0$.
			\end{proof}
		\end{lemma} 	
		\begin{lemma}[Eigenvectors for $\chi=0$]\label{evecChi0}
			The eigenvectors $\phi_{\rho}$ have the same form as the relations \eqref{c2}-\eqref{c5}, but with the following replacement:
			\begin{equation}
				x_1 \rightarrow 0,\quad \beta \rightarrow 0,\quad \mu=0
			\end{equation}
		\end{lemma}
		\begin{proof}
			In the case $\chi =0$ the function $\tilde{\Gamma(x)}$ has just one root $x_2$. Then there is not the solution $\phi_{\rho}^{(I)}$ and the part of the solution $\phi_{\rho}^{(II)}$ applies also up to $x=0$, which gives $\beta=0$. As a consequence, the quantity $\gamma_1$ and $\gamma_2$ are defined replacing the argument $\mu$ of the Airy functions and their derivatives with $0$. Lastly, the integral in $\xi(x)$ in \eqref{c2} is defined starting from $x=0$ and not $x_1$.
		\end{proof}
		Now that we have an approximated form for the eigenvectors $\phi_\rho$, we can proceed to prove \eqref{scal1coeffApp}. We start considering $\chi>0$ and odd eigenvectors $\phi_\rho(x)=-\phi_\rho(-x)$. We can restrict the analysis of the integral \eqref{coeff} which defines $c_\rho$ to $x\ge0$, since the integrand is an even function, given by $p(x)f^*(x)\phi_\rho(x)$, where $p$ is given by \eqref{pdf} and $f^*$ by \eqref{trueFun}.
		
		The logic to compute $c_\rho$ is the same as the one used to compute the leading order of the norm $||\phi_\rho||_p$ in \eqref{norm}: we split the integral \eqref{coeff} in five pieces for $x>0$ and we compute their value at the leading order in $\lambda_\rho$. Into each piece of the integral we will use the relative approximation for $\phi_\rho$ found in relations \eqref{c1}-\eqref{c5}. We do the computations disregarding numerical factors and looking just at the main behaviour in $\lambda_\rho$.
		\begin{itemize}
			\item For $x\in[0,\beta \lambda_\rho^{\frac{1}{\chi}}]$  we use \eqref{c1}. The contribution to the coefficient \eqref{coeff} is the following :
			\begin{equation}
				\int_{0}^{\beta \lambda_\rho^{\frac{1}{\chi}}}dx\,x^{\chi-\xi} e^{-x^2}\phi_{\rho}^{(I)}(x).
				\label{h1}
			\end{equation}
			We analyze the scaling in $\rho$ of just one of the two terms we get inserting \eqref{c1} into \eqref{h1}. The same logic can be applied to the other term, since the factor $\gamma_1$ is of order $\mathcal{O}(1)$ in $\lambda_\rho$.
			\begin{equation}
				\frac{1}{\lambda_\rho^{1/12}} \int_{0}^{\beta \lambda_\rho^{\frac{1}{\chi}}}dx\,x^{\chi-\xi} e^{-x^2} \text{Ai}(\lambda_\rho^{\frac{2}{3\chi}}-x\lambda_\rho^{-\frac{1}{3\chi}}) 
			\end{equation}
			We do the change of variable $y = x\lambda_\rho^{-\frac{1}{3\chi}}$ and we look at the leading order in $\lambda_\rho$:
			\begin{equation}
				\lambda_\rho^{\frac{1}{4}+\frac{1}{3\chi}-\frac{\xi}{3\chi}}\int_{0}^{\beta \lambda_\rho^{\frac{2}{3\chi}}}dy y^{\chi}\text{Ai}(-x)
			\end{equation}
			
			Since the function $\text{Ai}$ is continuous in the integration interval, we can bound \eqref{f1} from above with a constant $K>0$. Then we have that:
			\begin{equation}
				\lambda_\rho^{\frac{1}{4}+\frac{1}{3\chi}-\frac{\xi}{3\chi}}\int_{0}^{\beta \lambda_\rho^{\frac{2}{3\chi}}}dy y^{\chi}\text{Ai}^2(-x)\le K\lambda_\rho^{\frac{11}{12}+\frac{1}{\chi}-\frac{\xi}{3\chi}}.
			\end{equation}
			Hence:
			\begin{equation}
				\int^{\beta \lambda_\rho^{\frac{1}{\chi}}}_{0} dx\,x^\chi e^{-x^2}\phi_{\rho}^{(I)}(x)\le K_1\lambda_\rho^{\frac{11}{12}+\frac{1}{\chi}-\frac{\xi}{3\chi}},
				\label{j1}
			\end{equation}
			with $K_1>0$. The exponent in the right hand side of \eqref{j1} is positive, since we have imposed in \eqref{trueFun} the condition $\xi<\frac{\chi+1}{2}$.
			\item For $x\in[\beta \lambda_\rho^{\frac{1}{\chi}},\zeta \lambda_\rho^{\frac{1}{2+\chi}}]$ we use \eqref{c2}. As for $\phi^{(I)}$, we study one of the two terms we get using \eqref{c2} into \eqref{coeff}:
			\begin{equation}
				\begin{aligned}
					&\frac{1}{\lambda_{\rho}^{1/4}} \int_{\beta \lambda_\rho^{\frac{1}{\chi}}}^{\zeta \lambda_\rho^{\frac{1}{2+\chi}}}dx\,x^{\chi-\xi} e^{-x^2}\text{Ai}(\xi(x))\frac{|\xi(x)|^{1/4}}{|\Gamma^2(x)|^{1/4}}\sim\\ &\frac{1}{\lambda_{\rho}^{1/12}} \int_{\beta \lambda_\rho^{\frac{1}{\chi}}}^{\zeta \lambda_\rho^{\frac{1}{2+\chi}}} dx\,x^{\chi-\xi} \frac{\left(\int_{x_1}^{x}(z^{\chi}-\lambda_{\rho})^{\frac{1}{2}}dz\right)^{1/6}}{\left(x^\chi -\lambda_\rho\right)^{1/4}} \text{Ai}\left(\frac{1}{\lambda_\rho^{1/3}}\left(\int_{x_1}^{x}(z^{\chi}-\lambda_{\rho})^{\frac{1}{2}}dz\right)^{2/3}\right)
				\end{aligned}
				\label{h7}
			\end{equation}
			We plug \eqref{relaz1} in \eqref{h7} and we substitute $y = x - \beta \lambda_\rho^{\frac{1}{\chi}}$, getting at the leading order in $\lambda_\rho$:
			\begin{equation}
				\frac{1}{\lambda_{\rho}^{1/12}}\int^{\zeta \lambda_\rho^{\frac{1}{2+\chi}}}_{0}dy\, y^{\chi-\xi+\frac{1}{6}(\frac{\chi}{2}+1)-\frac{\chi}{4}}\text{Ai}\left(\frac{1}{\lambda_\rho^{1/3}} y^{\frac{2}{3}(\frac{\chi}{2}+1)}\right).
				\label{h8}
			\end{equation}
			Then we substitute $w = \frac{y^{\frac{2}{3}(\frac{\chi}{2}+1)}}{\lambda_\rho^{1/3}}$ in \eqref{h8}, obtaining:
			\begin{equation}
				\lambda_\rho^{\frac{\frac{3}{4}\chi+1-\xi}{\chi+2}} \int_{0}^{1} dw \,w^{\frac{3}{2}\frac{\chi+1}{2+\chi}}\text{Ai}(w)\sim \lambda_\rho^{\frac{\frac{3}{4}\chi+1-\xi}{\chi+2}},
			\end{equation}
			hence obtaining:
			\begin{equation}
				\int_{\beta \lambda_\rho^{\frac{1}{\chi}}}^{\zeta \lambda_\rho^{\frac{1}{2+\chi}}} dx\,x^\chi e^{-x^2}\phi_{\rho}^{(II)}(x)\sim \lambda_\rho^{\frac{\frac{3}{4}\chi+1-\xi}{\chi+2}}.
				\label{j2}
			\end{equation}
			The exponent in the right hand side of \eqref{j2} is positive, since we have imposed in \eqref{trueFun} the condition $\xi<\frac{\chi+1}{2}$.
			
			\item  For $x\in[\zeta \lambda_\rho^{\frac{1}{2+\chi}},x_2-\frac{\delta_1}{\sqrt{-\log \lambda_\rho}}]$  we use \eqref{c3}. Plugging $\phi_{\rho}^{(III)}$ in \eqref{coeff} we get two terms. As before, we focus on just one of them:
			\begin{equation}
				\int^{x_2-\frac{\delta_1}{\sqrt{-\log \lambda_\rho}}}_{\zeta \lambda_\rho^{\frac{1}{2+\chi}}}dx\,x^{\chi-\xi} e^{-x^2}\frac{1}{(x^\chi e^{-x^2}-\lambda_{\rho})^{1/4}}\sin\left(\int_{x_1}^x \sqrt{\Gamma^2(z)}dz+\frac{\pi}{4}\right),				
			\end{equation}
			which becomes, at the leading order in $\lambda_\rho$:
			\begin{equation}
				\int^{x_2-\frac{\delta_1}{\sqrt{-\log \lambda_\rho}}}_{\zeta \lambda_\rho^{\frac{1}{2+\chi}}}dx\,x^{\frac{3}{4}\chi-\xi} e^{-\frac{3}{4}x^2}\sin\left(\frac{1}{\lambda_\rho^{1/2}}\int_{0}^x z^{\frac{\chi}{2}}e^{-\frac{1}{2}z^2}dz\right).
				\label{h5}
			\end{equation}
			
			We now make use of the following result \cite{Olver2008} for oscillating integrals. Given an integral of the following kind:
			\begin{equation}
				I[f]=\int_a^b f(x)e^{i\omega g(x)}dx,
			\end{equation}
			with $f$ and $g$ sufficently differentiable functions. If $g'(x)\neq 0$ for $x\in[a,b]$, then the following expansion holds:
			\begin{equation}
				I[f]\sim \sum\limits_{k=1}^{\infty}\frac{1}{(-i\omega)^k}\left[ \sigma_k(b)e^{i\omega g(b)}-\sigma_k(a)e^{i\omega g(a)}\right],
				\label{olver1}
			\end{equation}
			where $\sigma_1=\frac{f}{g'}$ and $\sigma_{k+1}=\frac{\sigma_k^{'}}{g'}$ for $k\ge1$. The relation \eqref{olver1} can be proved integrating by parts.
			
			In our case, we have:
			\begin{equation}
				\omega =\frac{1}{\lambda_\rho^{1/2}},\quad f(x)=x^{\frac{3}{4}\chi-\xi} e^{-\frac{3}{4}x^2},\quad g(x)=\int_{0}^x z^{\frac{\chi}{2}}e^{-\frac{1}{2}z^2}dz,
			\end{equation}
			and:
			\begin{equation}
				a \sim \lambda_\rho^{\frac{1}{2+\chi}}, \qquad b\sim\sqrt{-\log \lambda_\rho}.
			\end{equation}
			Since $g'(x)=x^{\frac{1}{2}\chi} e^{-\frac{1}{2}x^2}$ is never 0 in the interval given by $a$ and $b$, we can apply \eqref{olver1}. Since we are interested in the limit of small $\lambda_\rho$ (hence of highly oscillating integrals), we can stop at the first order in the expansion \eqref{olver1}, getting:
			\begin{equation}
				\int^{x_2-\frac{\delta_1}{\sqrt{-\log \lambda_\rho}}}_{\zeta \lambda_\rho^{\frac{1}{2+\chi}}}dx\,x^{\frac{3}{4}\chi-\xi} e^{-\frac{3}{4}x^2}\sin\left(\frac{1}{\lambda_\rho^{1/2}}\int_{0}^x z^{\frac{\chi}{2}}e^{-\frac{1}{2}z^2}dz\right)\sim \lambda_\rho^{\frac{\frac{3}{4}\chi+1-\xi}{\chi+2}} +\lambda_\rho^{3/4},
			\end{equation}			
			where the first term comes from $x=a$, and it dominates the second term, which comes from $x=b$. Consequently, we have that:
			\begin{equation}
				\int^{x_2-\frac{\delta_1}{\sqrt{-\log \lambda_\rho}}}_{\zeta \lambda_\rho^{\frac{1}{2+\chi}}}dx\,x^{\chi-\xi} e^{-x^2} \left(\phi_{\rho}^{(III)}(x)\right)^2\sim \lambda_\rho^{\frac{\frac{3}{4}\chi+1-\xi}{\chi+2}}
				\label{j3}
			\end{equation}
			
			\item For $x\in[x_2-\frac{\delta_1}{\sqrt{-\log \lambda_\rho}}, x_2+\frac{\delta_2}{\sqrt{-\log \lambda_\rho}}]$ we use \eqref{c4}. As above, we consider just one of the two terms we get plugging $\phi_\rho^{(IV)}$ into \eqref{coeff}, and the same procedure can be applied to the other one:
			\begin{equation}
				W_1 \int_{x_2-\frac{\delta_1}{\sqrt{-\log \lambda_\rho}}}^{x_2+\frac{\delta_2}{\sqrt{-\log \lambda_\rho}}} dx\,x^{\chi-\xi} e^{-x^2}\text{Ai}\left[x_2^{1/3}(x-x_2)\right].
			\end{equation}
			Exploiting the fact that $x_2\sim \sqrt{-\log\lambda}\gg 1$, we do the following approximation, in addition to substituting $y=x-x_2$ and using the fact that $W_1\sim \lambda_\rho^{-1/4}$:
			\begin{equation}
				\begin{aligned}
					\frac{1}{\lambda_\rho^{1/4}}\int_{x_2-\frac{\delta_1}{\sqrt{-\log \lambda_\rho}}}^{x_2+\frac{\delta_2}{\sqrt{-\log \lambda_\rho}}} &dx\, e^{-x^2}\text{Ai}\left[x_2^{1/3}(x-x_2)\right]\sim\\ &\frac{1}{\lambda_\rho^{1/4}}\int_{-\frac{\delta_1}{\sqrt{-\log \lambda_\rho}}}^{\frac{\delta_2}{\sqrt{-\log \lambda_\rho}}} dy\, e^{-(y+x_2)^2}\text{Ai}\left[x_2^{1/3}y\right].
				\end{aligned}
			\end{equation}
			Noticing that $y\ll x_2$ and doing the substitution $w=x_2^{1/3}y$ we get:
			\begin{equation}
				\begin{aligned}
					\lambda_\rho^{3/4}\int_{-(\sqrt{-\log \lambda_\rho})^{2/3}}^{(\sqrt{-\log \lambda_\rho})^{2/3}} dw\, w^2\text{Ai}\left[w\right] \sim  \lambda_\rho^{3/4}.
				\end{aligned}
			\end{equation}
			Hence:
			\begin{equation}
				\int_{x_2-\frac{\delta_1}{\sqrt{-\log \lambda_\rho}}}^{x_2+\frac{\delta_2}{\sqrt{-\log \lambda_\rho}}} dx\,x^{\chi-\xi} e^{-x^2}\phi_{\rho}^{(IV)}(x)\sim \lambda_\rho^{3/4}.
				\label{j4}
			\end{equation}
			
			\item In the last interval $x\ge x_2+\frac{\delta_2}{\sqrt{-\log \lambda_\rho}}$ we plug \eqref{c5} into \eqref{coeff}, obtaining:
			\begin{equation}
				\int^{\infty}_{x_2+\frac{\delta_2}{\sqrt{-\log \lambda_\rho}}} dx\,x^{\chi-\xi} e^{-x^2}\phi_{\rho}^{(V)}(x)\le
				\int_{x_2+\frac{\delta_2}{\sqrt{-\log \lambda_\rho}}}^{\infty} dx\,x^{\frac{3\chi}{4}-\xi} e^{-\frac{3}{4}x^2} \sim \int_{x_2}^{\infty} dx\, e^{-\frac{3}{4}x^2}.
			\end{equation}
			Then we use the expansion of the erfc function for large $x$ \eqref{erfc} and, using \eqref{x1x2full}, we get for a constant $H_1>0$:
			\begin{equation}
				\int_{x_2+\frac{\delta_2}{\sqrt{-\log \lambda_\rho}}}^{\infty} dx\,x^{\chi-\xi} e^{-x^2}\phi_{\rho}^{(V)}(x) \le H_1 \lambda_\rho^{3/4}.
				\label{j5}
			\end{equation}
			 
		\end{itemize}
		As it is implied by Lemma \ref{evecChi0}, the  relations \eqref{j2}, \eqref{j3}, \eqref{j4} and \eqref{j5} hold also for $\chi=0$, since $\phi_{\rho}^{(II)}$ is valid up to $x=0$. Combining the contributes \eqref{j1}, \eqref{j2}, \eqref{j3}, \eqref{j4} and \eqref{j5} to the integral \eqref{coeff} defining $c_\rho$, we get the following asymptotic relation for odd eigenvectors $\phi_\rho$:
		\begin{equation}
			|c_\rho|\sim \lambda_\rho^{\frac{\frac{3}{4}\chi+1-\xi}{\chi+2}}.
		\end{equation}

		For even eigenvectors $\phi_\rho$, we have that the coefficient $c_\rho$ is 0, since it is an integral over all $\mathbb{R}$ of the odd function $p(x)f^{*}(x)\phi_\rho(x)$, with $p$ given by \eqref{pdf} and $f^*$ given by \eqref{trueFun}. We then get \eqref{scal1coeffApp}.			
		
	\end{proof}
	
	In the following proposition we find a numerical scheme to get small eigenvalues $\lambda_\rho$. 
	
	\begin{proposition}(Eigenvalues)    
		Let $x_1$ and $x_2$ be the solutions for $x>0$ of the equation $\frac{2p(x)}{\sigma\lambda_{\rho}}-\frac{1}{\sigma^2}=0$ for $p(x)$ given by \eqref{pdf}. Let $\phi_\rho$ be the eigenvector solution of \eqref{pdeApp}. We impose $\phi_\rho\rightarrow 0$ for $|x|\rightarrow\infty$. If $\phi_{\rho}$ is an odd function in $x$, then the eigenvalue $\lambda_{\rho}$ satisfies the following self-consistent equation for $\chi>0$:
		\begin{equation}
			\lambda_{\rho} =\left(\frac{\int_{x_1}^{x_2}dx\sqrt{2\frac{p(x)}{\sigma}-\frac{\lambda_\rho}{\sigma^2}}}{\arctan(-\gamma_{1}^{-1})+n \pi}\right)^2,\quad \rho=(2n+1)
			\label{sc1App}
		\end{equation}
		where $n\ge0$ is an integer, $\gamma_1 = \text{Ai}(\mu)/\text{Bi}(\mu)$, with $\text{Ai}$ and $\text{Bi}$ the Airy function of the first and second kind \cite{Florentin1966} and $\mu=\left(\frac{\chi (\lambda_\rho\Gamma[\frac{1+\chi}{2}])^{\frac{2}{\chi}}}{2^{\frac{2}{\chi}}\sigma^{2(1+\chi)}}\right)^{1/3}$. If $\phi_{\rho}$ is an even function in $x$ such that $\phi_{\rho}^{'}(0)=0$, then the eigenvalue $\lambda_{\rho}$ satisfies the following for $\chi>0$:
		\begin{equation}
			\lambda_{\rho} =\left(\frac{\int_{x_1}^{x_2}dx\sqrt{2\frac{p(x)}{\sigma}-\frac{\lambda_\rho}{\sigma^2}}}{\arctan(-\gamma_{2}^{-1})+n\pi}\right)^2,\quad \rho=(2n+2)
			\label{sc2App}
		\end{equation}
		where $\gamma_2 = \text{Ai}'(\mu)/\text{Bi}'(\mu)$ and $n\ge0$ integer.
		
		If $\chi=0$, the eigenvalues $\lambda_{\rho}$ satisfy the same equations \eqref{sc1App} and \eqref{sc2App} with the following replacements:
		\begin{equation}
			x_1 \rightarrow 0,\quad  \mu \rightarrow 0.
		\end{equation}
	\end{proposition}
	\begin{proof}
		We start from considering $\chi>0$ and odd eigenvectors $\phi_\rho(x)$. In the proof of Lemma \ref{evecChiMagg0}, we find that imposing the boundary condition $\phi_\rho(x)\rightarrow 0$ for $|x|\rightarrow\infty$, the following condition \eqref{condition} holds:
		\begin{equation}
			\gamma_1\sin\theta +\cos\theta=0,
			\label{condition2}
		\end{equation}
		where $\gamma_1 = \text{Ai}(\mu)/\text{Bi}(\mu)$, $\mu=\left(\frac{\chi (\lambda_\rho\Gamma[\frac{1+\chi}{2}])^{\frac{2}{\chi}}}{2^{\frac{2}{\chi}}\sigma^{2(1+\chi)}}\right)^{1/3}$ and:
		\begin{equation}
			\theta =\theta(\lambda_\rho)= \int_{x_1}^{x_2}\sqrt{2\frac{p(x)}{\lambda_\rho\sigma}-\frac{1}{\sigma^2}}dx+\frac{\pi}{4}.
			\label{thetaDef}
		\end{equation} 
		The relation \eqref{condition2} translates into a condition for the eigenvalues $\lambda_\rho$. The condition \eqref{condition2} can be solved by the following values of $\theta$:
		\begin{equation}
			\theta(\lambda_{\rho_1}) = -\arctan(-\gamma_1^{-1})+n_1\pi,
			\label{l1}
		\end{equation}
		with $n_1\ge 0$ and $\rho_1\ge1$ integers. A similar relation can be found for even eigenvectors:
		\begin{equation}
			\theta(\lambda_{\rho_2}) = -\arctan(-\gamma_2^{-1})+n_2\pi,
			\label{l2}
		\end{equation}
		where $\gamma_2 = \text{Ai}'(\mu)/\text{Bi}'(\mu)$, with $n_2\ge 0$ and and $\rho_2\ge1$ integers. We want now to find a relation between the integers $\rho_1$ and $n_1$ and between $\rho_2$ and $n_2$. We make the choice that, given a value of $n_1=n_2=n$, the eigenvalue of the odd eigenvector has rank $\rho_1=2n+1$ and the eigenvalue of the even eigenvector has rank $\rho_2=2n+2$. In that way, the integer $n$ is the index of an eigenvalue doublet. Moreover, the eigenvalues of the odd eigenvectors have odd rank, while the even ones have even rank.
		
		After this numbering choice, we develop the relation \eqref{l1} plugging \eqref{thetaDef} into it, getting then \eqref{sc1App}. The same can be done to obtain \eqref{sc2App}.
		
		Thanks to Lemma \eqref{evecChi0}, the same logic can be applied to the eigenvectors for $\chi=0$, just making the following replacements:
		\begin{equation}
			x_1 \rightarrow 0, \mu \rightarrow 0.
		\end{equation}
	\end{proof}
	
	\begin{corollary}
		Let $\lambda_{\rho}$ satisfy either the relation \eqref{sc1App} or \eqref{sc2App}, with $\gamma$ defined accordingly to $\chi$ as above. Then for large $\rho$, the following asymptotic relation holds:
		\begin{equation}
			\lambda_{\rho}\sim \rho^{-2}
			\label{scalEvalApp}
		\end{equation}
	\end{corollary}
	\begin{proof}
		Since we are looking at the limit of large ranks $\rho$ (small eigenvalues $\lambda_\rho$), we can just look at the relation \eqref{sc1App} and the same logic can be applied to \eqref{sc2App}. We rewrite \eqref{sc1App} expliciting the relation between $\rho$ and $n$:
		\begin{equation}
			\lambda_{\rho} =\left(\frac{\int_{x_1}^{x_2}dx\sqrt{2\frac{p(x)}{\sigma}-\frac{\lambda_\rho}{\sigma^2}}}{\arctan(-\gamma_{1}^{-1})+(\rho-1)\frac{\pi}{2}}\right)^2.
			\label{v1}
		\end{equation}
		Using \eqref{x1x2full} we have that, at the leading order in small $\lambda_\rho$, the numerator in \eqref{v1} is given by:
		\begin{equation}
			\left(\int_{x_1}^{x_2}dx\sqrt{2\frac{p(x)}{\sigma}-\frac{\lambda_\rho}{\sigma^2}}\right)^2= \mathcal{O}(1),
		\end{equation}
		while the denominator, for large $\rho$:
		\begin{equation}
			\left(\arctan(-\gamma_{1}^{-1})+(\rho-1)\frac{\pi}{2}\right)^2\sim \rho^{2},
		\end{equation}
		then giving the asymptotic relation \eqref{scalEvalApp}.
	\end{proof}
	\subsection{Numerics}
	
	The comparison between the eigenvalues obtained with the self-consistent numerical scheme \eqref{sc1App} and \eqref{sc2App} and the eigenvalues obtained diagonalising a large Gram matrix shows a good agreement. This is shown in Fig. \ref{eva_sc_gram}, realised in log-log scale and for $\sigma=100$ and for $\chi=0$ and $\chi=1$.	For very large $\rho$, the eigenvalues of the Gram matrix decay abruptly because of finite-size effects. The scaling \eqref{scalEval} captures the asymptotic behaviour of the eigenvalues $\lambda_\rho$, as we can notice in Fig. \ref{eva_sc_gram}.
	
	\begin{figure}[H]
		\centering
		\begin{subfigure}{.49\textwidth}
			\centering
			\includegraphics[width=1.\linewidth]{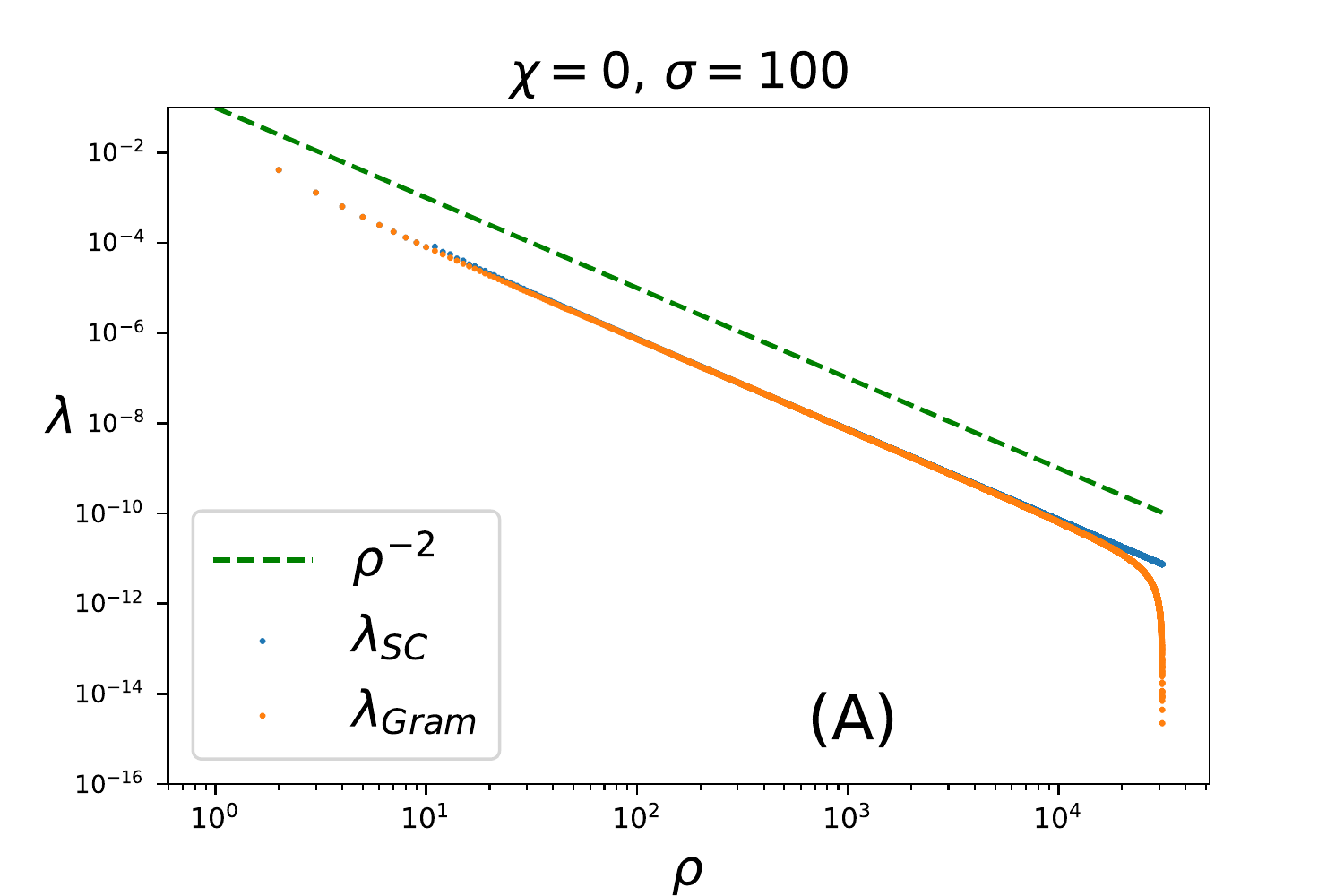}
			
		\end{subfigure}
		\begin{subfigure}{.49\textwidth}
			\centering
			\includegraphics[width=1.\linewidth]{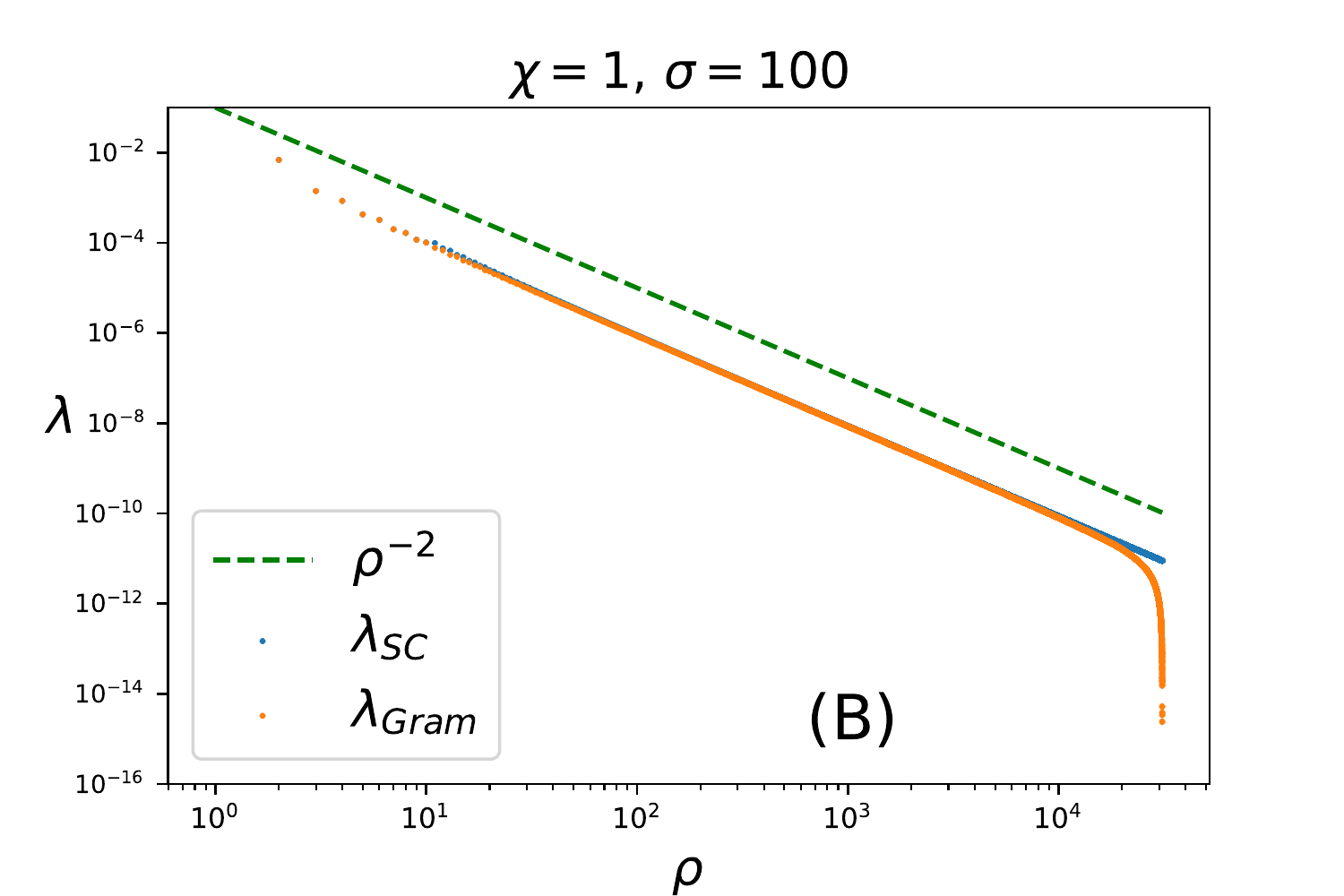}
			
		\end{subfigure}
		\caption{$d=1$, $\sigma=100$. Comparison of the eigenvalues $\lambda_{\rho}$ obtained via the self-consistent numerical scheme \eqref{sc1App} and \eqref{sc2App}, with label $\lambda_{SC}$ (blue points), with the eigenvalues obtained diagonalizing a large Gram Matrix $31k\times 31k$, with label $\lambda_{Gram}$ (orange points), for (A)  $\chi=0$ and (B) $\chi=1$. The dashed green line $\rho^{-2}$ indicates the predicted scaling $\lambda_{\rho}\sim \rho^{-2}$ in \eqref{scalEval}.}
		\label{eva_sc_gram}
	\end{figure}

		We can compute exactly the coefficients $c^2_\rho$  projecting the true function $f^*$ onto the eigenvectors $\phi_\rho$, obtained solving numerically the differential equation \eqref{pdeApp}, using as eigenvalues $\lambda_\rho$ the ones obtained from the numerical scheme \eqref{sc1} for ranks $\rho\ge 10^3$. For $\rho\le10^3$ we use the eigenvalues obtained diagonalising a large Gram matrix $31k\times 31k$. To solve the differential equation \eqref{pdeApp}, we use the method NDSolve in Mathematica. Once we compute them, we can compare their scaling with respect to $\rho$ with the one predicted in \eqref{coeffScalFinal} and the spectral bias prediction:
		\begin{equation}
		   c_{\rho}^2 \sim \rho^{-\frac{2\chi+2-\xi}{\chi+1}}, 
		   \label{coeffScalBias}
		\end{equation}
		obtained combining the test error scaling \eqref{scalTestError} and the spectral bias formula \eqref{sumP}. The comparison of the simulations shows a better agreement with the prediction \eqref{coeffScalFinal} than the spectral bias prediction, as shown in Fig. \ref{coeff_pde}, for $\xi=0$ and $\chi=0$ and 2. This suggests the \emph{non} applicability of the theory presented in \cite{Bordelon2020} in our setting, in the ridgeless case.
		
	\begin{figure*}
		\centering
		\begin{subfigure}{.49\textwidth}
			\centering
			\includegraphics[width=1.\linewidth]{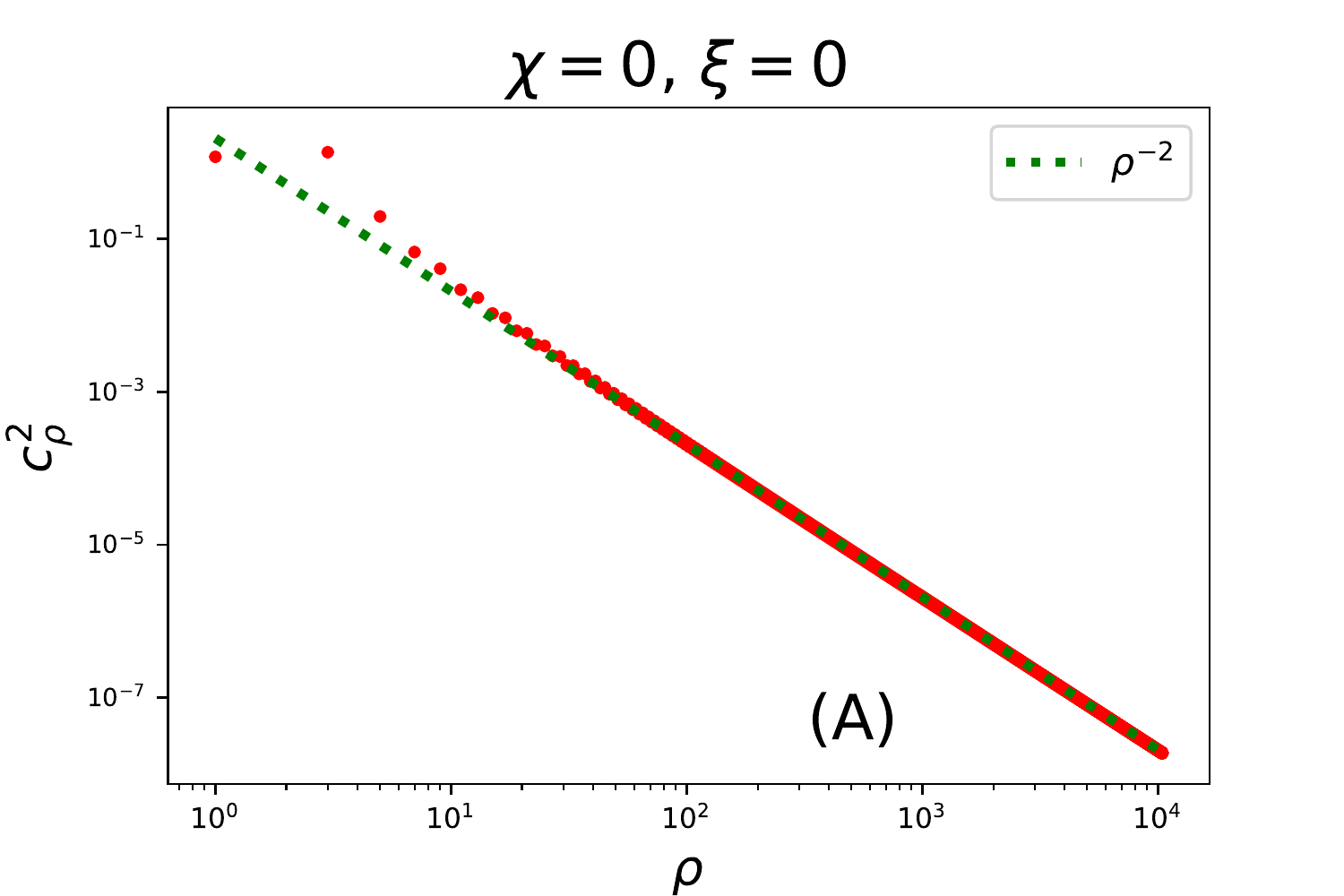}
			
		\end{subfigure}
		\begin{subfigure}{.49\textwidth}
			\centering
			\includegraphics[width=1.\linewidth]{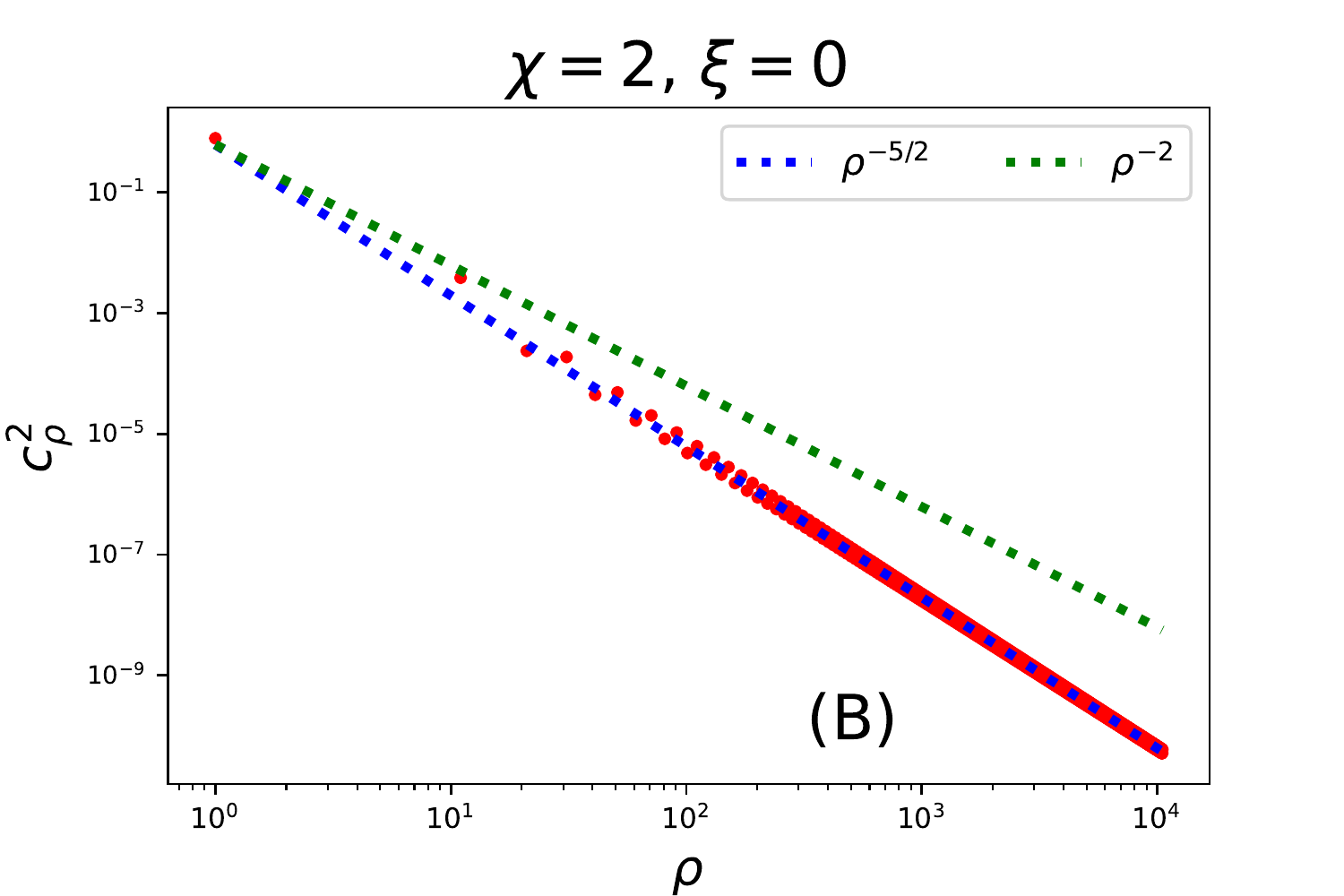}
			
		\end{subfigure}
		\caption{$d=1$, $\sigma=100$, $\xi=0$. Coefficients $c^2_{\rho}$ obtained projecting the true function $f^*$ onto the normalized eigenvector $\Psi_{\rho}$ solution of the PDE \eqref{pdeApp}. The coefficients are plotted with respect to the rank $\rho$ for (A) $\chi=0$ and (B) $\chi=2$. The green dashed line is the spectral bias prediction \eqref{coeffScalBias} and the blue dashed line is the theoretical prediction \eqref{coeffScalFinal}. For (A) $\chi=0$ the two predictions coincide, while for (B) $\chi=2$ they are different.}
		\label{coeff_pde}
	\end{figure*}

	\section{Proofs for finite ridge $\lambda$}\label{argumentPhase}
	\subsection{Case $d=1$}
	\label{scaling-one-d}
	
	\begin{proposition}
		Let $K$ be the Laplacian kernel with width $\sigma>0$: $K(x,y)=K(|x-y|)=\exp(-||x-y||_2 /\sigma)$. Let $f^*(x)$ be the true function and $p(x)$ the data distribution. Then the KRR predictor $f_P$ found via the minimisation problem \eqref{min} with ridge $\lambda$, in the limit of $P\rightarrow\infty$ and $\frac{\lambda}{P}$ finite, is given by the following differential equation:
		\begin{equation}
			\sigma^2 \partial_x^2 f_P(x) =\left(\frac{\sigma}{\lambda/P}p(x)+1\right)f_P(x)-\frac{\sigma}{\lambda/P}p(x)f^*(x).
			\label{sob2}
		\end{equation}
 	\end{proposition}
	
	\begin{proof}
		We start looking at a way to express the kernel norm $||.||_K$. Then, we find the KRR predictor $f_P$ taking the functional derivative of the minimisation problem \eqref{min} and imposing it to be zero.
		
		Let's now look at the kernel norm of the Laplace kernel $K(|x-y|)=\exp(-||x-y||_2 /\sigma)$. For a trial function $u(x)$ in the RKHS of the kernel $K$, the kernel norm is given by:
		\begin{equation}
			||u||_{K}^2=\int \, dx\,\int \, dy\, u(x)K^{-1}(x-y)u(y),
			\label{sobm1}
		\end{equation}
		where $K^{-1}(x-y)$ is the inverse kernel which satisfies $\int dy K^{-1}(x-y)K(y-z) = \delta(x-z)$. In \cite{ThomasAgnan1996} it is proven that the reproducing kernel $K_0(x,y)$ of the Sobolev space $S_{1,1}$ is given by $K_0(x,y)=e^{-|x-y|}$. This means that the kernel norm $||.||_{K_0}$ is given by the norm of $S_{1,1}$:
		\begin{equation}
			||u||_{K_0}^2 = ||u||_{S_{1,1}}^2 =\int dt\,u^2(t)+\int dt\,(u'(t))^2,
			\label{sob0}
		\end{equation}
		for any function $u(x)$ in $S_{1,1}$. 
		Following the proof of \eqref{sob0} in \cite{ThomasAgnan1996}, it is possible to prove that the kernel norm $||.||_{K}$ with $K$ given by the Laplace kernel $K(|x-y|)=\exp(-||x-y||_2 /\sigma)$ is very similar to \eqref{sob0}:
		\begin{equation}
			||u||_{K}^2 = \frac{1}{\sigma}\left(\int dt\,u^2(t)+\sigma^2\int dt\,(u'(t))^2\right).
			\label{sob1}
		\end{equation}
		For $P\rightarrow\infty$ and $\lambda/P$ fixed, we can restate the functional \eqref{min} which we want to minimise in KRR as follows:
		\begin{equation}
			\frac{\lambda/P}{\sigma}\left(\int dt\,u^2(t)+\sigma^2\int dt\,(u'(t))^2\right)+\int \,dx\,p(x)\left(f^*(x)-u\right)^2,
			\label{functional}
		\end{equation}
		for a trial function $u(x)$ in the RKHS of the kernel $K$. If we now take the functional derivative of \eqref{functional} with respect to $u$ and we put it equal to zero, we get the following differential equation for the KRR predictor $f_P$:
		\begin{equation}
			\sigma^2 f_P"(x) =\left(\frac{\sigma}{\lambda/P}p(x)+1\right)f_P(x)-\frac{\sigma}{\lambda/P}p(x)f^*(x).
		\end{equation}
		
	\end{proof}
	
	We want now to get the characteristic scale in $x$ of the predictor $f_P(x)$ with respect to $\lambda/P$. Since we are considering a $p(x)$ that is even in $x$ and an $f^*(x)$ that is odd in $x$, the predictor $f_P(x)$ obtained from Eq. \eqref{sob2} will be an odd function of $x$, therefore $f_P(0)=0$. We consider the characteristic scale $\ell$ of $f_P(x)$ for small $x$ and vanishing $\lambda/P$ as the scale over which the predictor grows from $f_P(0)=0$ to $f_P(\ell)\sim f^*(\ell)$. 
	
	\begin{lemma}
	    Let's consider the KRR predictor $f_P$ which solves the differential equation \eqref{sob2}. Its characteristic scale $\ell(\lambda,P)$, for $x\ll 1$ and $\lambda/P\rightarrow 0$, is given by:
	   \begin{equation}
 		\ell(\lambda,P)\sim \left(\frac{\lambda\sigma}{P}\right)^{\frac{1}{(2+\chi)}}.
    \label{ellAppLemma}
 	\end{equation}
	\end{lemma}
	
	\begin{proof}
	We notice that the differential equation \eqref{sob2} solved by $f_P$ is an inhomogeneous version of the following homogeneous equation:
		\begin{equation}
			\sigma^2 u"(x) =\left(\frac{\sigma}{\lambda/P}p(x)+1\right)u(x),
			\label{sob10}
		\end{equation}
	solved by a generic function $u(x)$. Using the variation of parameters method \cite{Teschl2004}, the general solution $f_P$ of the inhomogeneous equation \eqref{sob2} is given by a linear combination in two independent solutions $u_1(x)$ and $u_2(x)$ of the homogeneous \eqref{sob10}:
	\begin{equation}
	    f_P(x) = A(x)u_1(x) + B(x)u_2(x),
	    \label{sob11}
	\end{equation}
	where the functions $A(x)$ and $B(x)$ satisfy the following relation:
	\begin{equation}
	    A'(x)u_1(x)+B'(x)u_2(x)=0.
	\end{equation}
	Imposing that $f_P$, in the form \eqref{sob11}, solves \eqref{sob10}, the following expressions for $A(x)$ and $B(x)$ are obtained:
	\begin{equation}
	\begin{aligned}
	  A(x) &= \frac{\sigma}{\lambda/P}\int_0^x \,dy\, \frac{1}{W(y)}u_2(y) p(y)f^*(y)+a\\
	  B(x) &= -\frac{\sigma}{\lambda/P}\int_0^x \,dy\, \frac{1}{W(y)}u_1(y) \frac{\sigma}{\lambda/P}p(y)f^*(y)+b,
	  \label{sob13}
	\end{aligned}
	\end{equation}
	where $a$ and $b$ are integration constants and $W$ is the Wronskian of $u_1$ and $u_2$:
	\begin{equation}
	    W(y) = u_1(y)u'_2(y)-u'_1(y)u_2(y),
	\end{equation}
	which is different from 0 since $u_1$ and $u_2$ are independent. We now obtain an expansion in $\lambda/P$ for the solutions $u_1$ and $u_2$ . The homogeneous equation \eqref{sob10} belongs to the type of second-order equations that can be solved by the WKB method \cite{Ghatak1991} in the limit of small $\lambda/P$. It is a method of multi-scale analysis and  the idea behind it is described in Section \ref{1d1}. The generic WKB solution which is proposed to solve \eqref{sob10} has the form:
	\begin{equation}
	    u(x) = e^{\pm i \frac{S(x)}{\lambda/P}},\quad S(x) =S_0(x)+(\lambda/P)S_1(x)+ o\left(\frac{\lambda}{P}\right),
	    \label{sob15}
	\end{equation}
	where $S(x)$ has been expanded in powers of the small parameter $\lambda/P$. Looking at the order 0 in $\lambda/P$, we get the following two independent solutions:
	\begin{equation}
	    u_{1,2}(x) = C e^{\pm \sqrt{\frac{\sigma}{\lambda/P}}\int_0^x\,d\eta\sqrt{p(\eta)+\frac{\lambda/P}{\sigma}}} +O(\lambda/P),
	    \label{sob12}
	\end{equation}
	where $C$ is a constant. We stop at the order 0 in $\lambda/P$ in the expansion \eqref{sob15} since we are interested only in the characteristic scale of $u(x)$ with respect to $x$, which can be extracted by the exponential in \eqref{sob12}. Indeed, considering higher orders in $\lambda/P$, we would get a polynomial factor multiplying the exponential in \eqref{sob12}, as shown in \eqref{wkb4} in the main text. Plugging \eqref{sob12} into the expression of $A(x)$ in \eqref{sob13} we get:
	\begin{equation}
	    A(x) = \sqrt{\frac{\sigma}{\lambda/P}} \int_0^x\,dy\, u_2(y) \sqrt{p(y)}f^*(y)+a.
	    \label{sob14}
	\end{equation}
	We now extract the characteristic scale of the first term $A(x)u_1(x)$ in the relation defining $f_P$ in \eqref{sob11}. The same analysis will apply for the second term in that relation. The exponent in \eqref{sob12}, for small $\lambda/P$ and small $x$, is given by $\sqrt{\frac{\sigma}{\lambda/P}}\int_0^{x}dy\  y^{\chi/2} \propto \sqrt{\frac{\sigma}{\lambda/P}}x^{1+\chi/2}$, yielding the following scale for the functions $u_1(x)$ and $u_2(x)$ at small $x$:
 	\begin{equation}
 		\ell(\lambda,P)\sim \left(\frac{\lambda\sigma}{P}\right)^{\frac{1}{(2+\chi)}}.
    \label{ellApp}
 	\end{equation}
	The characteristic scale of $A(x)$ in \eqref{sob14} is given by the scale of $u_2$, which is again that of \eqref{ellApp}. Indeed, for small $x$ and $y$, the factor $\sqrt{p(y)}f^*(y) $ in \eqref{sob14} is just a polynomial factor $y^{\frac{\chi}{2}-\xi}$, which does not affect the fact that the main scale of $A(x)$ is the one given by $u_2$. As a consequence, the characteristic scale of the predictor $f_P$ is given by \eqref{ellApp}.
	\end{proof}
	We remark that it exists a second, and quicker, way to obtain the characteristic scale $\ell$ with respect to $x$ of $f_P$. We notice that the left hand side of Eq. \eqref{sob2} scales as:
	\begin{align}
	    \sigma^2 f_P''(x)\sim \sigma^2\frac{f^*(\ell)}{\ell^2}    
	\end{align}
    and the right hand side of \eqref{sob2} scales as:
    \begin{align}
        \frac{\sigma}{\lambda/P}p(x)\left(f_P(x)-f^*(x)\right)\sim \frac{\sigma}{\lambda/P}p(\ell)f^*(\ell) \sim \frac{\sigma}{\lambda/P}\ell^\chi f^*(\ell)
    \end{align} 
    Comparing the two sides we obtain again the characteristic scale \eqref{ellApp}.
	

	\subsection{Case $d>1$}
    \label{scaling-high-d}
    
    For a given kernel $K(x-y)$, we consider the predictor $f_P(x)$ that minimizes the functional $\lambda/P||f_P||^2_{K} + \int \,dx\,p(x)\left(f^*(x)-f_P(x)\right)^2$.    
    The predictor $f_P(x)$ can be written as $f_P(x)=\int d^d\eta \frac{p(\eta) f^*(\eta)}{\lambda/P} G(x,\eta)$, where the Green function $G(x,\eta) = G_{\eta}(x)$ satisfies the equation
    $\int d^dy K^{-1}(x-y) G_{\eta}(y) = \frac{p(x)}{\lambda/P} G_{\eta}(x) + \delta(x-\eta)$, where  $\int d^dy K^{-1}(x-y)K(y-z)=\delta(x-z)$.
    Taking the Fourier transform $\mathcal{F}[...]$ we get
    \begin{align}
        \mathcal{F}[K](q)^{-1} \mathcal{F}[G_{\eta}](q) = \frac{1}{\lambda/P} \mathcal{F}[p\  G_{\eta}](q) + e^{-i q \eta}
        \label{fourier}
    \end{align}
    where $q$ is the Fourier frequency. We now estimate each of this term, in the limit of small $x$, large $q$ and vanishing $\lambda/P$. Since the term $e^{-i q \eta}$ is such that $|e^{-i q \eta}|=1$, we drop the dependence of $G_{\eta}$ from $\eta$.
    
     For small $x\ll 1$, the transform on the right hand side of Eq. \eqref{fourier} becomes $\mathcal{F}[p(x)\  G(x)](q) \sim \mathcal{F}[x^{\chi}\  G(x)](q) \sim \partial_q^{\chi} \mathcal{F}[G](q)$. Assuming a power-law behavior (to be confirmed self-consistently) of this quantity for large $q$, we have $\partial_q^{\chi} \mathcal{F}[G](q)\sim q^{-\chi}\mathcal{F}[G](q)$.
    
    Furthermore, for a Laplace kernel we have $\mathcal{F}[K](q)^{-1} \sim q^{1+d}$. 
    
    Comparing these two terms, we  obtain two regimes: 
    \begin{eqnarray}
    \mathcal{F}[G](q)&\sim& q^{-1-d}\ \ \  \hbox{for}\ \ \ q\gg q_c\\
    \mathcal{F}[G](q)&\sim& \frac{\lambda}{P}q^\chi\ \ \  \hbox{for}\ \ \ q\ll q_c\\
      \hbox{with}\ \ \   q_C&\sim& \left(\frac{\lambda}{P}\right)^{-\frac{1}{1+d+\chi}}
    \end{eqnarray}
    \vspace{-.2em}
     Thus in magnitude, $\mathcal{F}[G](q)$ is maximum for $q\sim q_c$. It implies that in real space, $G(x)$ is characterized by a length scale:
       \begin{equation}
           \ell(\lambda,P)\sim 1/q_c\sim \left(\frac{\lambda}{P}\right)^{\frac{1}{1+d+\chi}}
       \end{equation} 
    
    
    
	
	\vspace{-1em}
	\section{Sampling scheme details} \label{num}
	In this appendix we give further details about our sampling scheme for the training and test sets used for KRR simulations. \subsection{One dimension}
	
	\textbf{Training set} We sample $P$ points from the PDF \eqref{pdf}, in the interval $x\in [-x_{\text{max}},x_{\text{max}}]$. We do it using the rejection sampling algorithm. We choose $x_{\text{max}}=3$. 
	
	\textbf{Test set} Given $\lambda$ and $P$, we find the characteristic length of the predictor $f_P$, computed from the training set, as follows. We compute the derivative of $f_P$ on a fine grid over $x\in[0,x_{\text{max}}]$. We take as estimate of the characteristic length of $f_P$ the point $\tilde{x}$ such that $f_P'(\tilde{x})=\frac{1}{10}f_P'(0)$. Then, we divide the interval $x\in[0,x_{\text{max}}]$ in $m$ bins $[x_j,x_{j+1}]$ of width given by $\tilde{x}$, with $j\in{0,...,m-1}$. Then we compute the contribute of the test error $\varepsilon_t$ as an integral over a grid made by $10^5\cdot e^{-j}+q$ points, where $q$ is given by $\text{max}[\frac{1}{m}10^5,2000]$. The number $q$ states the minimum number of points per bin. If $\frac{1}{m}10^5 >2000$ then we sample at least $10^5$ points (in addition to $10^5\cdot e^{-j}$), otherwise we sample at maximum 2000 points per bin. We sum the contributes over $j$ to get the full $\varepsilon_t$. 
	\vspace{-1em}
	\subsection{$d>1$ case}
	In the case of generic dimension $d$ described in Section \ref{setting}, the sampling along the informative direction $x_1$ is the same as in the one-dimensional case above. For the other $d$ coordinates, we first sample from a $d-$dimensional standard Gaussian distribution. Secondly, we normalize these coordinates by their $d-$dimensional $L_2$ norm, to collocate the points on a cylindrical surface.
	

	\vspace{-1em}
	\section{Additional Figures} \label{additionalFigures}

	In Fig. \ref{realdata_mnist} we repeat the analysis done for CIFAR10 in the main text in Fig. \ref{realdata_cifar10} for a binary version of the dataset MNIST.

    \begin{figure*}[b]
		\centering
		\includegraphics[width=0.85\linewidth]{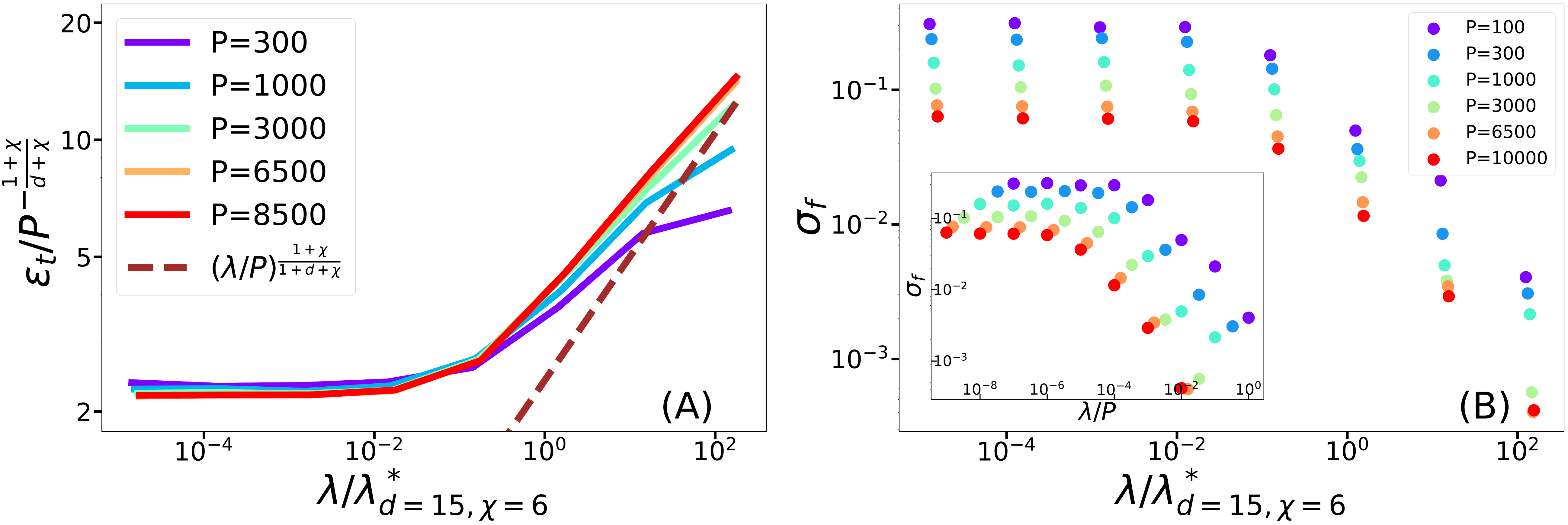}
		\caption{Binary MNIST. (A): Empirical test error $\varepsilon_t$ {\it v.s. } ridge. Each quantity is rescaled by our predictions  \eqref{errT-scaling} and \eqref{lstarD} for $d=15$ and $\chi=6$. The dashed brown line is the scaling prediction of the test error with respect to $\lambda$ of \eqref{errBlambda}. (B) Inset: variance of the predictor $\sigma_f$ {\it v.s.}   re-scaled ridge $\lambda/P$. Main plot: After rescaling the ridge by $\lambda^*_{d=15,\chi=6}$, curves nearly collapse.}
		\label{realdata_mnist}
	\end{figure*}

\end{document}